\documentclass{article}

\usepackage{microtype}
\usepackage{graphicx}
\usepackage{subfigure}
\usepackage{booktabs} 

\usepackage{comment}
\usepackage{multirow}
\usepackage[table]{xcolor}

\usepackage{amsmath,amsfonts,bm}

\def\1{\bm{1}}

\def\eps{{\epsilon}}

\def\rvg{{\mathbf{g}}}

\def\rvk{{\mathbf{k}}}

\def\rvz{{\mathbf{z}}}

\DeclareMathAlphabet{\mathsfit}{\encodingdefault}{\sfdefault}{m}{sl}
\SetMathAlphabet{\mathsfit}{bold}{\encodingdefault}{\sfdefault}{bx}{n}

\DeclareMathOperator*{\argmin}{arg\,min}

\usepackage{hyperref}

\usepackage[accepted]{icml2024}

\usepackage{amsmath}
\usepackage{amssymb}
\usepackage{mathtools}
\usepackage{amsthm}

\usepackage[capitalize,noabbrev]{cleveref}

\theoremstyle{plain}
\newtheorem{theorem}{Theorem}[section]
\newtheorem{propo}[theorem]{Proposition}
\newtheorem{lemma}[theorem]{Lemma}
\newtheorem{corollary}[theorem]{Corollary}
\theoremstyle{definition}
\newtheorem{defi}[theorem]{Definition}

\theoremstyle{remark}

\newtheorem{example}[theorem]{Example}

\usepackage[textsize=tiny]{todonotes}

\icmltitlerunning{Single-Model Attribution of Generative Models Through Final-Layer Inversion}

\newcommand{\ml}[1]{\textcolor{black}{ #1}}

\newcommand{\method}{\mbox{FLIPAD}}

\newcommand{\finger}{SM-F}
\newcommand{\inv}{$\text{SM-Inv}_{2}$}
\newcommand{\incinv}{$\text{SM-Inv}_{\operatorname{inc}}$}

\newcommand{\raw}{RawPAD}
\newcommand{\dct}{DCTPAD}
\newcommand{\Routin}{\mathbb{R}^{D_{\operatorname{out}} \times D_{\operatorname{in}}}}
\newcommand{\Din}{D_{\operatorname{in}}}
\newcommand{\Dout}{D_{\operatorname{out}}}

\usepackage{url}

\begin{document}

\twocolumn[
\icmltitle{Single-Model Attribution of Generative Models Through Final-Layer Inversion}




\begin{icmlauthorlist}
\icmlauthor{Mike Laszkiewicz}{rub}
\icmlauthor{Jonas Ricker}{rub}
\icmlauthor{Johannes Lederer}{hh}
\icmlauthor{Asja Fischer}{rub}
\end{icmlauthorlist}

\icmlaffiliation{rub}{Faculty of Computer Science, Ruhr University Bochum, Germany}
\icmlaffiliation{hh}{Department of Mathematics, Computer Science, and Natural Sciences, University of
Hamburg, Germany}

\icmlcorrespondingauthor{Mike Laszkiewicz}{mike.laszkiewicz@rub.de}


\vskip 0.3in
]



\printAffiliationsAndNotice{}  

\begin{abstract}
Recent breakthroughs in generative modeling have sparked interest in practical single-model attribution. Such methods predict whether a sample was generated by a \textit{specific} generator or not, for instance, to prove intellectual property theft. However, previous works are either limited to the closed-world setting or require undesirable changes to the generative model. We address these shortcomings by, first, viewing single-model attribution through the lens of anomaly detection. Arising from this change of perspective, we propose \method{}, a new approach for single-model attribution in the open-world setting based on final-layer inversion and anomaly detection. We show that the utilized final-layer inversion can be reduced to a convex lasso optimization problem, making our approach theoretically sound and computationally efficient. The theoretical findings are accompanied by an experimental study demonstrating the effectiveness of our approach and its flexibility to various domains.
\end{abstract}

\begin{figure}
    \centering
    \includegraphics[width=\linewidth]{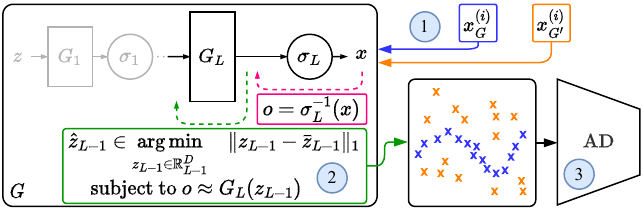}
    \caption{
    Single-model attribution with \method{}. Given a generative model $G$, \method{} performs the following steps: 1) The training data includes generated samples $x_G^{(i)}$ from $G$ and samples from a different source  $x^{(i)}_{G'}$. For each $x$ we compute the optimization target $o$ by inverting the final activation $\sigma_L$. 2) For each output $o$, we perform final-layer inversion by finding an activation $\hat{z}_{L-1}$ that is close to the expected activation $\bar{z}_{L-1}$ and an approximate solution to $o \approx G_L(\hat{z}_{L-1})$. 3) Since final-layer inversion reveals differences between different data sources, the activations can be used as features to train an anomaly detector.
    }
    \label{fig:pipeline}
\end{figure}


\section{Introduction} \label{sec:intro}
Deep generative models have taken a giant leap forward in recent years;
humans are not able to distinguish real and synthetic images anymore \citep{nightingale2022, mink2022, lago2022},
and text-to-image models like DALL-E~2 \citep{ramesh2022}, Imagen \citep{saharia2022}, and Stable Diffusion \citep{rombach2022} have sparked a controversial debate about AI-generated art \citep{reaCouldRobotEver2023}.
But this superb performance comes, quite literally, at a price:
besides technical expertise and dedication, 
large training datasets and enormous computational resources are needed.
For instance, the training of the large-language model GPT-3 \citep{brown2020} is estimated to cost almost five million US dollars \citep{OpenAIGPT3Language}.
Trained models have become valuable company assets---and an attractive target for intellectual-property theft.
An interesting task is, therefore, to \textit{attribute} a given sample to the generative model that created it.

Existing solutions to this problem can be divided into two categories: fingerprinting methods \citep[e.g.,][]{marraGANsLeaveArtificial2019, yuAttributingFakeImages2019} exploit the fact that most synthetic samples contain generator-specific traces that a classifier can use to assign a sample to its corresponding source.
Inversion methods \citep{albright2019source, zhangAttributionDeepfakes2021, hirofumi2022}, on the other hand, are based on the idea that a synthetic sample can be reconstructed most accurately by the generator that created it.
However, both approaches are designed for the \textit{closed-world setting}, which means that 
they are designed to differentiate between models from a given set that is chosen before training and stays fixed during inference.
In practice, this is rarely the case, since 
new models keep emerging.
For detecting intellectual property theft, the setting of \textit{single-model attribution} is more suitable.
Here, the task is to determine whether a sample was created by \textit{one} particular generative model or not, in the \textit{open-world setting}.
Previous work has found that this problem can be solved by complementing the model with a watermark that can be extracted from generated samples \citep{yu2021artificial, kim2021decentralized, yu2020responsible, nie2023attributing, wen2023tree}.
Although watermarking has shown to be very effective, a major disadvantage is that it has to be incorporated into the training process, which a) might not always be desirable or even possible and b) can deteriorate generation quality.

\begin{figure}
    \centering
    \includegraphics[width=\columnwidth]{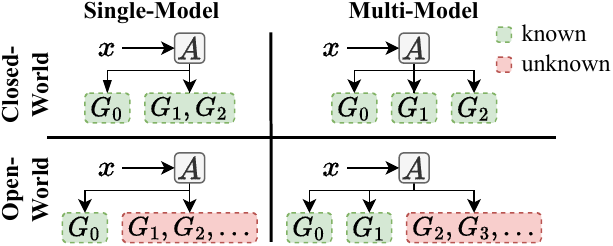}
    \caption{The two dimensions of the model attribution problem. \ml{Let $A$ be the attribution method.} While single-model attribution solves a binary decision problem ($G_0$ or something else?), multi-model attribution can distinguish between more than two classes. In an open-world setting, it is also possible that a sample stems from an unknown generator.}
    \label{fig:problem_setup}
\end{figure}

In this work, we tackle the problem of single-model attribution in the open-world setting by viewing it as an anomaly detection task, which has not been done in previous works. 
This novel viewpoint motivates two simple, yet effective, baselines for the single-model attribution task. 
Furthermore, we find that incorporating knowledge about the generative model can improve the attribution accuracy and robustness. 
Specifically, we propose \textbf{F}inal-\textbf{L}ayer \textbf{I}nversion \textbf{P}lus \textbf{A}nomaly \textbf{D}etection (\method{}).
First, it extracts meaningful features by leveraging the available model parameters to perform \textit{final-layer inversion}.
Second, it uses established anomaly detection techniques on the extracted features to predict whether a sample was generated by \textit{our} generative model.
While related to existing attribution methods based on inversion, our method is i) significantly more efficient due to the convexity of the optimization problem and ii) theoretically sound given the connection to the denoising basis pursuit.
We highlight that this approach can be performed without altering the training process of the generative model.
Moreover, it is not limited to the image domain since features are directly related to the generative process and do not exploit visual characteristics like frequency artifacts.
We illustrate \method{} in Figure~\ref{fig:pipeline}.
Our contributions are summarized as follows:
\begin{itemize}
    \item We address the problem of single-model attribution in the open-world setting, which previously has not been solved adequately. 
    Our work is the first to establish the natural connection between single-model attribution and anomaly detection. This change of perspective motivates the application of anomaly detection on raw inputs (\raw) and on DCT-features (\dct). 

    \item Building upon on these methods, we introduce \method{}, which combines final-layer inversion and anomaly detection.
    We prove that the proposed inversion scheme can be reframed as a convex lasso optimization problem, 
    which makes it theoretically sound and computationally efficient. 
    \item In an empirical analysis, we show the effectiveness and versatility of the proposed methods for a range of generative models, including GANs and diffusion models. Notably, our approach is not limited to the image domain.
\end{itemize}


\section{Problem Setup} \label{sec:problem_setup}
We consider the problem of model attribution to have two fundamental dimensions, which are illustrated in \cref{fig:problem_setup}.
First, we differentiate between single-model and multi-model attribution.
While in multi-model attribution the task is to find out by which specific generator (of a given set) a sample was created, single-model attribution only cares about whether a sample stems from a single model of interest or not.
Note that, conceptually, real samples can be considered to be generated by a model as well, and thus, the task of distinguishing fake from real samples also fits into this framework.  
Second, attribution can be performed in the closed-world or open-world setting.
The requirement for the closed-world setting is that all generators which \textit{could} have generated the samples are known beforehand.
In contrast, in the open-world setting an attribution method is able to state that the sample was created by an unknown generator.

The method proposed in this work solves the problem of \textit{single-model attribution in the open-world setting}.
Naturally, this setting applies to model creators interested in uncovering illegitimate usage of their model (or samples generated by it).
In this scenario, \method{} only requires resources that are already available to the model creator: the model $G$ itself (including its parameters), samples $x_G$ generated by $G$, and real samples $x_{\operatorname{real}}$ used to train $G$.
We emphasize that no samples generated by other models are needed.
Therefore, we do not rely on restrictive data acquisition procedures, which might involve i) using additional computational resources to train models or sample new data, ii) expensive API calls, or iii) access to models that are kept private.


\section{Related Work} \label{sec:related_work}
We divide the existing work on deepfake attribution into three areas: fingerprinting, inversion, and watermarking.
Regarding the problem setup, methods based on fingerprinting and inversion perform multi-model attribution in the open-world setting, while watermarking provides single-model attribution in the open-world setting. \ml{Note that all existing methods are designed for image data.}
A categorization of the related work is given in \cref{app:related_work}.

\paragraph{Fingerprinting}
\label{sec:related_work:fingerprinting}
The existence of fingerprints in GAN-generated images was demonstrated first by \citet{marraGANsLeaveArtificial2019}.
They show that by averaging the noise residuals of generated images, different models can be distinguished based on the correlation coefficient between the image of interest and a set of model-specific fingerprints.
Other works show that the accuracy can be improved by training an attribution network \citep{yuAttributingFakeImages2019} or extracting fingerprints in the frequency domain \citep{frank2020leveraging}.
Since fingerprinting only works in closed-world settings, methods to efficiently extend the set of known models (without expensive re-training) have been proposed \citep{marra2019incremental, xuanScalableFinegrainedGenerated2019}.

\paragraph{Inversion}
\label{sec:related_work:inversion}
Inversion-based attribution is based on the finding that an image can be reconstructed by the generator that it originates from. 
While existing works \citep{albright2019source, karrasAnalyzingImprovingImage2020, zhangAttributionDeepfakes2021, hirofumi2022} prove that this approach is effective, it is computationally expensive since the optimal reconstruction has to be found iteratively using backpropagation. 
Additionally, since there is no guaranteed convergence due to the non-convexity of the optimization problem, multiple reconstruction attempts are necessary for each image.

\paragraph{Watermarking}
\label{sec:related_work:watermarking}
Watermarking methods solve the attribution problem by proactively altering the generative model such that all generated images include a recognizable identifier.
\citet{yu2021artificial} were the first to watermark GANs by encoding a visually imperceptible fingerprint into the training dataset.
Generated images can be attributed by checking the decoded fingerprint.
In a follow-up work \citep{yu2020responsible}, the authors propose a mechanism for embedding fingerprints more efficiently by
directly embedding them into the model's convolutional filters.
A related work by \citet{kim2021decentralized} uses fine-tuning to embed keys into user-end models, which can then be distinguished using a set of linear classifiers.
Recent works focus on embedding keys into the latent representation of an image,
for instance, by altering its representation in specific dimensions~\citep{nie2023attributing} or by perturbing its Fourier representation~\citep{wen2023tree}. 
A disadvantage of watermarking methods is that the model itself has to be altered, which might not always be feasible.

\paragraph{Other Related Work}
\citet{girish2021towards} propose an approach for the discovery and attribution of GAN-generated images in an open-world setting.
Given a set of images, their iterative pipeline forms clusters of images corresponding to different GANs.
A special kind of watermarking inspired by backdooring \citep{AdiBackdooring} is proposed by \citet{ong2021protecting}.
They train GANs which, given a specific trigger input as latent $z$, generate images with a visual marking that proves model ownership.
As a consequence, this method requires query access to the suspected model, making it applicable in those scenarios only.


\section{Methodology}\label{sec:main}
We begin by developing a viewpoint on single-model attribution through the lens of anomaly detection in~\cref{sec:framework}. 
This approach is complemented by a novel feature extraction method based on final-layer inversion.
We provide examples illustrating why these features are appropriate for model attribution in~\cref{sec:motivationalExample}, followed by a practical and efficient algorithm for final-layer inversion  in~\cref{sec:method}. 

\subsection{Leveraging Anomaly Detection for Single-Model Attribution}\label{sec:framework}
Recalling \cref{sec:problem_setup}, we aim at deciding whether a sample was generated by our model or not.
Treating samples from our model as normal samples and samples from unknown models as anomalies, we can rephrase the single-model attribution problem to an anomaly detection task.
Our proposed approach, therefore, consists of two modular components.

The first one is the anomaly detection itself.
We decide to use DeepSAD~\cite{ruff2019deep}, which works particularly well in high-dimensional computer-vision tasks but is also capable of generalizing to other domains. 
The underlying idea of DeepSAD builds upon SVDD~\citep{Tax2004SupportVD} but leverages the success of modern deep neural networks. In essence, DeepSAD learns a deep feature extractor $\phi$ that maps normal samples $x_{\operatorname{norm}}$ close and anomalies $x_{\operatorname{anom}}$ far from a prefixed point $c$ (i.e., $\Vert \phi(x_{\operatorname{norm}}) - c \Vert \ll \Vert \phi(x_{\operatorname{anom}}) - c \Vert$ ). The resulting distance to $c$ acts as the anomaly score. 
More details are provided in \cref{sec:deep_sad}.

The second---and crucial---component is extracting the features that serve as the input of the anomaly detector.
A trivial approach, which we refer to as \raw{} (\textbf{Raw} \textbf{P}lus \textbf{A}nomaly \textbf{D}etection), is to simply use the raw input, that is, to skip the feature extraction part, as features for the anomaly detector. 
Another option is to use pre-existing domain knowledge to handcraft suitable features.
In the case of GANs, we can exploit generation artifacts in the frequency domain \citep{frank2020leveraging} by using the discrete cosine transform of an image as features (\dct{}).

\subsection{Layer Inversion Reveals Model-Characteristic Features}
However, it is not always known which handcrafted features perform well, and they might not generalize to unseen models.
We, therefore, propose an alternative feature extraction method
that is universally applicable, i.e., not restricted to the image domain for instance.
The key idea is to incorporate knowledge about the generative model itself.
Inspired by inversion-based attribution methods, we estimate the hidden representation of a given image before the final layer.
We first clarify how these hidden representations can be used for single-model attribution by providing illustrative examples with simple linear generative models $G: \mathbb{R}^{D_0} \rightarrow \mathbb{R}^{D_1}$.
\label{sec:motivationalExample}

\begin{example}[Impossible Output] \label{exa:impossible}
    Given an output $x\in\mathbb{R}^{3\times 3}$ and a transposed convolution\footnote{without padding and with stride one} $G_{k}$ 
    parameterized by a $(2\times2)$-kernel $k$ specified below, we can solve the linear system $G_{k}(z)=x$ to find a solution $z$:
    \begin{equation*}
        x := \begin{pmatrix}
            0 & 0 & 1 \\
            0 & 4 & 6 \\ 
            4 & 12 & 9 
        \end{pmatrix}
        \text{ and } 
        k:= \begin{pmatrix}
            0 & 1 \\ 
            2 & 3 
        \end{pmatrix}
         \Rightarrow
        z = \begin{pmatrix}
            0 & 1 \\ 
            2 & 3 
        \end{pmatrix} \enspace . 
    \end{equation*}
    However, note that $G_{k}$ is not surjective, and more specifically, there exists no input $z$ that could have generated $x^\prime = x + ( (0,0,0)^\top, (0, \varepsilon, 0)^\top, (0, 0, 0)^\top)$ for $\varepsilon\neq 0$. 
    Hence, even though $x^\prime$ can be arbitrarily close to $x$---which means that continuous feature mappings cannot reveal differences between $x$ and $x^\prime$ in the limiting case---we can argue that $x^\prime$ has not been generated by $G_k$.
    This example demonstrates the first requirement of a sample $x$ to be generated by  $G_{k}$: it needs to be an element of the image space of $G_k$. 
\end{example}
\begin{example}[Unlikely Hidden Representation]\label{exa:unlikely}
        Let $G, G^\prime: \; \mathbb{R}^2 \rightarrow \mathbb{R}^2 $ be two linear generators given by the diagonal matrices $G:=\operatorname{diag}(2, 0.5)$ and $G^\prime:= \operatorname{diag}(1, 1)$.
        Furthermore, let $\rvz\sim\mathcal{N}(0, I)$. Since $G$ and $G^\prime$ are surjective, they have both the capacity to generate any output. But, given some output $x:= (1, 1)^\top$, one can readily compute 
         that 
        \begin{align*}
            \log \mathbb{P}\Bigl(\max \bigl( \vert G(\rvz) - x \vert \bigr) \leq 0.1 \Bigr) &\approx -7.16 \enspace , \\ 
            \log \mathbb{P}\Bigl(\max \bigl( \vert G^{\prime}(\rvz) - x \vert \bigr) \leq 0.1 \Bigr) &\approx -6.06 \enspace , 
        \end{align*}
        where the probability is taken over $\rvz$, $\max$ denotes the maximum over all coordinates, and $\vert \cdot \vert$ is the component-wise absolute value. We conclude that---even though $x$ is in the image space of both generators---$G^\prime$ is much more likely (note the logarithmic scale in the display) to generate an output close to $x$. 
        Once again, by inferring the hidden representation of $x$, we obtain implicit information, which can be used for model attribution.        
    \end{example}

\begin{example}[Structured Hidden Representation] \label{exa:structure}
    Lastly, the hidden representation of $x$ can reveal differences based on its structure, such as its sparsity pattern. 
    For illustration, let us again consider two linear generators $G$ and $G^\prime$ given by 
    $
    G=I \in \mathbb{R}^{d \times d},\ G^\prime =
    \begin{pmatrix}
                \mathbf{1} & e_2 & \dots & e_{d} 
    \end{pmatrix}\in \mathbb{R}^{d \times d} ,
    $
    where $\mathbf{1}\in \mathbb{R}^d$ is the vector of ones and $e_j\in\mathbb{R}^d$ are Euclidean unit vectors. Again, both generators are surjective but for $x=\mathbf{1}$ it is 
    $G(\mathbf{1})=x$ and $G^{\prime}(e_1) = x$. In particular, there is no input to $G$ with less than $d$ non-zero entries capable of generating $x$, whereas $G^\prime$ can generate $x$ with the simplest input, namely $e_1$. Hence, the 
    required input $z$ illuminates implicit structures of $x$ by incorporating knowledge about the generator. 
\end{example}

\subsection{Introducing \method{}}\label{sec:method}
 The examples in the previous section demonstrate that incorporating knowledge about the generative model~$G$ can reveal hidden characteristics of a sample. 
Inspired by the field of compressed sensing, we now derive an optimization procedure, which is capable of performing final-layer inversion in deep generative models. 
Combining this feature extraction method with the framework proposed in~\cref{sec:framework}, we end up with a powerful and versatile single-model attribution method, which we coin FLIPAD (\textbf{F}inal-\textbf{L}ayer \textbf{I}nversion \textbf{P}lus \textbf{A}nomaly \textbf{D}etection).

In contrast to the examples presented in Section~\ref{sec:motivationalExample}, we proceed with a complex deep generative model, which generates samples according to
\begin{equation} \label{eq:generative_model}
    x = G(z) = \sigma_L (G_L ( \sigma_{L-1} G_{L-1}( \cdots G_1(z) \cdots )))\enspace ,  
\end{equation}
for activation functions $\sigma_l$, linearities $G_l:\mathbb{R}^{D_{l-1}} \rightarrow \mathbb{R}^{D_l}$ for $l\in \{1,\dots, L\}$, and $z$ are samples from the base random variable $\rvz\sim \mathcal{N}(0, I)$. 
In the same spirit of Example~\ref{exa:impossible} we can prove that $x$ cannot be generated by $G$ if we can show that there exists no input $z\in \mathbb{R}^{D_0}$ such that $x = G(z)$. However, since $G$ is typically a highly non-convex function, it remains a non-trivial task to check for the above criterion. Hence, we propose the following simplification: Is there any $(L-1)$-layer activation $z_{L-1}\in\mathbb{R}^{D_{L-1}}$ such that $x=\sigma_L ( G_L(z_{L-1}))$? Moreover, we assume $\sigma_L$ to be invertible\footnote{which is typically the case, e.g., for $\operatorname{tanh}$ or $\operatorname{sigmoid}$ activations}, so that we can define $o:= \sigma_L^{-1}(x)$ and solve the equivalent optimization problem 
\begin{equation}
    \label{eq:linear_sytem}
    \text{find } z_{L-1}\in\mathbb{R}^{D_{L-1}} \text{ such that } o=G_L (z_{L-1}) \enspace . 
\end{equation}
Note how this relaxes the intractable full inversion problem to a simpler inversion of a linear system similar to the examples covered in Section~\ref{sec:motivationalExample}.

In practice $G_L$ is typically surjective, that is, there always exists some $z_{L-1}\in\mathbb{R}^{D_{L-1}}$ such that $o = G_L(z_{L-1})$. And even worse, linear algebra provides us with the following fundamental result. 
\begin{propo}[see e.g., \citet{Hefferon2012LinearA}, Lemma 3.7]\label{prop:linear_algebra_fund}
    Let $G_L : \mathbb{R}^{D_{L - 1}} \rightarrow \mathbb{R}^{D_L}$ be a surjective linear function and $o\in \mathbb{R}^{D_L}$. Furthermore, let $z_{L-1} \in \mathbb{R}^{D_{L-1}}$ be a solution to the linear system~\eqref{eq:linear_sytem}. Then, every $z \in \mathcal{S}$ solves the linear system~\eqref{eq:linear_sytem}, where 
    \begin{equation*}
       \mathcal{S}:= \{ z_{L-1} + z: \; z \in \operatorname{ker}(G_L) \}  \enspace ,
    \end{equation*}
    and $\operatorname{ker}(G_L):=\{ z \in \mathbb{R}^{D_{L-1}}:\; G_L(z)=0 \}$ defines the kernel of $G_L$.
    Hence, the solution set $\mathcal{S}$ is a \mbox{$(D_{L-1} - {D_L})$-dimensional} affine space. 
\end{propo}
\noindent
First, this tells us that if $D_{L-1} > D_L$, there is not a single solution but an affine space of solutions of dimensionality $D_{L-1} - D_L$.
For instance, our experiments in Section~\ref{sec:exp_small} use a LSGAN with $D_{L-1}=64 \times 64 \times 64 = 262\,144$ and $D_{L}=3 \times 64 \times 64 = 12\, 288$, resulting in a $249\,856$-dimensional solution space of~\eqref{eq:linear_sytem}. Secondly, the solution sets $\mathcal{S}_1, \, \mathcal{S}_2$ for different outputs $o_1, \, o_2$ are just shifted versions of another. 

Hence, since infinitely many activations $z_{L-1}$ can generate~$o$, we want to modify~\eqref{eq:linear_sytem} such that we obtain reasonable and more likely activations. In particular, we can consider the expected activation generated by the generative model, which we estimate via Monte-Carlo
\begin{equation*}
    \bar{z}_{L-1} = \frac{1}{n} \sum_{i=1}^n \sigma_{L-1} (G_{L-1}(  \cdots ( \sigma_1 ( G_1 (z^{(i)}) \cdots )  \enspace  
\end{equation*}
with samples $z^{(i)}$ from the model's base distribution, that is, from $ \rvz \sim \mathcal{N}(0, I)$
and solve 
\begin{figure}[t]
    \centering \includegraphics[width=0.345\textwidth,trim=5 7 20 5,clip]{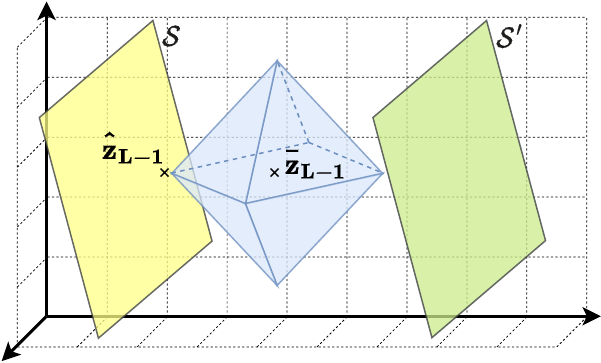}
    \caption{Geometry of the optimization problem~(\ref{eq:basis_pursuit}). According to Proposition~\ref{prop:linear_algebra_fund}, the solution sets $\mathcal{S}$ (yellow) and $\mathcal{S}^\prime$ (green) for outputs $o$ and $o^\prime$, respectively, are shifted versions of another. The solution $\hat{z}_{L-1}$ is the point where the smallest $\ell_1$-diamond (blue) around $\bar{z}_{L-1}$ touches the solution set $\mathcal{S}$. In particular, the second and third components of $\hat{z}_{L-1}$, i.e., the ones corresponding to the $y$- and $z$-axis, coincide with the component of $\bar{z}_{L-1}$.}
    \label{fig:geometry}
\end{figure} 
\begin{equation}
\label{eq:basis_pursuit}
\begin{split}
    \hat{z}_{L-1} \in \argmin_{z_{L-1} \in \mathbb{R}^{D_{L-1}}}  \Vert z_{L-1} - \bar{z}_{L-1} \Vert_1  \\ \text{subject to }   o = G_L (z_{L-1})   \enspace   .
\end{split}
\end{equation}
This optimization problem is a modification of the basis pursuit algorithm \citep{chen2001atomic}; in particular, it has a similar interpretation: \eqref{eq:basis_pursuit} returns solutions that are in the solution space $\mathcal{S}$ and are close to the average activation $\bar{z}_{L-1}$. The $\ell_1$-distance regularizes towards sparsity of $\hat{z}_{L-1} - \bar{z}_{L-1} $, i.e., towards many similar components of $\hat{z}_{L-1}$ and $\bar{z}_{L-1}$. We present a geometric illustration in Figure~\ref{fig:geometry}. In fact, exploiting the concept of sparsity is not new and has proven to be useful in theory and practice. For instance, \citet{parhi2021banach} have shown that deep neural networks are implicitly learning sparse networks, \citet{lederer2022statistical} has derived statistical guarantees for deep networks under sparsity, sparsity is a ubiquitous concept in high-dimensional statistics \citep{lederer2021fundamentals}, and it has proven successful in compressed sensing as well \citep{eldar2012compressed}.

Finally, due to numerical inaccuracies of computations in high-dimensional linear systems, we allow for some slack $\varepsilon > 0$ and  simplify~\eqref{eq:basis_pursuit} to a modification of the basis pursuit denoising \citep{chen2001atomic}
\begin{equation}
\begin{split}
    \label{eq:basis_pursuit_denoising}
     \hat{z}_{L-1} \in \argmin_{z_{L-1} \in \mathbb{R}^{D_{L-1}}} \quad \Vert z_{L-1} - \bar{z}_{L-1} \Vert_1  \\
    \quad \text{subject to } \quad  \Vert G_L (z_{L-1}) - o \Vert_2 \leq \varepsilon \enspace , 
\end{split}
\end{equation}
which is equivalent to the well-known (modified) lasso \citep{tibshirani1996regression} optimization problem, whose Lagrangian form is given by
\small
\begin{equation}
    \label{eq:opti_problem} 
    \hat{z}_{L-1} \in \argmin_{z_{L-1} \in \mathbb{R}^{D_{L-1}}} \Vert G_L (z_{L-1}) - o \Vert_2^2 + \lambda  \Vert z_{L-1} - \bar{z}_{L-1} \Vert_1 , 
\end{equation}
\normalsize
where $\lambda>0$ is a parameter that depends on $\varepsilon$. 

Re-examining the examples in Section~\ref{sec:motivationalExample}, we see that our proposed optimization problem~\eqref{eq:opti_problem} now combines the underlying motivations of all three examples: The reconstruction loss guarantees to provide solutions that approximately solve the linear system (Example~\ref{exa:impossible}) and the regularization towards the average activation $\bar{z}_{L-1}$ biases the solutions towards reasonable activations (Example~\ref{exa:unlikely}), which share a similar structure to $\bar{z}_{L-1}$ (Example~\ref{exa:structure}).
The rationale is that for an output $o^-$ that is not generated from $G$, we expect the estimated activations $\hat{z}$ necessary for approximately reconstructing $o^-$ using the model $G$ to be distinguishable from real activations of the model. 

While our inversion scheme is closely related to the inversion methods from Section~\ref{sec:related_work}, we want to emphasize that \method{} enjoys several advantages summarized in the following theorem. 
\begin{theorem}{(Informal)}\label{theo:informal}
    Assume that $G_L$ is a 2D-convolution with kernel weights sampled from a continuous distribution $\rvk$. Then,~\eqref{eq:opti_problem} satisfies the following properties: (i)~It is equivalent to a lasso optimization problem; (ii)~It has a unique solution with probability $1$;
    (iii)~Under certain moment assumption on $\rvk$ there exists some $S\in \mathbb{N}$ and constants $C_0, C_1$ such that $\Vert \hat{z} - z \Vert_2 \leq C_0 / \sqrt{S} \Vert z - z_S \Vert_1 + C_1 $ 
        if $\Vert z - \hat{z}\Vert_0 \leq S$, where $z_S$ is the best $S$-term approximation of $z$\footnote{See~\eqref{eq:sterm} in Section~\ref{sec:cs} and the modification to~\eqref{eq:opti_problem} in Section~\ref{sec:theory_modified}}. 
\end{theorem}
\noindent 
Before providing a proof sketch, let us first examine the implications of Theorem~\ref{theo:informal}: Property 1) allows us to use fast and computationally tractable lasso algorithms such as FISTA~\citep{beck2009fast}, which enjoys finite-sample convergence guarantees. According to 2), it is not required to solve the optimization problem on multiple seeds, as in standard inversion techniques, see Section~\ref{sec:related_work}. Finally, 3) might be useful from a theoretical perspective, but we want to highlight its limitations. The bound depends on $\Vert z - z_S \Vert_1$, which might be expected to be small in diverse problems in compressed sensing, but it is not necessarily small for activations of a generative model. Finally, we provide a brief description of the proof. All details can be found in~\cref{sec:theory} and~\ref{sec:unique}.
\begin{proof}
    (i) By a suitable zero-extension $0=-G_L(\bar{z}) + G_L(\bar{z})$ within the $\ell_2$-norm in~\eqref{eq:opti_problem} and by using the linearity of $G_L$, we end up with a usual lasso optimization problem, see Section~\ref{sec:optiislasso}. (ii) Lasso optimization problems have a unique solution if the columns of the design matrix are in general linear position~\citep{tibshirani2013lasso}. We show that this is the case for convolutions $G_L$ with probability $1$, see Section~\ref{sec:unique}. (iii) If $G_L$ satisfies the restricted isometry property~\citep{candes_rip}, then we can obtain bounds of the desired form for lasso optimization problems. We modify a result by \citet{haupt_toeplitz} to prove this property for 2D-convolutions, see Section~\ref{sec:2dconv}. 
\end{proof}

\begin{figure}
\centering
\subfigure
{\includegraphics[width=.14\linewidth]{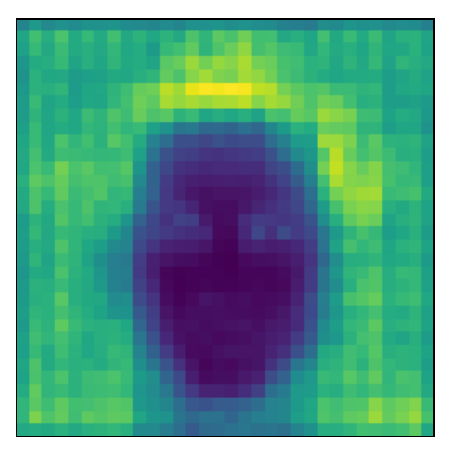}}
 \subfigure
 {\includegraphics[width=.14\linewidth]{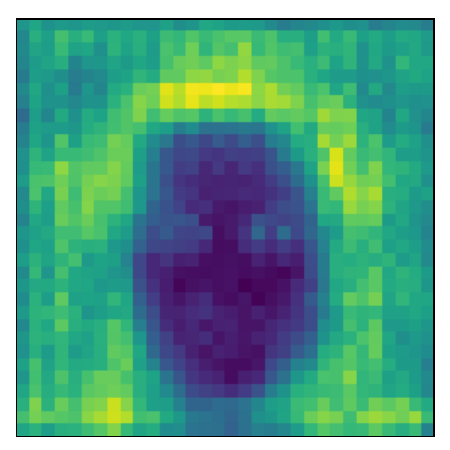}}
  \subfigure 
  {\includegraphics[width=.14\linewidth]{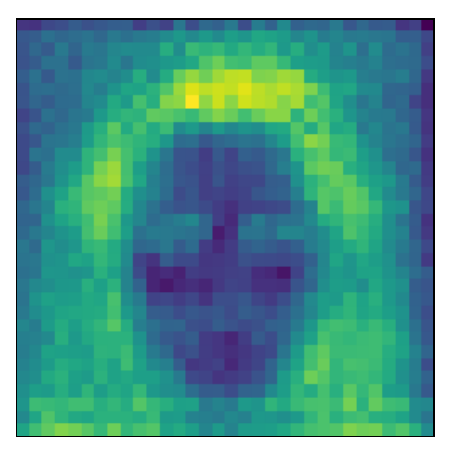}}
  \subfigure
  { \includegraphics[width=.14\linewidth]{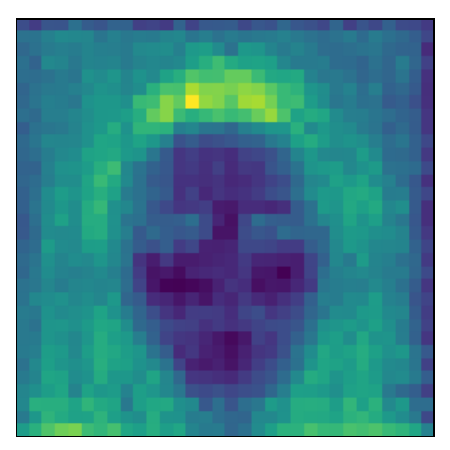}}
  \subfigure
  {\includegraphics[width=.14\linewidth]{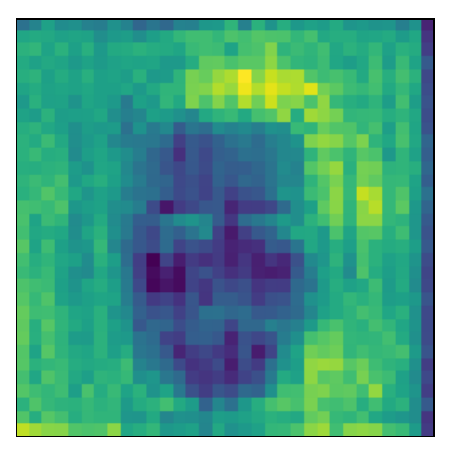}}
  \subfigure
  {\includegraphics[width=.14\linewidth]{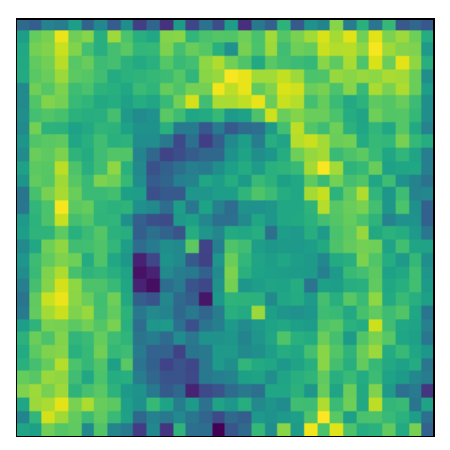}}
  
\caption{Cherry-picked channel dimension $c$ of the average reconstructed features according to (\ref{eq:opti_problem}) when $G$ is a DCGAN. The left-most figure shows the average activation $\bar{z}_c$ over channel $c$, and the remaining figures show the average feature taken over DCGAN, real, WGAN-GP, LSGAN, and EBGAN samples, respectively.}
\label{fig:features_dcgan}
\end{figure}



\section{Experiments}\label{sec:exp} \label{sec:exp_small} 
\begin{table*}[]
\small
\centering
\begin{tabular}{lccccrcccc}
\toprule
 & \multicolumn{4}{c}{CelebA} & & \multicolumn{4}{c}{LSUN} \\
 \cmidrule{2-5} \cmidrule{7-10}
 & DCGAN & WGAN-GP & LSGAN & EBGAN && DCGAN & WGAN-GP & LSGAN & EBGAN \\
\midrule

\finger & 64.81 &	50.18 &	53.57 &	95.78 &&	69.94& 	52.96 &	56.30 	&74.58 \\
\inv 	&98.52 	& \textbf{99.80}  &	62.55 &	\textbf{99.80} &&	51.44 &	50.20 &	65.46 	&54.35 \\
\incinv &	51.23 	&50.16 &	52.64 &	78.46 &&	50.61 &	50.35 &	51.61 &	52.12 \\
\raw &	94.95 &	59.12 &	\textbf{99.53} &	99.54 	&&67.14 &	52.36 &	\textbf{98.22} 	&93.96 \\
\dct 	&81.33 &	64.95 &	96.79 &	92.08&& 	80.48 	&68.10 &	67.93 &	90.40 \\
\method &	\textbf{99.34} &	99.40 &	98.94& 	99.68 	&& \textbf{97.75} &	\textbf{97.73} 	& 98.19 &	\textbf{99.38} \\ 
\bottomrule
\end{tabular}
\caption{Single-model attribution accuracy averaged over all $G'\in\mathcal{G}$ over five runs. A more detailed report is provided in Table~\ref{tab:confusion} in the Appendix. Note, due to their excessive computational load, we do not repeat the inversion methods multiple times.}
\label{tab:attribution_small}
\end{table*}

\begin{table*}[h]
\small
\centering
\begin{tabular}{lccccrcccc}
\toprule
 & \multicolumn{4}{c}{CelebA} & & \multicolumn{4}{c}{LSUN} \\
\cmidrule{2-5} \cmidrule{7-10}
 & DCGAN & WGAN-GP & LSGAN & EBGAN && DCGAN & WGANGP & LSGAN & EBGAN \\
\midrule
\raw & 95.77 &	61.49 &	\textbf{97.48} &	74.35& &	61.82& 	51.65 &	\textbf{92.71} &	\textbf{90.46} \\ 
\dct &	50.01 &	49.81 &	55.37 &	55.45 &	&49.84 &	51.03 &	51.44 &	50.96 \\
\method &	\textbf{99.31} 	& \textbf{97.88} &	69.75 &	\textbf{99.76} 	&& \textbf{98.04} 	& \textbf{97.52} 	& 88.84 	& 89.79 \\
\bottomrule
\end{tabular}
\caption{Average single-model attribution accuracy for the same model trained on different initialization seeds. Each score denotes the average accuracy over all $G' \in \mathcal{G}$ over five runs.}
\label{table:attr_seed}
\end{table*}

We demonstrate the effectiveness of the proposed methods---\raw{}, \dct{}, and \method{}---in a range of extensive experiments. We begin with smaller generative models trained on CelebA and LSUN, which allow a fine-grained evaluation in various setups. 
We extend the experiments to modern large-scale diffusion models and style-based models, to medical image generators, and to tabular models.
All experiments can be reproduced using our public code repository.\footnote{\url{https://github.com/MikeLasz/flipad}}

\subsection{Setup}
Throughout all experiments, we train on $n_{\text{tr}}$ labeled samples $(x_G , 1)$ generated by $G$, and on $n_{\text{tr}}$ labeled negative samples $(x_{\text{neg}}, -1)$ sampled either from the real set of images used to train $G$ or from another generative model $G'$. The latter case is reasonable in settings, in which i) it is not entirely clear on what data $G$ was trained on (e.g.\ in Stable Diffusion) or ii) when the real data is not available at all (e.g.\ in medical applications). We repeat all experiments five times. To evaluate the performance, we measure the classification accuracy over $n_{\text{test}}$ samples from $G$ and $G'$, respectively, where $G'\in \mathcal{G}\backslash \{ G\}$ for a set of generative models $\mathcal{G}$. We compare the proposed methods with fingerprinting (\finger{}) and inversion (\inv{} and \incinv{}) methods, which we adapt to the single-model setting (see Section \ref{sec:adapting}). 
All experimental and architectural details as well as further empirical results, containing standard deviations to all experiments and visualizations, are deferred to Sections~\ref{sec:exp_details} and \ref{sec:app_smallexp} in the Appendix.

\subsection{Results} \label{sec:exp_images}
We begin by considering a set $\mathcal{G}$ of small generative models, namely DCGAN~\citep{radford2015unsupervised}, WGAN-GP~\citep{gulrajani2017improved}, LSGAN~\citep{mao2017least}, EBGAN~\citep{zhao2016energy} trained on either the CelebA \citep{liu2015faceattributes} or the LSUN bedroom~\citep{yu2015lsun} dataset. 
\paragraph{Feature Extraction} 
As a first sanity check of the significance of the extracted features from \eqref{eq:opti_problem}, we present average features of images generated by different sources in Figure~\ref{fig:features_dcgan}. While the average reconstructed features from images generated by the model are very close to its actual average activation, the other images exhibit fairly different features. We provide further examples and investigate the robustness with respect to the parameter $\lambda$ in Section~\ref{sec:app_smallexp} of the Appendix. These results provide qualitative evidence that the extracted features are distinguishable and detectable by an anomaly detector, which we demonstrate in the next paragraph. 

\paragraph{Single-Model Attribution}
We present the single-model attribution performance of each method in Table~\ref{tab:attribution_small}. First of all, we see that the adapted methods are very unstable, their accuracy ranges from $50$\% to almost perfect attribution. Due to their inferior performance, yet high and restrictive computational costs of \inv{} and \incinv{} (see the runtime evaluation below), we exclude them from the remaining experiments on those datasets.  
Even though \raw{} and \dct{} are performing slightly better, their performance is not as consistent as \method{}'s, which achieves excellent results with an average attribution accuracy of over $97.5$\% in all cases.
In a second line of experiments, we evaluate the single-model attribution performance in a notably more difficult task: For a model $G$, we set $\mathcal{G}$ as the set of models with the exact same architecture and training data but initialized with a different seed.
We present the results in Table~\ref{table:attr_seed} and observe a similar pattern as in the previous experiments. Since \method{} involves the knowledge of the exact weights of $G$, we argue that it enables reliable model attribution even in the case of these subtle model variations.

\begin{table*}
    \centering
    \small 
    \begin{tabular}{@{\ }lc@{\hspace{1mm}}cc@{\hspace{1mm}}cc@{\hspace{1mm}}cc@{\hspace{1mm}}c@{\hspace{0mm}}
    cc@{\hspace{1mm}}cc@{\hspace{1mm}}cc@{\hspace{1mm}}cc@{\hspace{1mm}}c@{\ }
    }
    \toprule
    & \multicolumn{8}{c}{CelebA} && \multicolumn{8}{c}{LSUN}\\ 
    \cmidrule{2-9} \cmidrule{11-18}
    & \multicolumn{2}{c}{Blur} & \multicolumn{2}{c}{Crop} & \multicolumn{2}{c}{Noise} & \multicolumn{2}{c}{JPEG} 
    && \multicolumn{2}{c}{Blur} & \multicolumn{2}{c}{Crop} & \multicolumn{2}{c}{Noise} & \multicolumn{2}{c}{JPEG}
    \\
    & 1 & 3 & 60 & 55 & 0.05 & 0.1 & 90 & 80
    && 1 & 3 & 60 & 55 & 0.05 & 0.1 & 90 & 80\\ 
     \midrule 
     \raw &88.71 &	87.86& 	88.69 	&89.33 &	\textbf{84.09} &	\textbf{80.36} &	\textbf{87.05} &	\textbf{86.45} 	&&77.73 &	76.65 &	76.24 &	75.76 	& \textbf{73.47} &	\textbf{69.86} 	& \textbf{75.68} &	\textbf{74.87} \\ 
\dct 	&82.67 	&77.41 	&83.42 	&82.84 	&63.78 &	58.61 	&71.09 	&68.83 &&	76.91 &	65.71 &	74.70 &	73.78 &	51.43 &	50.43 	&55.20 	&52.77\\ 
\method{} &	\textbf{99.36} &	\textbf{98.34} 	& \textbf{98.25} 	&\textbf{98.27} & 	76.06& 	68.80 &	83.46 &	81.36 	&& \textbf{98.13} &	\textbf{97.21} 	& \textbf{94.64} &	\textbf{89.96} 	& 66.58& 	53.17 &	72.83 	&67.70\\
    \bottomrule
    \end{tabular}
    \caption{Single-model attribution accuracy with immunization averaged over all $G \in \mathcal{G},\; G' \in \mathcal{G}\backslash \{G\}$ over five runs. For blurring we consider kernel sizes of $1$ and $3$, for cropping we extract a center crop of $60\times60$ and $55\times55$ pixels followed by upsampling to $64\times64$ pixels. For noise we use standard deviations of $0.05$ and $0.01$, for JPEG we use quality factors of 90 and 80.} 
    \label{table:perturb1}
    \vspace{0pt}
\end{table*}
\begin{table*}
    \centering
    \small
    \begin{tabular}{@{\ }lccccc@{\hspace{1mm}}ccccc@{\hspace{1mm}}cc@{\ }}
        \toprule
        & \multicolumn{4}{c}{Stable Diffusion} && \multicolumn{4}{c}{Style-Based Models} && \multicolumn{2}{c}{Medical Image Models}
        \\ 
        \cmidrule{2-5} \cmidrule{7-10} \cmidrule{12-13}
         & v1-4 & v1-1 & v1-1+ & v2 & & \textit{real} & StyleGAN-XL & StyleNAT & StyleSwin && WGAN-GP & C-DCGAN \\
        \midrule
        \finger & 93.75 &  93.75 & \textbf{93.75} & \textbf{93.75} & & 99.60 &  99.60 & \textbf{99.60} & 99.60 && \textbf{99.81} & \textbf{99.81} \\ 
        \raw & 91.85 & 92.70 & 90.00 & 89.25 & & 55.00 & 54.17 & 53.62 & 51.84 && 99.73 & 78.76 \\ 
        \dct & \textbf{95.45} & \textbf{95.55} & 49.40 & 48.90  & & \textbf{99.93} & \textbf{99.93} & 99.54 & \textbf{99.85} && 99.73 & 71.49 \\ 
        \method & 92.08 & 92.65 & 90.55 & 91.30 && - & - & - & - && 99.75 & 99.63 \\ 
        \bottomrule
    \end{tabular}
    \caption{Single-model attribution accuracy for Stable Diffusion v2-1, StyleGAN2, and DCGAN trained on BCDR. Each score denotes the average accuracy against $G'$ as indicated by the column name over five runs.} 
    \label{tab:attribution_other_image}
    \vspace{0pt}
\end{table*}
\begin{table}[t]
    \centering
    \small
    \begin{tabular}{@{\ }lccc@{\ }}
        \toprule
         & TVAE & CTGAN & CopulaGAN \\
        \midrule
        \inv & 88.02 &  95.74 & 94.45  \\ 
        \raw & 90.61 &  96.60 & 94.93 \\
        \method & \textbf{91.18} &  \textbf{98.07} & \textbf{96.70} \\
        \bottomrule
    \end{tabular}
    \caption{Single-model attribution accuracy of KL-WGAN trained on Redwine against $G'$ as indicated by the column name averaged over five runs.} 
    \label{tab:redwine}
\end{table}
\begin{table*}[h]
    \centering
    \small
    \begin{tabular}{@{\ }lcccc@{\ }}
        \toprule
        & Essential Computation & CelebA ($1\,000$ samples) & LSUN ($1\,000$ samples) & StableDiffusion ($100$ samples) \\  
        \midrule
        \finger & Computing Fingerprints&	\phantom{+ 11}1.82&	\phantom{+ 11}1.77&	\phantom{+ 111}0.32 \\
        \inv & Inverting Samples	&\phantom{+ }300.75	&\phantom{+ }292.22	&\phantom{+ }4964.42\\ 
        \incinv & Inverting Samples&	\phantom{+ }452.60	&\phantom{+ }451.75	&\phantom{+ }5144.75 \\ 
        \raw & Training DeepSAD	&\phantom{+ 11}1.68	&\phantom{+ 11}1.72	&\phantom{+ 111}1.67 \\ 
        \dct & Computing DCT &	\phantom{+ 11}0.40 	&\phantom{+ 11}0.42&	\phantom{+ 111}0.03\\ 
        & + Training DeepSAD & \phantom{11}+ 1.78 & \phantom{11}+ 1.76 & \phantom{111}+ 1.68 \\ 
        \method{} & Final-Layer Inversion  &	\phantom{+ 1}17.48	&\phantom{+ 1}20.05 &\phantom{+ 11}11.63  \\
        & + Training DeepSAD & \phantom{11}+ 8.29 & \phantom{11}+ 8.21 & \phantom{111}+ 2.81 \\ 
        \bottomrule
    \end{tabular}
    \caption{Wall-clock run time of the essential computations in minutes.}
\label{tab:runtime}
    \vspace{0pt}
\end{table*}
\paragraph{Single-Model Attribution on Perturbed Samples}
To hinder model attribution, an adversary could perturb the generated samples. 
We investigate the attribution performance on perturbed samples in the immunized setting, i.e., we train the models on data that is modified by the same type of perturbation.
For the sake of simplicity, we average the attribution accuracies over all models and present the results 
in Table~\ref{table:perturb1}.  
For blurred and cropped images we observe superior performance of \method{} over \raw{} and \dct{}.
However, in the case of JPEG compression and the presence of random noise, we can see a performance drop of \method{}.
In contrast, those perturbations influence the performance of \raw{} only slightly. \ml{Table~\ref{table:perturb_higher_fnr} demonstrates how the performance is improved by allowing a more liberal false negative rate (see Section~\ref{sec:thresholding_method}).}

\paragraph{Stable Diffusion}
In this experiment, we evaluate that \method{} is effective against the powerful text-to-image model Stable Diffusion~\citep{rombach2022}.
Since multiple versions of Stable Diffusion have been released, an appropriate setting is to attribute images to a specific version.
In particular, we consider Stable Diffusion v2-1 to be our model, while the set of other generators consists of v1-1, v1-1+ (similar to v1-1 but with a different autoencoder\footnote{\url{https://huggingface.co/stabilityai/sd-vae-ft-mse}}), v1-4, and v2.
In contrast to previous experiments, we use images from v1-4 instead of real images for training, since it is not entirely clear on which data the model was trained. 
The results are provided in Table~\ref{tab:attribution_other_image}.
All methods can distinguish between v2-1 and v1-4, which they are trained on. 
Given that v1-1 and v1-4 share the exact same autoencoder, this is less surprising.
While \dct{} fails to generalize to other generative models, the other approaches achieve high attribution accuracies above $90\%$ across all models.  

\paragraph{Style-based Generative Models}
In the next experiment, we analyze the single-model attribution of style-based generative models trained on FFHQ~\cite{karras2019style}. 
Specifically, $\mathcal{G}$ consists of StyleGAN2~\cite{karras2021alias}, StyleGAN-XL~\cite{sauer2022stylegan}, StyleNAT~\cite{walton2022stylenat}, StyleSwin~\cite{zhang2022styleswin}, and of FFHQ images. We consider StyleGAN2 to be our model $G$. Note that the skip-connections in StyleGAN2 do not prohibit but complicate the application of \method{} considerably.\footnote{We discuss this challenge in more detail in Section~\ref{sec:skip} of the Appendix.} 
Consequently, we omit its application in this setting and shift our attention to \dct{}. It is known that StyleGANs possess distinctive DCT artifacts~\cite{frank2020leveraging}, which \dct{} apparently benefits from, as demonstrated by the attribution accuracy in~\cref{tab:attribution_other_image}. The performance is on par with \finger{}. 

\paragraph{Generative Models for Medical Image Data}
In contrast to the models in~\cref{sec:exp_images}, generative models for medical images are facing very different challenges. Most importantly, medical data is typically much more scarce and the images are contextually different~\cite{alyafi2020dcgans, varoquaux2022machine}, which is why they require a separate treatment.
We set $\mathcal{G}$ to consist of a DCGAN, a WGAN-GP, and a c-DCGAN~\cite{szafranowska2022sharing, osuala2023medigan} trained on the BCDR~\cite{Lopez2012BCDRA}, which contains $128\times 128$ breast mammography images. Since access to the original dataset is restricted, we train the models on samples from $\mathcal{G}$ and from WGAN-GP. 
We report the single-model attribution accuracies in Table~\ref{tab:attribution_other_image} and observe that both, \raw{} and \dct{}, are performing well on attributing samples from WGAN-GP but fail to generalize to C-DCGAN. Both, \finger{} and \method{}, achieve close to perfect performance. 

\paragraph{Generative Models for Tabular Data}
Many model attribution methods are specifically tailored toward high-dimensional image data, such as \finger{}, making use of well-studied image processing pipelines. 
However, to the best of our knowledge, there is not a single work conducting model attribution for tabular data, which we address in the following experiment. 
We set $\mathcal{G}$ to be the set consisting of KL-WGAN \cite{song2020bridging}, CTGAN, TVAE, and CopulaGAN~\cite{ctgan} trained on the Redwine dataset~\cite{wine}.
Since the real dataset is small, we train the model attribution methods on samples from $\mathcal{G}$ and from TVAE. As displayed in Table~\ref{tab:redwine}, \method{} achieves the highest attribution performance for all $G'$.

\paragraph{Runtime Comparison}
Finally, we summarize the most relevant computational procedures and their corresponding wall-clock time in Table~\ref{tab:runtime}.
To highlight the computational requirements of full inversion (\inv{} and \incinv), we also applied it to Stable Diffusion. Due to its computational constraints, we inverted only 12 samples\footnote{we could fit only a batch of size 4 on an NVIDIA A40 GPU} using \inv{} and \incinv. Inverting these 12 samples already took $595.73$ and $617.37$ minutes, respectively. The values presented in the above table are the inferred values for 100 samples, e.g., $4964.42 \approx 595.73 / 12 \cdot 100$. All computations were conducted on an NVIDIA A40 GPU.
While our proposed methods are not as efficient as \finger, they are still fast and significantly more efficient than those based on full inversion.

\paragraph{Discussion} In summary, 
we conclude that, first, \inv{} and \incinv{} perform well only in very few settings and their computational load restricts their practicality considerably. Second, \finger{} achieves excellent performance for high-dimensional images but fails for low-dimensional images ($64\times 64$). Furthermore, its application is bound to the image domain. 
Third, \raw{} and \dct{} are simple baselines that work decently in most settings but tend to generalize worse to the open-world setting. However, for the style-based generators, \dct{} achieves excellent results, even for unseen generative models. 
Lastly, \method{} is capable of adapting to various settings and performs best, or only slightly worse, than competing methods.


\section{Conclusion} 
This work tackles a neglected but pressing problem: determining whether a sample was generated by a specific model or not.
We associate this task, which we term single-model attribution, with anomaly detection---a connection that has not been drawn previously.
While applying anomaly detection directly to the images performs reasonably well, the performance can be further enhanced by incorporating knowledge about the generative model.    
More precisely, we develop a novel feature-extraction technique based on final-layer inversion.  
Given an image, we reconstruct plausible activations before the final layer under the premise that it was generated by the model at hand.
This reconstruction task can be reduced to a lasso optimization problem, which, unlike existing methods, can be optimized efficiently, globally, and uniquely. 
Beyond these beneficial theoretical properties, our approach is not bound to a specific domain and achieves excellent empirical results on a variety of different generative models, including GANs and diffusion models. 

\section*{Acknowledgements}
We thank Somnath Chakraborty for reviewing our paper and pointing out typos in an older preprint version. Furthermore, we thank all the anonymous reviewers, whose valuable feedback helped improve our work. 
This work was supported by the Deutsche Forschungsgemeinschaft
(DFG, German Research Foundation) under Germany’s Excellence
Strategy – EXC 2092 CASA – 390781972.
\section*{Impact Statement}
Generative AI has the potential to cause significant ethical and societal harm. As the past has revealed, a myriad of misuses of generative models exist, including deep fake generation~\citep{deepfake_survey} to defame individuals~\citep{deepfake_porn}, the spread of fake speech~\citep{howcroft2018faking} and disinformation and propaganda~\citep{genai_disinformation} to manipulate political campaigns~\citep{ chesney2019deep2, chesney2019deepfakes}.
As generative models become more accessible and efficient, the risk of large-scale generation of harmful content intensifies.

In light of these threats, diminishing the spread of misinformation and identifying the sources of misuse should be considered as one of the main challenges of future research in this field.
Our work is dedicated to contributing to this effort, making a valuable step towards the direction of accountable and responsible generative AI systems and their monitoring.

\bibliography{main}
\bibliographystyle{icml2024}


\newpage
\appendix
\onecolumn

\section{Summary of the Related Work}\label{app:related_work}
We complement the summary provided in Section~\ref{sec:related_work} with an overview given in Table~\ref{tab:related_work:overview}.

\begin{table*}[!htbp]
    \centering
    \begin{tabular}{llll} \toprule
                                                & Method                                        & Target        & Setting       \\ \midrule
         \multirow{5}{*}{Fingerprinting}        & \citet{marraGANsLeaveArtificial2019}          & multi-model   & closed-world  \\
                                                & \citet{yuAttributingFakeImages2019}           & multi-model   & closed-world  \\
                                                & \citet{xuanScalableFinegrainedGenerated2019}  & multi-model   & closed-world  \\
                                                & \citet{marra2019incremental}                  & multi-model   & closed-world  \\
                                                & \citet{frank2020leveraging}                   & multi-model   & closed-world  \\ \midrule
                                            
         \multirow{4}{*}{Inversion}             & \citet{albright2019source}                    & multi-model   & closed-world  \\
                                                & \citet{karrasAnalyzingImprovingImage2020}     & single-model  & closed-world  \\
                                                & \citet{zhangAttributionDeepfakes2021}         & multi-model   & closed-world  \\
                                                & \citet{hirofumi2022}                          & both          & both          \\ \midrule
        
         \multirow{5}{*}{Watermarking}          & \citet{yu2021artificial}                      & single-model  & open-world    \\
        & \citet{kim2021decentralized}                  & single-model  & open-world    \\
        & \citet{yu2020responsible}                     & single-model  & open-world    \\ 
        & \citet{wen2023tree}  & single-model  & open-world    \\ 
         & \citet{nie2023attributing}  & single-model  & open-world    \\ 
        \midrule
                                                
         Ours                                   &                                               & single-model  & open-world    \\ \bottomrule
    \end{tabular}
    \caption{Classification of the related work.}
    \label{tab:related_work:overview}
\end{table*}

\section{Deep Semi-Supervised Anomaly Detection}\label{sec:deep_sad}
In anomaly detection, we seek to identify samples that are most likely not generated by a given generative process. We call such samples out-of-distribution samples. But in contrast to classical binary classification procedures, anomaly detection methods strive for bounded decision areas for the in-distribution samples, while logistic regression for instance yields unbounded decision areas.
For example, SVDD \citep{Tax2004SupportVD} maps the data into a kernel-based feature space and finds the smallest hypersphere that encloses the majority of the data. Samples outside that hypersphere are then deemed as out-of-distribution samples. Recently, \cite{pmlr-v80-ruff18a} replaced the kernel-based feature space with a feature space modeled by a deep network, which was further extended to the semi-supervised setting in \cite{ruff2019deep}. More specifically, given $n$ unlabeled samples $x_1, \dots , x_n$, and $m$ labeled samples $(\tilde{x}_1, \tilde{y}_1), \dots  ,(\tilde{x}_m, \tilde{y}_m)$, where $\tilde{y}=1$ denotes an inlier sample, and $\tilde{y}=-1$ denotes an outlier sample, \citet{ruff2019deep} seek to find a feature transformation function $\phi$ parameterized by weights $\mathcal{W}=(W^1, \dots W^L)$. The corresponding optimization problem is defined as 
\begin{align}
    \mathcal{W}^\ast \in \argmin_{\mathcal{W}} &\frac{1}{n+m} \sum_{i=1}^n \Vert \phi(x_i; \mathcal{W}) - c \Vert_2^2 \nonumber \\ 
    +&\frac{\eta}{n+m} \sum_{j=1}^n \bigl( \Vert \phi(\tilde{x}_j; \mathcal{W}) - c \Vert_2^2 \bigr)^{\tilde{y_j}} \\ 
    + &\frac{\lambda}{2} \sum_{l=1}^L \Vert W^l\Vert_F^2 \enspace \nonumber ,  \label{eq:deepsad_opti}
\end{align}
where $\eta>0$ is a hyperparameter balancing the importance of the labeled samples, and $c$ is a point in the feature space representing the center of the hypersphere. 
Intuitively, out-of-distribution samples $x$ are mapped further away from $c$, resulting in a large $\ell_2$-distance $\Vert \phi(x; \mathcal{W}^\ast) - c \Vert_2$ and AD is performed by fixing a threshold $\tau> 0$ to classify 
\begin{equation}
    \hat{y} = \begin{cases}
    1 \;&, \text{ if }  \Vert \phi(x; \mathcal{W}^\ast) - c \Vert_2 \leq \tau   \\
    -1 \;&, \text{ if } \Vert \phi(x; \mathcal{W}^\ast) - c \Vert_2 > \tau   
    \end{cases}\enspace . \label{eq:deepsadscore}
\end{equation}
We explain how we tuned $\tau$ in Section~\ref{sec:thresholding_method}. 

While DeepSAD has been applied to typical anomaly detection tasks, it has never been applied for attributing samples to a generative model. Moreover, in our setting described in Section~\ref{sec:problem_setup}, we assume to have only access to samples generated by $G$ and to real samples. Hence, DeepSAD reduces to a supervised anomaly detection method in this case.

\section{Recovery Guarantees for the Optimization Problem~\ref{eq:opti_problem}} \label{sec:theory}
In this section, we equip the proposed optimization problem~\eqref{eq:opti_problem} with theoretical recovery guarantees of the activation for random 2D-convolutions\footnote{Transposed 2D-convolutions can be treated similarly.}.
First, we show in Section~\ref{sec:optiislasso} the equivalence of~\eqref{eq:opti_problem} to a standard lasso-problem. 
Hence, we can draw a connection to the theory of compressed sensing, which we briefly review in Section~\ref{sec:cs}. We show how 2D-convolutions originate from deleting rows and columns of a Toeplitz matrix in Section~\ref{sec:2dconv}, which motivates a proof for the restricted isometry property for 2D-convolutions. Finally, we generalize the recovery bounds for the classical lasso to our modified lasso problem in Section~\ref{sec:theory_modified}.

\subsection{\eqref{eq:opti_problem} is equivalent to a Lasso-Problem} \label{sec:optiislasso}
We can rewrite~\eqref{eq:opti_problem} as follows: 
\begin{align*}
    \hat{z}_{L-1} &= \argmin_{z_{L-1} \in \mathbb{R}^{D_{L-1}}} \Vert G_L (z_{L-1}) -o \Vert_2^2 + \lambda  \Vert z_{L-1} - \bar{z} \Vert_1 \\ 
    &= \argmin_{z_{L-1} \in \mathbb{R}^{D_{L-1}}} \Vert G_L (z_{L-1}) - G_L(\bar{z}) + G_L (\bar{z}) -o \Vert_2^2 \\
    &\hspace{2cm} + \lambda  \Vert z_{L-1} - \bar{z} \Vert_1 \\ 
    &=  \argmin_{z_{L-1} \in \mathbb{R}^{D_{L-1}}} \Vert G_L( z_{L-1} - \bar{z}) - ( o - G_L(\bar{z}) ) \Vert_2^2 \\ 
    &\hspace{2cm} + \lambda \Vert z_{L-1} - \bar{z} \Vert_1  \enspace .
\end{align*}
Furthermore, since 
\begin{align*}
    &\min_{z_{L-1} \in \mathbb{R}^{D_{L-1}}} \Vert G_L( z_{L-1} - \bar{z}_{L-1}) -  o^\prime \Vert_2^2 + \lambda \Vert z_{L-1} - \bar{z}_{L-1} \Vert_1  \\ = &\min_{z_{L-1} \in \mathbb{R}^{D_{L-1}}} \Vert G_L( z_{L-1}) - o^\prime \Vert_2^2 + \lambda \Vert z_{L-1}\Vert_1 \enspace ,
\end{align*}
for $o^\prime:= o - G_L(\bar{z})$, we can recover $\hat{z}_{L-1} = \hat{z}_{L-1}^\prime + \bar{z}_{L-1}$, where 
\begin{equation}\label{eq:opti_lasso}
     \hat{z}_{L-1}^\prime = \argmin_{z_{L-1} \in \mathbb{R}^{D_{L-1}}} \Vert G_L( z_{L-1}) - o^\prime \Vert_2^2 + \lambda \Vert z_{L-1}\Vert_1 \enspace . 
\end{equation}
Hence, to solve~\eqref{eq:opti_problem}, we can solve~\eqref{eq:opti_lasso} using standard lasso-algorithms such as FISTA \citep{beck2009fast}. 
Additionally, note that the optimization problem is convex and satisfies finite-sample convergence guarantees, see \citet{beck2009fast}. 

\subsection{Background on Compressed Sensing}\label{sec:cs} 
Having observed an output $o \in \mathbb{R}^{D_{\operatorname{out}}}$, the goal of compressed sensing is the recovery of a high-dimensional but structured signal $z^\ast \in \mathbb{R}^{D_{\operatorname{in}}}$ with $D_{\operatorname{in}}>>D_{\operatorname{out}}$ that has produced $o=\Phi(z^\ast)$ for some transformation $\Phi: \mathbb{R}^{D_{\operatorname{in}}} \rightarrow \mathbb{R}^{D_{\operatorname{out}}}$. 
In its simplest form, we define $\Phi$ as a linear function that characterizes an underdetermined linear system
\begin{equation*}
    o = G z^\ast 
\end{equation*}
where $G\in\mathbb{R}^{D_{\operatorname{out}} \times D_{\operatorname{in}}}$ is the defining matrix of $\Phi$. $G$ is usually referred to as the measurement matrix. As stated in Proposition~\ref{prop:linear_algebra_fund}, there usually exists a vector space of solutions $\{ z \in \mathbb{R}^{D_{\operatorname{in}}}\vert o = Gz \}$, and hence, signal recovery seems impossible. However, under certain assumptions on the true (but unknown) signal $z^\ast$ and measurement matrix $G$, we can derive algorithms that are guaranteed to recover $z^\ast$, which is the subject of compressed sensing. In this section, we review some results from compressed sensing for sparse signals along the lines of the survey paper by \citet{eldar2012compressed}.

One fundamental recovery algorithm for sparse signals is the basis pursuit denoising \citep{chen2001atomic}
\begin{equation}
\label{eq:bpd_basic}
    \min_{z} \Vert z \Vert_1 \; \text{subject to} \; \Vert G(z) - o \Vert_2 \leq \varepsilon \enspace 
\end{equation}
for some noise level $\varepsilon>0$.  
Next, we define one set of assumptions on $z^\ast$ and $G$---namely sparsity of $z^\ast$ and the restricted isometry property of $G$---that allow the recovery of $z^\ast$ using~\eqref{eq:bpd_basic}. 

\begin{defi}
    For $S\in \mathbb{N}$, we call a vector $z\in\mathbb{R}^{D_{\operatorname{in}}}$ $S$-sparse, if 
    \begin{equation*}
        \Vert z \Vert_0 := \# \{j \in \{1, \dots , D_{\operatorname{in}}\} : z_j \neq 0 \}  \leq S \enspace . 
    \end{equation*}
    We denote the set of all $S$-sparse vectors by $\Sigma_S \subset \mathbb{R}^{D_{\operatorname{in}}}$.
\end{defi}

\begin{defi}\label{def:rip}
    Let $G$ be a $D_{\operatorname{out}} \times D_{\operatorname{in}}$ matrix. We say $G$ has the restricted isometry property of order $S\in \mathbb{N}$, if there exists a $\delta_S \in (0, 1)$ such that 
    \begin{equation} \label{eq:RIP}
        (1- \delta_S) \Vert z \Vert_2^2 \leq \Vert Gz \Vert_2^2 \leq (1+\delta_S) \Vert z \Vert_2^2 \enspace  \forall z \in \Sigma_S \enspace . 
    \end{equation}
    In that case, we say that $G$ is $(S, \delta_S)$-RIP. 
\end{defi}
To provide some intuition to that definition, note that $G$ is $(S, \delta_S)$-RIP if 
\begin{equation*}
    \Vert G z \Vert_2^2 = z^\top G^\top G z \approx \Vert z \Vert_2^2 = z^\top z \enspace \forall z \in \Sigma_S \enspace ,
\end{equation*}
where the approximation is due to the bounds in~\eqref{eq:RIP}. 
Hence, the RIP is, loosely speaking, related to the condition $z^\top G^\top G z \approx z^\top z$ for $S$-sparse vectors $z$, i.e., $G$ behaves like an orthogonal matrix on vectors $z \in \Sigma_S$. In fact, the restricted isometry property can be interpreted as a generalized relaxation of orthogonality for non-square matrices. 

Lastly, let us define $z_S$ as the best $S$-term approximation of $z^\ast$, i.e., 
\begin{equation} \label{eq:sterm}
    z_S = \argmin_{z \in \Sigma_S} \Vert z^\ast - z \Vert_1 \enspace .
\end{equation}

\begin{theorem}[Noisy Recovery \citep{candes_rip}]\label{theo:candes2}
    Assume that $G$ is $(2S, \delta_{2S})$-RIP with $\delta_{2S} < \sqrt{2} - 1$ and let $\hat{z}$ be the solution to~\eqref{eq:bpd_basic} with a noise-level $\Vert \varepsilon \Vert_2 \leq E$. 
    Then, there exist constants $C_0, C_1$ such that 
    \begin{equation*}
        \Vert \hat{z} - z^\ast \Vert_2 \leq \frac{C_0}{\sqrt{S}} \Vert z^\ast - z_S
        \Vert_1 + C_1 E    \enspace . 
    \end{equation*}
\end{theorem}

This fundamental theorem allows deriving theoretical recovery bounds for sparse signals under the RIP-assumption on $G$. Unfortunately, the combinatorial nature\footnote{The assumption of a certain behavior for $S$-sparse inputs can be interpreted as an assumption on all $D_{\operatorname{out}} \times S$ submatrices of $G$.} of the RIP makes it difficult to certify. In fact, it has been shown that certifying RIP is NP-hard \citep{bandeira2013certifying}. However, it can be shown that random matrices drawn from a distribution satisfying a certain concentration inequality are RIP with high probability \citep{Baraniuk2008ASP}. Those random matrices include, for instance, the set of Gaussian random matrices, which could be used to model random fully-connected layers.

In the remainder of Section~\ref{sec:theory} we will deal with the derivation of the RIP for random 2D-convolutions (Section~\ref{sec:2dconv}) and a modified version of Theorem~\ref{theo:candes2} for the proposed optimization problem~\eqref{eq:opti_problem} (Corollary~\ref{cor:2dconv_recovery} and Section~\ref{sec:theory_modified}). 

\subsection{Recovery Guarantees for 2D-Convolutions} \label{sec:2dconv}
To the best of our knowledge, there exist RIP-results for random convolutions, however, they are restricted to unsuited matrix distributions, such as Rademacher entries \cite{krahmer_rip_convoltions}, and/or do not generalize to 2D-convolutions in their full generality \cite{haupt_toeplitz}, including convolutions with multiple channels, padding, stride, and dilation. 
In this section, we prove that randomly-generated 2D-convolutions indeed obey the restricted isometry property with high probability, which allows the application of recovery guarantees as provided by Theorem~\ref{theo:candes2}. 
We begin by viewing 2D-convolutions as Toeplitz matrices with deleted rows and columns. In fact, this motivates a proof strategy for the RIP of 2D-convolutions by recycling prior results \citep{haupt_toeplitz}.

\begin{defi}\label{def:toeplitz}
    Let $G\in \mathbb{R}^{D_{\operatorname{out}}\times D_{\operatorname{in}}}$ be a matrix given by 
    \begin{equation*}
        G := \begin{pmatrix}
            g_0 & g_{-1} & g_{-2} & \dots & \dots & g_{-(D_{\operatorname{in}} - 1)} \\ 
            g_1 & g_0 & g_{-1} & \ddots & & \vdots  \\ 
            g_2 & g_1 & \ddots & \ddots & \ddots & \vdots \\ 
            \vdots & \ddots & \ddots & \ddots & g_{-1} & g_{-2} \\ 
            \vdots &  & \ddots & g_1 & g_0 & g_{-1} \\ 
            g_{(D_{\operatorname{out}} - 1)} & \dots & \dots & g_2 & g_{1} & g_{0} \\              
        \end{pmatrix} \enspace . 
    \end{equation*}
    Then, we call $G$ the Toeplitz matrix generated by the sequence $\{g_i\}_{i\in I}$ for 
        $I:=\{-(D_{\operatorname{in}} - 1), \dots, 0, \dots, D_{\operatorname{out}} - 1\}$.
\end{defi}

We begin by illustrating the relationship between Toeplitz matrices and standard 2D-convolutions along the example depicted in Figure~\ref{fig:conv_a}. 
Algorithmically, we can construct the matrix $G$ by performing the following steps. First, we begin with constructing the Toeplitz matrix $T(k)$ generated by the sequence 
\begin{align*}
    &g_0 = k_{0,0}, \; g_{-1}=k_{0, 1}, \; g_{-4} = k_{1,0}, \; g_{-5} = k_{1,1}, 
    \\ &g_j=0 \text{ for } j \not\in \{ 0, -1, -4, -5\} \enspace .
\end{align*}
Secondly, we delete any row that corresponds to a non-conformal 2D-convolution, such as the convolution applied on the values $x_{0,3}, \; x_{1,3}, \; x_{1,0}, \;  x_{2, 0}$, as illustrated in the last row in Figure~\ref{fig:conv_a}. Similarly, we can implement strided convolutions by deleting rows of $T(k)$. 
Dilation is implemented by dilating $T(k)$, i.e., padding zeros in the sequence that generates $T(k)$ and multi-channel 2D-convolutions are implemented by simply stacking one-channel matrices $G_{\operatorname{channel} \, i}$ in a row-wise fashion.   
Padded 2D-convolutions can be implemented by either padding zeros to the input, see Figure~\ref{fig:conv_f}, or by deleting the corresponding columns of $G$ without altering the input. In summary, we can construct a 2D-convolution by generating a Toeplitz matrix from the sequence of zero-padded\footnote{the exact padding is given by the number of channels and the dilation} kernel parameters, from which we delete $n_r$ rows and $n_c$ columns, where $n_r$ is given by the input dimensions and the stride, whereas $n_c$ is given by the padding. 

\definecolor{xred}{RGB}{248,206,204}
\definecolor{xgreen}{RGB}{213,232,212}
\definecolor{xblue}{RGB}{218,232,252}
\definecolor{xyellow}{RGB}{255,242,204}

\newcommand{\rbg}[1]{\multicolumn{1}{>{\columncolor{xred}[1pt]}c}{#1}}
\newcommand{\gbg}[1]{\multicolumn{1}{>{\columncolor{xgreen}[1pt]}c}{#1}}
\newcommand{\bbg}[1]{\multicolumn{1}{>{\columncolor{xblue}[1pt]}c}{#1}}
\newcommand{\ybg}[1]{\multicolumn{1}{>{\columncolor{xyellow}[1pt]}c}{#1}}
\setcounter{MaxMatrixCols}{20}

\begin{figure*}
    \centering
    \subfigure{
        \parbox{.23\linewidth}{
            \includegraphics[width=100pt]{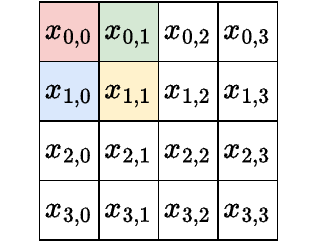}
        }
    }
    \subfigure{
        $
        \scriptsize
        \begin{bmatrix}
            \rbg{k_{0, 0}} & \gbg{k_{0, 1}} & 0 & 0 & \bbg{k_{1, 0}} & \ybg{k_{1, 1}} & 0 & 0 & 0 & 0& \cdots\\
            0 & k_{0, 0} & k_{0, 1} & 0 & 0 & k_{1, 0} & k_{1, 1} & 0 & 0 & 0&\cdots\\
            0 & 0 & k_{0, 0} & k_{0, 1} & 0 & 0 & k_{1, 0} & k_{1, 1} & 0 & 0 & \cdots\\
            0 & 0 & 0 & 0 & k_{0, 0} & k_{0, 1} & 0 & 0 & k_{1, 0} & k_{1, 1} & \cdots\\
            \vdots & \vdots & \vdots & \vdots & \vdots & \vdots & \vdots & \vdots & \vdots & \vdots & \ddots\\
        \end{bmatrix}
        \begin{bmatrix}
            \rbg{x_{0, 0}}\\
            \gbg{x_{0, 1}}\\
            x_{0, 2}\\
            x_{0, 3}\\
            \bbg{x_{1, 0}}\\
            \ybg{x_{1, 1}}\\
            x_{1, 2}\\
            x_{1, 3}\\
            \vdots\\
        \end{bmatrix}
        $
    }
    \subfigure{
        \parbox{.23\linewidth}{
            \includegraphics[width=100pt]{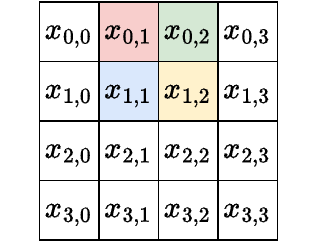}
        }
    }
    \subfigure{
        $
        \scriptsize
        \begin{bmatrix}
            k_{0, 0} & k_{0, 1} & 0 & 0 & k_{1, 0} & k_{1, 1} & 0 & 0 & 0 & 0 & \cdots\\
            0 & \rbg{k_{0, 0}} & \gbg{k_{0, 1}} & 0 & 0 & \bbg{k_{1, 0}} & \ybg{k_{1, 1}} & 0 & 0 & 0 & \cdots\\
            0 & 0 & k_{0, 0} & k_{0, 1} & 0 & 0 & k_{1, 0} & k_{1, 1} & 0 & 0 & \cdots\\
            0 & 0 & 0 & 0 & k_{0, 0} & k_{0, 1} & 0 & 0 & k_{1, 0} & k_{1, 1} & \cdots\\
            \vdots & \vdots & \vdots & \vdots & \vdots & \vdots & \vdots & \vdots & \vdots & \vdots & \ddots\\
        \end{bmatrix}
        \begin{bmatrix}
            x_{0, 0}\\
            \rbg{x_{0, 1}}\\
            \gbg{x_{0, 2}}\\
            x_{0, 3}\\
            x_{1, 0}\\
            \bbg{x_{1, 1}}\\
            \ybg{x_{1, 2}}\\
            x_{1, 3}\\
            \vdots\\
        \end{bmatrix}
        $
    }
    \subfigure{
        \parbox{.23\linewidth}{
            \includegraphics[width=100pt]{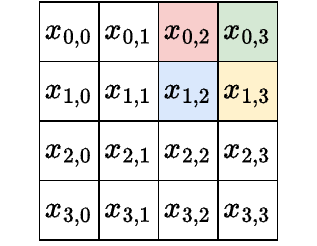}
        }
    }
    \subfigure{
        $
        \scriptsize
        \begin{bmatrix}
            k_{0, 0} & k_{0, 1} & 0 & 0 & k_{1, 0} & k_{1, 1} & 0 & 0 & 0 & 0 & \cdots\\
            0 & k_{0, 0} & k_{0, 1} & 0 & 0 & k_{1, 0} & k_{1, 1} & 0 & 0 & 0 & \cdots\\
            0 & 0 & \rbg{k_{0, 0}} & \gbg{k_{0, 1}} & 0 & 0 & \bbg{k_{1, 0}} & \ybg{k_{1, 1}} & 0 & 0 & \cdots\\
            0 & 0 & 0 & 0 & k_{0, 0} & k_{0, 1} & 0 & 0 & k_{1, 0} & k_{1, 1} & \cdots\\
            \vdots & \vdots & \vdots & \vdots & \vdots & \vdots & \vdots & \vdots & \vdots & \vdots & \ddots\\
        \end{bmatrix}
        \begin{bmatrix}
            x_{0, 0}\\
            x_{0, 1}\\
            \rbg{x_{0, 2}}\\
            \gbg{x_{0, 3}}\\
            x_{1, 0}\\
            x_{1, 1}\\
            \bbg{x_{1, 2}}\\
            \ybg{x_{1, 3}}\\
            \vdots\\
        \end{bmatrix}
        $
    }
    \subfigure{
        \parbox{.23\linewidth}{
            \includegraphics[width=100pt]{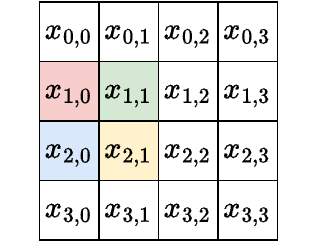}
        }
    }
    \subfigure{
        $
        \scriptsize
        \begin{bmatrix}
            k_{0, 0} & k_{0, 1} & 0 & 0 & k_{1, 0} & k_{1, 1} & 0 & 0 & 0 & 0 & \cdots \\
            0 & k_{0, 0} & k_{0, 1} & 0 & 0 & k_{1, 0} & k_{1, 1} & 0 & 0 & 0 &\cdots\\
            0 & 0 & k_{0, 0} & k_{0, 1} & 0 & 0 & k_{1, 0} & k_{1, 1} & 0 & 0 & \cdots\\
            0 & 0 & 0 & 0 & \rbg{k_{0, 0}} & \gbg{k_{0, 1}} & 0 & 0 & \bbg{k_{1, 0}} & \ybg{k_{1, 1}} & \cdots\\
            \vdots & \vdots & \vdots & \vdots & \vdots & \vdots & \vdots & \vdots & \vdots & \vdots & \ddots\\
        \end{bmatrix}
        \begin{bmatrix}
            \vdots\\
            \rbg{x_{1, 0}}\\
            \gbg{x_{1, 1}}\\
            x_{1, 2}\\
            x_{1, 3}\\
            \bbg{x_{2, 0}}\\
            \ybg{x_{2, 1}}\\
            x_{2, 2}\\
            x_{2, 3}\\
            \vdots\\
        \end{bmatrix}
        $ 
    }
\caption{Convolutional arithmetic as matrix multiplication. Each row shows one convolutional operation. Left: Conventional illustration of a 2D-convolution. Right: 2D-convolution as matrix multiplication $G\cdot x $.}
\label{fig:conv_a}
\end{figure*}

\begin{figure*}[t]
    \centering
    \subfigure{
        \parbox{.23\linewidth}{
            \centering
            \includegraphics[width=90pt]{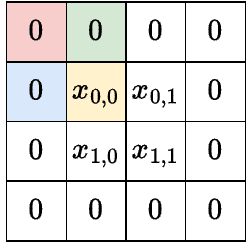}
        }
    }
    \subfigure{
        $
        \scriptsize
        \begin{bmatrix}
            \rbg{k_{0, 0}} & \gbg{k_{0, 1}} & 0 & 0 & \bbg{k_{1, 0}} & \ybg{k_{1, 1}} & 0 & 0 & \cdots\\
            0 & k_{0, 0} & k_{0, 1} & 0 & 0 & k_{1, 0} & k_{1, 1} & 0 & \cdots\\
            0 & 0 & k_{0, 0} & k_{0, 1} & 0 & 0 & k_{1, 0} & k_{1, 1} & \cdots\\
            \vdots & \vdots & \vdots & \vdots & \vdots & \vdots & \vdots & \vdots & \ddots\\
        \end{bmatrix}
        \begin{bmatrix}
            \rbg{0}\\
            \gbg{0}\\
            0\\
            0\\
            \bbg{0}\\
            \ybg{x_{0, 0}}\\
            x_{0, 1}\\
            0\\
            \vdots\\
        \end{bmatrix}
        $
    }
\caption{Padded convolutional arithmetic as matrix multiplication. Left: Conventional illustration of a padded 2D-convolution. Right: 2D-convolution as matrix multiplication $G\cdot x$.}
\label{fig:conv_f}
\end{figure*}

Those relationships motivate how we approach the proof for RIP for 2D-convolutions. First, we recycle a result by \cite{haupt_toeplitz} and generalize the restricted isometry property for Toeplitz matrices generated by sequences $\{g_i\}_{i=1}^p$ containing zeros (Theorem~\ref{theo:haupt_modified}). Secondly, we show that deleting rows and columns of the Toeplitz matrix retains the RIP to some extent (Lemma~\ref{lemma:del_col} and Lemma~\ref{lemma:del_row}). Combining both ideas provides the RIP, and therefore a recovery bound, for random 2D-convolutions with multiple channels, padding, stride, and dilation (Corollary~\ref{cor:2dconv_recovery}). 

\begin{theorem}\label{theo:haupt_modified}
    Let $\{ g_i \}_{i=1}^p$ be a sequence of length $p=p_0 + p_r$ such that 
    \begin{itemize}
        \item $g_i=0$ for $i\in\mathcal{O}$ 
  and
        \item $g_i \sim \rvg_i$, where $\rvg_i$ are i.i.d. zero-mean random variables with variance $\mathbb{E}(\rvg_i^2)=1/p_r$ and $\vert \rvg_i \vert \leq \sqrt{c/p_r}$ for some $c\geq 1$ and $i \not \in \mathcal{O}$,
    \end{itemize}
     where $\mathcal{O}\subset \{1, \dots, p\}$ with cardinality $\vert \mathcal{O}\vert = p_0$. 
     Furthermore, let $G$ be the Toeplitz matrix generated by the sequence $\{g_i\}_{i=1}^p$. 
     Then, for any $\delta_S \in (0, 1)$ there exist constants $c_1, c_2$, such that for any sparsity level $S \leq c_2 \sqrt{ p_r / \log(D_{\operatorname{in}}) } $ it holds with probability at least $1-\exp (-c_1 p_r / S^2 )$ that $G$ is $(S, \delta_S)$-RIP. 
\end{theorem}
This result is very similar to Theorem~6 by \citet{haupt_toeplitz}. In fact, the difference here is that we pad the sequence that generates the Toeplitz matrix by zero values $\{g_i\}_{i \in \mathcal{O}}$. To prove this modified version, we note that we can apply Lemma~6 and Lemma~7 by \citet{haupt_toeplitz} for the sequence of zero-padded random variables $\{g_i\}_{i=1}^p$ as well. Besides that, we can perform the exact same steps as in the proof of Theorem~6 by \citet{haupt_toeplitz}, which is why we leave out the proof.

We want to highlight that the sparsity level $S$ and the probability for $G$ being $(S, \delta_S)$-RIP both scale with $p_r$, which is essentially determined by the number of kernels and the kernel size of the 2D-convolution.  

\begin{lemma}\label{lemma:del_col}
    Let $G\in \mathbb{R}^{D_{\operatorname{out}} \times D_{\operatorname{in}}}$ be $(S, \delta_S)$-RIP and let $G_{:, -j} \in \mathbb{R}^{D_{\operatorname{out}} \times (D_{\operatorname{in}} - 1)} $ be the matrix resulting from deleting the $j$th column from $G$. Then it follows that $G_{:, -j}$ is $(S-1, \delta_S)$-RIP. 
\end{lemma}
\begin{proof}
    Let $z^\prime \in \mathbb{R}^{D_{\operatorname{in}}-1}$ be a $S-1$-sparse vector and consider its zero-padded version $z \in \mathbb{R}^{D_{\operatorname{in}}}$ with marginals $z_i$ defined by 
    \begin{equation*}
        z_i := \begin{cases}
            z^\prime_i \quad &\text{if } i<j \\ 
            0 \quad &\text{if } i=j \\ 
            z^{\prime_{i - 1}} \quad &\text{if }i>j 
        \end{cases} \enspace .
    \end{equation*} 
    Note that $z$ is $S$-sparse and since $G$ is assumed to be $(S, \delta_S)$-RIP it follows that
    \begin{equation*}
        (1- \delta_S) \Vert z \Vert_2^2 \leq \Vert Gz \Vert_2^2 \leq (1+\delta_S) \Vert z \Vert_2^2 \enspace  . 
    \end{equation*}
    It holds that 
    \begin{equation*}
        \Vert G z \Vert_2^2 = \Vert G_{:, -j} z^\prime \Vert_2^2 \enspace \text{and} \enspace \Vert z \Vert_2^2 = \Vert z^\prime\Vert_2^2 \enspace 
    \end{equation*}
    and therefore 
    \begin{equation*}
        (1- \delta_S) \Vert z^\prime \Vert_2^2 \leq \Vert G_{:, -j}z^\prime \Vert_2^2 \leq (1+\delta_S) \Vert z^\prime \Vert_2^2 \enspace  ,
    \end{equation*}
    which proves that $G_{:, -j}$ is $(S-1, \delta_S)$-RIP.
\end{proof}

\begin{lemma}\label{lemma:del_row}
    Let $G\in \mathbb{R}^{D_{\operatorname{out}} \times D_{\operatorname{in}}}$ be $(S, \delta_S)$-RIP and let $G_{-j, :} \in \mathbb{R}^{(D_{\operatorname{out}} - 1) \times D_{\operatorname{in}}} $ be the matrix resulting from deleting the $j$th row from $G$. Then it follows that $G_{-j, :}$ is $(S, \delta_S^\prime)$-RIP, where 
    \begin{equation*}
        \delta_S^\prime:= \delta_S+ S \max_i G_{j, i}^2 \enspace .
    \end{equation*}
    
\end{lemma}
\begin{proof}
    Let $z \in \Sigma_S$. Since $G$ is assumed to be $(S, \delta_S)$-RIP, it holds that 
    \begin{equation*}
        \Vert G_{-j, :} z\Vert_2^2 \leq \Vert G z \Vert_2^2 \leq (1+\delta_S) \Vert z \Vert_2^2 \leq (1+\delta_S^\prime) \Vert z \Vert_2^2 \enspace ,
    \end{equation*}
    where the first inequality follows from the positivity of the summands, the second inequality follows from~\eqref{eq:RIP}, and the third inequality follows from $\delta_S \leq \delta_S^\prime$. Hence, it remains to prove the lower bound from~\eqref{eq:RIP}: 
    \begin{align*}
          \Vert G_{-j, :} z\Vert_2^2  &= \sum_{\substack{i=1\\ i \neq j}}^{D_{\operatorname{out}}} \bigl( G_{i,:}^\top z \bigr)^2  \\ 
          &= \biggl( \sum_{i=1}^{D_{\operatorname{out}}} \bigl( G_{i,:}^\top z \bigr)^2 \biggr)- \bigl(G_{j,:}^\top z \bigr)^2 \\ 
          &= \biggl( \sum_{i=1}^{D_{\operatorname{out}}} \bigl( G_{i,:}^\top z \bigr)^2 \biggr)- \bigl(G_{j,\mathcal{S}}^\top z_{\mathcal{S}} \bigr)^2 \enspace,
    \end{align*}
    where $\mathcal{S}:=\{j: z_j\neq 0\}$ denotes the support of $z$. 
    Finally, we can apply the Cauchy-Schwartz inequality to conclude that 
    \begin{equation*}
     \Vert G_{-j, :} z\Vert_2^2   \geq (1-\delta_S) \Vert z \Vert_2^2 - \Vert G_{j, \mathcal{S}} \Vert_2^2 \Vert z_{\mathcal{S}} \Vert_2^2 
          = (1-\delta_S^\prime) \Vert z \Vert_2^2 \enspace 
    \end{equation*}
    since $\Vert z \Vert_2^2 = \Vert z_{\mathcal{S}} \Vert_2^2$. Hence, $G_{j, :}$ is $(S, \delta_S^\prime)$-RIP. 
\end{proof}
 
Gathering our results, we can derive the following probabilistic recovery bound. 
\begin{corollary}
\label{cor:2dconv_recovery}
    Let $G$ be a 2D-convolution resulting from $n_r$ row deletions with $n_k$ kernel parameters sampled from an i.i.d. zero-mean random variables $\rvk_i$ with $\mathbb{E}(\rvk_i^2)=1/n_k$ and $\vert \rvk_i\vert\leq \sqrt{c / n_k}$ for some $c\geq 1$.
    Then, there exists a $S\in \mathbb{N}$ and $C_0, C_1, C_2 \geq 0$ such that if additionally $\vert \rvk_i \vert \leq \sqrt{C_2/n_d S}$ then it holds with high probability that
    \begin{equation}
        \Vert \hat{z} - z^\ast \Vert_2 \leq \frac{C_0}{\sqrt{S}} \Vert z^\ast - z_S
        \Vert_1 + C_1 E    \enspace , \label{eq:recovery_bound}
    \end{equation}
    where $E\geq \Vert \varepsilon \Vert_2$ is the noise-level.
\end{corollary}

\begin{proof}
    From Theorem~\ref{theo:haupt_modified}, it follows that there exists a $\delta_S<\sqrt{2}-1$ and a sparsity level $S$ such that the Toeplitz matrix $T(G)$ corresponding to $G$ is $(\delta_S, S)$-RIP with high probability. Furthermore, let us define $C_2:=\sqrt{2}-1-\delta_S$. By iteratively deleting $n_d$ rows from on $T(G)$, we end up with $G$, which is---according to Lemma~\ref{lemma:del_col}---$(\delta_S^\prime, S)$-RIP with 
    \begin{equation*}
        \delta_S^\prime = \delta_S + n_d S \max_i k_i \leq \delta_S + n_d S \frac{C_2}{n_d S} = \delta_S + C_2 = \sqrt{2}-1 \enspace.   
    \end{equation*}
    Hence, we can apply Theorem~\ref{theo:candes2} to conclude that there exist constants $C_0, C_1$ such that 
    \begin{equation*}
        \Vert \hat{z} - z^\ast \Vert_2 \leq \frac{C_0}{\sqrt{S}} \Vert z^\ast - z_S
        \Vert_1 + C_1 E    \enspace .
    \end{equation*}
\end{proof} 

Note that by iteratively applying Lemma~\ref{lemma:del_col}, we can derive a similar result for 2D-convolutions that employ zero-padding, which we leave out for the sake of simplicity. 

\subsection{Extending the Results to the Optimization Problem~\eqref{eq:opti_problem}} 
\label{sec:theory_modified}
So far, we have only derived recovery guarantees for solutions of the optimization problem~\eqref{eq:bpd_basic}. Fortunately, as presented in~\eqref{eq:opti_lasso} in the main paper, we can readily draw a connection between~\eqref{eq:bpd_basic} and~\eqref{eq:opti_problem}. 

\begin{defi}[$(\bar{z}, S)$-Similarity]
        Let $z, \bar{z} \in \mathbb{R}^D$. We say that $z$ is $(\bar{z},S)$-similar, if $\Vert z - \bar{z} \Vert_0$ is $S$-sparse. We denote the set of all $(\bar{z}, S)$-similar vectors by $\Sigma_S (\bar{z})$.
\end{defi}

In the following, we switch back to the notation from Section~\ref{sec:main} but leave out the layer index\footnote{compare with~\eqref{eq:opti_lasso} for instance} to emphasize the generality of this section, i.e.,  
\begin{align*}
    \hat{z} &= \argmin_{z \in \mathbb{R}^{D_{\operatorname{in}}}} \Vert Gz - o \Vert_2^2 + \lambda \Vert z - \bar{z} \Vert_1 \enspace ,  \\ 
    \hat{z}^\prime &= \argmin_{z \in \mathbb{R}^{D_{\operatorname{in}}}} \Vert Gz - o^\prime \Vert_2^2 + \lambda \Vert z \Vert_1 \enspace , 
\end{align*}
with $o^\prime = o - G\bar{z} = G (z^\ast - \bar{z})$. Since $z^\ast - \bar{z}$ is the input producing $o^\prime$, we define $z^{\ast \prime}=z^\ast - \bar{z}$. 
Then, we can readily extend sparse recovery results, such as Corollary~\ref{cor:2dconv_recovery}, to $(\bar{z}, S)$-similar recovery results by applying a simple zero-extensions: 
\begin{align*}
    \Vert \hat{z} - z^\ast \Vert_2^2 &= \Vert \hat{z} - \bar{z} + \bar{z} - z^\ast \Vert_2^2 = \Vert (\hat{z} - \bar{z}) - ( z^\ast - \bar{z}) \Vert_2^2 \\
    &= \Vert \hat{z}^\prime - z^{\ast\prime} \Vert_2^2 \enspace . 
\end{align*}
The right-hand side can be upper-bounded by our sparse recovery bound~\eqref{eq:recovery_bound} since, per definition, $(\bar{z}, S)$-similarity of $z^\ast$ implies $S$-sparsity of $z^{\ast \prime}$.

\section{Uniqueness of the Solution of~\eqref{eq:opti_problem}}\label{sec:unique}
The optimization problem~\eqref{eq:opti_problem} can be reformulated as an equivalent lasso optimization problem (Section~\ref{sec:optiislasso}), offering the advantage of leveraging efficient optimization algorithms like FISTA~\citep{beck2009fast}, which are guaranteed to converge. In addition, we show in this section that the solution is also unique, i.e., FISTA converges to a \textit{unique solution}. 
Therefore, there is no need to solve the optimization problem for different seeds, as seen in similar methods discussed in Section~\ref{sec:related_work}. We prove the uniqueness by applying techniques from~\citet{tibshirani2013lasso}. \\ In the following, we assume that $\Din > \Dout$.
\begin{defi}\label{def:generalpos}
    Let $G\in\mathbb{R}^{D_{\operatorname{out}} \times D_{\operatorname{in}}}$. We say that $G$ has columns in \textit{general position} if for any selection of columns $G_{i_1}, \dots, G_{i_{\Din}}$ of $G$ and any signs $\sigma_1, \dots , \sigma_{i_{\Din}} \in \{-1, 1\}$ it holds
    \small
    \begin{equation*}
        \operatorname{Aff}(\{\sigma_1 G_{i_1}, \dots , \sigma_n G_{i_{\Din}} \}) \cap \{\pm G_i: \; i \not \in  \{i_1, \dots , i_{\Din}\} \} = \emptyset \enspace , 
    \end{equation*}
    \normalsize
    where $\operatorname{Aff}(S):=\{\sum_{n=1}^N \lambda_n s_n: \; \sum_{n=1}^N \lambda_n = 1 \}$ defines the affine hull of the set $S=\{s_1, \dots, s_N\}$. 
\end{defi}

\begin{lemma}[\cite{tibshirani2013lasso}]\label{lemma:uniquelasso}
    If the columns of $G\in\mathbb{R}^{D_{\operatorname{out}} \times D_{\operatorname{in}}}$ are in general position, then for any $\lambda>0, \, o\in\mathbb{R}^{D_{\operatorname{out}}}$ the lasso problem 
    \begin{equation*}
        \hat{z} = \argmin_{z \in \mathbb{R}^{D_{\operatorname{in}}}} \Vert G z - o \Vert_2^2 + \lambda \Vert z \Vert_1 
    \end{equation*}
    has a unique solution. 
\end{lemma}
In the following, we prove that a $2D$-convolution has columns in general position.  
\begin{propo}\label{prop:U}
    Let $G\in\Routin$ be a 2D-convolution with kernel $k\in\mathbb{R}^{D_k}$ 
    with pairwise different kernel values and $D_k\geq 2$. 
    Then, there exist binary matrices $U_j \in \{0,1\}^{\Dout \times \Din }$ for $j\in\{1,\dots, \Din\}$ such that:
    \begin{enumerate}
        \item the $j$th column of $G$ can be written as $G_j = U_j \cdot k \in \mathbb{R}^{\Dout}$ \enspace ; 
        
        \item The supports of $U_j$, 
        \small
        $$
        S_j := \{ (l,m) \in \{1, \dots, \Dout\} \times \{1,\dots, \Din \}: \bigl(U_j\bigr)_{l,m} \neq 0 \} 
        $$
        \normalsize
        are pairwise disjoint for all $j$.
    \end{enumerate}
\end{propo}
\begin{proof}
    $G$ is a 2D-convolution, therefore, each column $G_j$ is a sparse vector with $l$th entry $(G_j)_l\in\{0, k_1,\dots , k_{D_k}\}$, i.e., it is either $0$ or one of the kernel values. 
    Hence, we can define the $U_j$ by 
    \begin{equation} \label{eq:construct_binary}
        (U_j)_{l, m} = \begin{cases}
            1\enspace , \quad &\text{if } (G_j)_l = k_m \\ 
            0 \enspace, \quad &\text{otherwise }
        \end{cases}
    \end{equation}
    for all $(l,m) \in \{1, \dots, \Dout\} \times \{1,\dots, \Din \}$.
    Using that construction, we get 
    \begin{align*}
        \bigl(U_j \cdot k \bigr)_{l} &=  (U_j)_l \cdot k  \\
        &= \begin{cases}
            k_m &\enspace, \text{if } (G_j)_l = k_m \\
            0 &\enspace, \text{otherwise}
        \end{cases}   \enspace  \\
        &= (G_j)_l  \enspace \text{ for all } l \in \{1, \dots, \Dout\} \enspace , 
    \end{align*}
    where $(U_j)_l$ denotes the $l$th row of $U_j$. This proves the first claim, that is, $G_j= U_j \cdot k$. 
    The second claim becomes immediate from our construction: Assume that there exist $j, j^\prime$ with $j\neq j^\prime$ but $S_j \cap S_{j^\prime} \neq \emptyset$. This means that there exists a $(l,m)$ such that $(U_j)_{l,m} = (U_{j^\prime})_{l,m}$. According to~\eqref{eq:construct_binary}, this means that $(G_j)_l=k_m=(G_{j^\prime})_l$, which contradicts the Toeplitz-like structure of $G$, see Section~\ref{sec:2dconv}. Hence, $S_j$ and $S_{j\prime}$ must be disjoint. 
\end{proof}

\begin{theorem}\label{theo:general_position}
    Let $G\in\Routin$ be a 2D-convolution with kernel values $k_j$ sampled from independent continuous random variables $\rvk_j$. Then, $G$ has columns in general position with probability $1$.
\end{theorem}
Before proving Theorem~\ref{theo:general_position}, let us provide a basic result from linear algebra.
\begin{lemma}\label{lemma:affine}
    Let $G_1, \dots , G_{n+1}\in\mathbb{R}^D$ for some $n, D \in \mathbb{N}$. Then it is 
    \begin{align*}
        &G_{n+1} \in \operatorname{Aff}(\{G_1, \dots, G_n \})  
        \\ 
        \Leftrightarrow \quad &G_{n+1} - G_n \in  \operatorname{span}(G_1 - G_n , \dots, G_{n-1} - G_n) \enspace ,  
    \end{align*}
    where $\operatorname{span}$ denotes the linear span. 
\end{lemma}
\begin{proof}
    Let $G_{n+1} - G_n \in \operatorname{span}(G_1 - G_n , \dots, G_{n-1} - G_n)$. Therefore, there exist $\lambda_1, \dots, \lambda_{n-1} \in \mathbb{R}$ such that 
    \begin{align*}
        G_{n+1} - G_n &= \sum_{j=1}^{n-1} \lambda_j (G_j - G_n)  \\ 
        \Leftrightarrow G_{n+1} &= \sum_{j=1}^{n-1} \lambda_j G_j - \sum_{j=1}^{n-1} \lambda_j G_n + G_n \\ 
        \Leftrightarrow G_{n+1} &= \sum_{j=1}^{n-1} \lambda_j G_j + \biggl(1 - \sum_{j=1}^{n-1} \lambda_j \biggr) G_n \enspace .         
    \end{align*}
    The above coefficients of $G_j$ satisfy 
    \begin{equation*}
        \sum_{j=1}^{n-1} \lambda_j + \biggl(1-\sum_{j=1}^{n-1} \lambda_j \biggr) = 1 \enspace. 
    \end{equation*}
    Therefore, we have shown that there exist coefficients $\lambda_j^\prime$ such that $G_{n+1} = \sum_{j=1}^n \lambda_j^\prime G_j$, that is, $G_{n+1} \in \operatorname{Aff}(\{G_1, \dots, G_n \}$. \\
    Conversely, let $G_{n+1} \in \operatorname{Aff}(\{G_1, \dots, G_n \}$. Then, there exist $\lambda_1, \dots, \lambda_{n} \in \mathbb{R}$ with $\sum_{j=1}^{n}\lambda_j = 1$ such that 
    \begin{align*}
        G_{n+1} &= \sum_{j=1}^n \lambda_j G_j \\    
        \Leftrightarrow \quad G_{n+1} - G_n &= \sum_{j=1}^n \lambda_j G_j - G_n \\
         &= \sum_{j=1}^n \lambda_j G_j - \sum_{j=1}^n \lambda_j G_n    \\
         &= \sum_{j=1}^n \lambda_j (G_j - G_n) \enspace , 
    \end{align*}
    which proves that $G_{n+1} - G_n \in \operatorname{span}(G_1 - G_n , \dots, G_{n-1} - G_n)$. 
\end{proof}
Next, we provide the proof of Theorem~\ref{theo:general_position}. 

\begin{proof}
    Let $G_{i_1}, \dots , G_{i_{\Din}}, G_{i_{\Din + 1}}$ be a collection of $\Din + 1$ columns of $G$. We show that $G_{i_{\Din + 1}} \in \operatorname{Aff}(\{G_{i_1}, \dots, G_{i_{\Din}} \})$ by proving that $G_{i_{\Din+1}} - G_{i_{\Din}} \in \operatorname{span}(G_{i_1} - G_{i_{\Din}}, \dots, G_{i_{\Din- 1}} - G_{i_{\Din}})$ with probability~$1$, see Lemma~\ref{lemma:affine}. According to Proposition~\ref{prop:U}, there exist binary matrices $U_j$ such that 
    \begin{align}
        & G_{i_{\Din+1}} - G_{i_{\Din}} \in \operatorname{span}(G_{i_1} - G_{i_{\Din}},  \nonumber
        \\
        & \quad \quad \dots, G_{i_{\Din- 1}} - G_{i_{\Din}} ) \nonumber \\ 
          \Leftrightarrow & \;(U_{i_{\Din + 1}} - U_{i_{\Din}}) \cdot  k \in \operatorname{span} ( (U_{i_1} - U_{i_{\Din}})\cdot k ,  \nonumber\\ 
          & \quad \quad   \dots, (U_{i_{\Din - 1}} - U_{i_{\Din}}) \cdot k ) \nonumber
          \\
          \Leftrightarrow & \; \exists 
        (\lambda_1, \dots, \lambda_{\Din - 1} ) \neq (0, \dots, 0): \nonumber \\  & (U_{i_{\Din}} - U_{i_{\Din - 1}}) \cdot k = \sum_{j=1}^{\Din - 1} \lambda_j (U_{i_j} - U_{i_{\Din}}) \cdot k \nonumber \\ 
        \Leftrightarrow &\; \exists 
        (\lambda_1, \dots, \lambda_{\Din - 1} ) \neq (0, \dots, 0): \nonumber \\ & 0 = \biggl( \sum_{\substack{j=1 \\ j \neq \Din }}^{\Din +1} \lambda_j (U_{i_{j}} - U_{i_{\Din}}) \biggr) \cdot k , \enspace \text{ with } \lambda_{\Din +1}:= -1  \nonumber
        \\ 
         \Rightarrow &\; k \in \operatorname{nul}\biggl(\sum_{\substack{j=1 \\ j \neq \Din }}^{\Din +1} \lambda_j (U_{i_{j}} - U_{i_{\Din}}) \biggr) \enspace 
         \label{eq:kinnull}  \\ & \text{ for the above choice of } \lambda_1, \dots \lambda_{\Din - 1} , \lambda_{\Din +1} \enspace , \nonumber 
    \end{align}
    where $\operatorname{nul}(A)$ denotes the null-space of $A$. We show that~\eqref{eq:kinnull} holds with probability $0$ by proving that the null-space has dimension less than $D_k$. By the rank-nullity theorem, this is equivalent to showing that 
    \begin{equation*}
        \operatorname{rank}(U) \geq 1 \enspace \text{ for } U:= \sum_{\substack{j=1 \\ j \neq \Din }}^{\Din +1} \lambda_j (U_{i_{j}} - U_{i_{\Din}}) \enspace . 
    \end{equation*}
    Since $U$ is a linear mapping, it is sufficient to show that there exists some $k \in \mathbb{R}$ such that $Uk\neq 0$. 
    \\
    Without loss of generality, let $\lambda_1 \neq 0$. Then, since $U_1\neq 0$, we know that there exist some $l,m$ such that $(U_1)_{l,m}=1$. Then, it holds for $k=e_m$, where $e_m$ is the $m$th Euclidean unit vector, that 
    \begin{align*}
        \bigl(  U k \bigr)_l &= \biggl(  \sum_{j=1}^{\Din -1} \lambda_j U_j e_m - \sum_{j=1}^{\Din -1} \lambda_j U_{\Din} e_m  \biggr)_l \\ 
        &= \biggl(  \sum_{j=1}^{\Din } \lambda_j U_j e_m \biggr)_l = \lambda_1 \neq 0 \enspace ,
    \end{align*}
    where $\lambda_{\Din}:= -\sum_{j=1}^{\Din } \lambda_j $. Note, the last equality follows by the fact that all $U_j$ have pairwise disjoint supports. Since the $l$th component of $Uk$ is not zero, we conclude that there exists some $k\in\mathbb{R}^{D_k}$ such that $Uk\neq 0$. Therefore, $\operatorname{rank}(U)\geq 1$, from which derive that $\operatorname{dim}(\operatorname{nul}(U))<D_k$, i.e., the null-space of $U$ has Lebesgue mass $0$. \\
    In summary, if $G_{i_{\Din + 1}} \in \operatorname{Aff}(\{G_{i_1}, \dots, G_{i_{\Din}} \})$, then $k$ needs to be in a Lebesgue null-set, which has probability $0$ since $\rvk_j$ are assumed to be continuous random variables. We can repeat the same arguments independent of the sign of $G_{i_{\Din + 1}}$, the sign of $G_{i_1}, \dots, G_{i_{\Din}}$, for each $G_i$ with $i \not\in \{i_1, \dots, i_{\Din}\}$, and for any collection of $\{i_1, \dots, i_{\Din}\}$. By taking a union bound over all choices, we conclude that the columns of $G$ are in general position with probability $1$.
\end{proof}

Finally, plugging the pieces together, we prove the uniqueness of~\eqref{eq:opti_problem}. 
\begin{corollary}
    Let $G\in\Routin$ be a 2D-convolution with kernel values $k_j$ sampled from independent continuous random variables $\rvk_j$. Then, the optimization problem~\eqref{eq:opti_problem} has a unique solution with probability $1$. 
\end{corollary}
\begin{proof}
    As shown in Section~\ref{sec:optiislasso}, we can rewrite~\eqref{eq:opti_problem} to the equivalent lasso optimization problem~\eqref{eq:opti_lasso}. According to Theorem~\ref{theo:general_position}, the columns of $G$ are in general linear position with probability $1$. Using Lemma~\ref{lemma:uniquelasso}, we conclude that~\eqref{eq:opti_problem} has a unique solution with probability $1$. 
\end{proof}

\section{Adapting Closed-World Attribution Methods to the Open-World Setting}\label{sec:adapting} 
While the method proposed in this work provides an efficient solution to the problem of single-model attribution in the open-world setting, it might not be suitable under all circumstances (e.g., if final-layer inversion is infeasible due to non-invertible ReLU-activations).
Furthermore, if a particularly effective method exists in the closed-world setting, it might be beneficial to build upon this strong foundation.
We, therefore, propose two adapted methods, SM (single-model)-Fingerprinting and SM-Inversion, which leverage the existing approaches to the open-world setting.
Both return an anomaly score $s(x)$ that can be used to identify outliers by comparing it to a threshold $\tau$.
All samples with an anomaly score greater than $\tau$ are classified as outliers. 

\paragraph{SM-Fingerprinting}
\citet{marraGANsLeaveArtificial2019} propose to use the photo response non-uniformity (PRNU) to obtain noise residuals from images.
The fingerprint $f_G$ is computed by averaging the residuals over a sufficient amount of images generated by $G$.
Given an unknown image $x$, the anomaly score is defined as the negative inner product between $x$ and $f_G$, $s(x):= - \tilde{f}_G^\top \tilde{x}$, where $\tilde{f}_G$ and $\tilde{x}$ are the flattened and standardized versions of $\tilde{f}_G$ and $x$, respectively. In the following, we refer to this kind of SM-Fingerprinting as \finger. 

\paragraph{SM-Inversion}
For inversion-based methods, the anomaly score is naturally given by the distance between the original image $x \in \mathbb{R}^D$ and its best reconstruction $G(\hat{z})$.
It is defined as $s(x) := 1/{D} \Vert x  - G(\hat{z}) \Vert$, where $\Vert \cdot \Vert$ is either the $\ell_2$-norm \cite{albright2019source} or a distance based on a pre-trained Inception-V3\footnote{In our experiments on $64\times 64$ CelebA and LSUN, we omit resizing to $299\times299$, since this would introduce significant upsampling patterns.} \citep{Szegedy_2016_CVPR} as proposed by \citet{zhangAttributionDeepfakes2021}. 
In the following, we refer to the two variants as \inv{} and \incinv.
The reconstruction seed $\hat{z}$ is found by performing multiple gradient-based reconstruction attempts, each initialized on a different seed, and selecting the solution resulting in the smallest reconstruction distance.

\section{Final-Layer Inversion with Skip-Connections} \label{sec:skip}
The inversion scheme of \method{}, as written down in Section~\ref{sec:method}, is not suited for architectures that employ skip-connections. 
\emph{But} note that this does not mean that it cannot be adapted to that setting. To see that, consider the output of a residual network 
$
x = \sigma_L (G_L z_{\operatorname{last}} + z_{\operatorname{skip}})
$,
where $z_{\operatorname{last}}$ is the last hidden activation, $z_{\operatorname{skip}}$ is the part that is added from the skip-connection, $G_L$ is the matrix representing the last linear layer, and $\sigma_L$ is the last activation function. We can rewrite the above as 
$
x = \sigma_L ( G_L^\prime (z_{\operatorname{last}}, z_{\operatorname{skip}})^\top ) 
$,
where $G_L^{\prime}$ is an extended matrix padded by binary values, which takes as input the concatenated vector $(z_{\operatorname{last}}, z_{\operatorname{skip}})^\top$. 
Now, we are ready to apply the same techniques as proposed in Section~\ref{sec:method}. The difference is that, instead of only reconstructing $z_{\operatorname{last}}$, we additionally need to reconstruct $z_{\operatorname{skip}}$. 

While being a valuable and interesting extension (both, from a practical and theoretical point of view), we have not investigated this idea further due to the following technical considerations. First, architectures like StyleGAN employ skip-connections from multiple hidden layers to the output. This may result in prohibitively large  $z_{\operatorname{skip}}:= ( z_{\operatorname{skip}_1}, \dots, z_{\operatorname{skip}_h})^\top$, 
where $z_{\operatorname{skip}_j}$ for $j\in\{1, \dots, h\}$ denote the latents corresponding to the $j$th skip-layer. Secondly, the structural properties, such as the level of sparsity, in each hidden layer might be different. Therefore, the regularization term in~\eqref{eq:opti_problem} should treat latents from each hidden layer separately. 
While, in theory, we could use a separate regularization parameter $\lambda$ for each group of latents, the implementation involves nontrivial engineering decisions and hyperparameter choices.

From a theoretical point of view, we need to expand the results from \Cref{sec:theory} to convolutional operations padded by binary values. For instance, reconstruction guarantees for binary matrices can be derived using the robust null space property (see e.g.,~\citet{lotfi2020compressed}).

Another way of extending \method{} to skip-connections could build upon the results by~\citet{lei2019inverting}.
Their proposed linear program inverts a single layer, leading to solutions that satisfy recovery guarantees (see Theorem~4 in~\citet{lei2019inverting}) for $z_{\operatorname{last}}$ under assumptions closely related to the restricted isometry property (Definition~\ref{def:rip}). 

\section{Experimental Details}\label{sec:exp_details}

\subsection{Thresholding Method} \label{sec:thresholding_method}
We select the threshold $\tau$ in~\eqref{eq:deepsadscore} by fixing a false negative rate $\operatorname{fnr}$ and set $\tau$ as the $(1-\operatorname{fnr})$-quantile of $\{s(x)\}_{x_i \in \mathcal{X}_{\operatorname{val}}}$, where $s(x)$ is the anomaly score function and $\mathcal{X}_{\operatorname{val}}$ is a validation set consisting of $n_{\text{val}}$ inlier samples (i.e., generated by $G$). We specify $n_{\text{val}}$ for each experiment is the following sections. 
In all of our experiments, we set $\operatorname{fnr}=0.005$ or $\operatorname{fnr}=0.05$ in the case of the Stable Diffusion experiments. In the latter, we decide to set a higher $\operatorname{fnr}$ due to the smaller validation set size.

\subsection{Generative Models trained on CelebA and LSUN}\label{sec:trained_gen_models}
In the following, we describe the generative models utilized in Section~\ref{sec:exp_small} in full detail. All models were optimized using Adam with a batch size of $128$, a learning rate of $0.0002$ and---if not stated differently---parameters $\beta_1 =0.5, \, \beta_2=0.999$. The goal was not to train SOTA-models but rather to provide a diverse set of generative models as baselines for evaluating attribution performances in the setting described in Section~\ref{sec:problem_setup}. We visualize generated samples in Figure~\ref{fig:samples_small_models}.
All experiments on CelebA and LSUN utilize $n_{\text{tr}}=10\,000$ and $n_{\text{val}}=n_{\text{test}}=1\,000$ samples. 

\begin{figure}[ht]
  \centering
  \subfigure[CelebA]{\includegraphics[width=.45\linewidth]{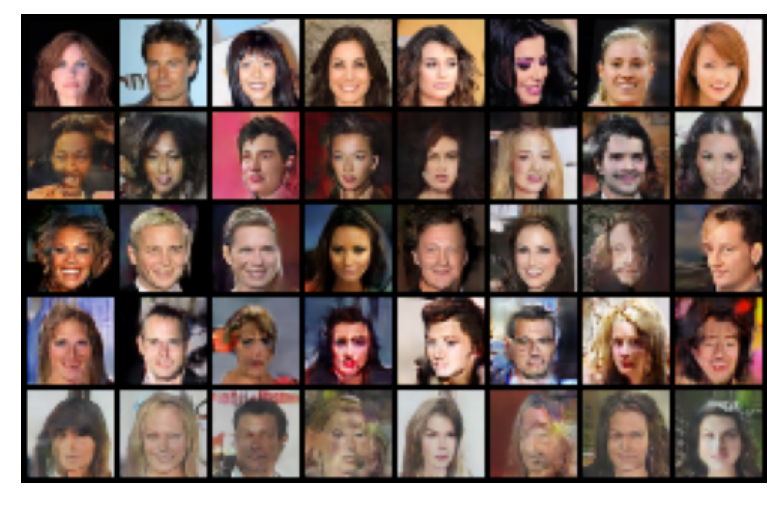}
  \label{fig:examples_celeba}}
 \subfigure[LSUN]{\includegraphics[width=.45\linewidth]{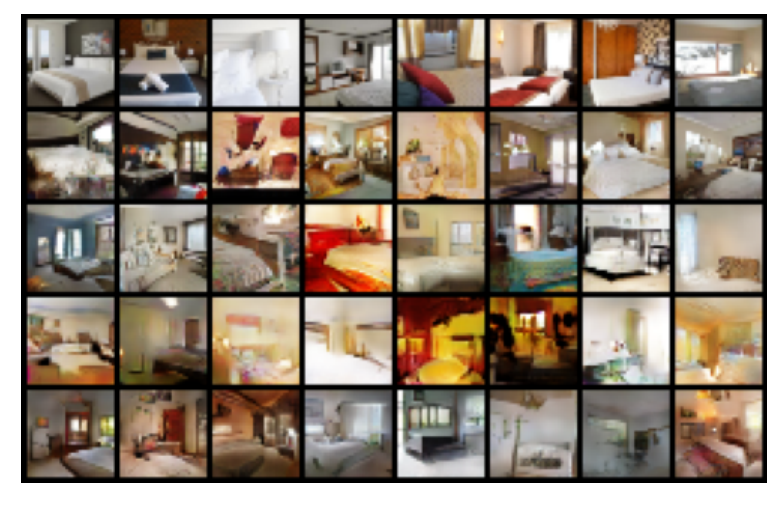}
  \label{fig:examples_lsun}}
\caption{Examples of generated images. The first row represents real images, and each further row represents samples from DCGAN, 
WGAN-GP, 
LSGAN, 
EBGAN, 
respectively.}
\label{fig:samples_small_models}
\end{figure}

\paragraph{DCGAN}
Our DCGAN Generator has $5$ layers, where the first 4 layers consist of a transposed convolution, followed by batchnorm and a ReLU-activation, and the last layer consists of a transposed convolution (without bias) and a $\operatorname{tanh}$-activation. The transposed convolutions in- and output-channel dimensions are $100 \rightarrow 512 \rightarrow 256 \rightarrow 128 \rightarrow 64 \rightarrow 3$ and have a kernel size of $4$, stride of $2$, and a padding of $1$ (besides the first transposed convolution, which has a padding of $0$). The discriminator has $5$ layers, where the first layer is a convolution (without bias) followed by a leaky-ReLU-activation, the next $3$ consist of a convolution, batchnorm, and leaky-ReLU-activation, and the last layer consists of a convolution and a final sigmoid-activation. We set the leaky-ReLU tuning parameter to $0.2$. The convolutions channel dimensions are $3 \rightarrow 64 \rightarrow 128 \rightarrow 256 \rightarrow 512 \rightarrow 1$, and we employed a kernel of size $4$, stride of $2$, and a padding of $1$ in all but the last layer, which has a padding of $0$. 
We used soft-labels of $0.1$ and $0.9$ and trained the generator in each iteration for $50$ and $10$ epochs for CelebA and LSUN, respectively. 

\paragraph{WGAN-GP}
We utilized the same architecture for the generator as the one in DCGAN. The discriminator is similar to the one in DCGAN but uses layernorm instead of batchnorm and has no sigmoid-activation at the end. We trained the generator every $5$th iteration and utilized a gradient penalty of $\lambda=10$. In that case, we used $\beta_1=0.0$ and $\beta_2=0.9$ as Adam parameters and trained for $200$ and $10$ epochs for CelebA and LSUN, respectively.

\paragraph{LSGAN}
The generator consists of a linear layer mapping from $100$ to $128 \times 8 \times 8$ dimensions, followed by 3 layers, each consisting of a nearest-neighbor upsampling ($\times 2$), a convolution, batchnorm, and a ReLU-activation. The final layer consists of a convolution, followed by a $\operatorname{tanh}$-activation. The convolution channel dimensions are $128 \rightarrow 128 \rightarrow 128 \rightarrow 64 \rightarrow 3$ and utilize a kernel of size $3$, stride of $1$, and a padding of $1$. The discriminator consists of $5$ layers discriminator blocks and a final linear layer. Each discriminator block contains a convolution, leaky-ReLU with parameter $0.2$, dropout with probability $0.25$, and batchnorm in all but the first layer. The convolution channel dimensions are $3 \rightarrow 16 \rightarrow 32 \rightarrow 64 \rightarrow 128 \rightarrow 128$ and have a kernel of size $3$, stride of $2$, and padding of $1$. The final linear layer maps from $512$ to $1$ dimensions. We trained the generator every $5$th iteration for $100$ and $10$ epochs for CelebA and LSUN, respectively.   

\paragraph{EBGAN}
The architecture of the generator is similar to the one of the DCGAN generator. The discriminator uses an encoder-decoder architecture. The encoder employs convolutions (without bias) mapping from $3 \rightarrow 64 \rightarrow 128 \rightarrow 256$ channel dimensions, each having a kernel of size $4$, stride of $2$, and padding of $1$. Between each convolution, we use batchnorm and a leaky-ReLU-activation with parameter $0.2$. The decoder has a similar structure but uses transposed convolutions instead of convolutions with channel dimensions $256 \rightarrow 128 \rightarrow 64 \rightarrow 3$ and ReLUs instead of leaky-ReLUs. We set the pull-away parameter to $0.1$ and trained the generator every $5$th iteration for $100$ and $5$ epochs for CelebA and LSUN, respectively.  

\subsection{Stable Diffusion}
We generate samples using the \textit{diffusers}\footnote{\url{https://github.com/huggingface/diffusers}} library which provides checkpoints for each version of Stable Diffusion.
We use a subset of the COCO2014 annotations~\cite{coco} as prompts and sample with default settings.
All generated images have a resolution of $512\times 512$ pixels.
We set $n_{\text{tr}}=2\,000, \; n_{\text{val}}=100, \; n_{\text{test}}=200$. 

\subsection{Style-based Generators}
Each style-based generative model comes with an official implementation, which offers pre-trained models on FFHQ for download. We used those implementations to generate $n_{\text{tr}}=10\,000,\; n_{\text{val}}=n_{\text{test}}=1\,000$ images of resolution $256\times 256$. 

\subsection{Medical Image Generators}
We sampled from the pretrained models provided by the \emph{medigan} library \cite{osuala2023medigan}. 
The model identifiers\footnote{The model identifiers specify the exact \emph{medigan}-models.} are 5, 6, and 12 for the DCGAN, WGAN-GP, and c-DCGAN, respectively. The resolution images have a resolution of $128 \times 128$ pixels.
We set $n_{\text{tr}}=10\,000,\; n_{\text{val}}=n_{\text{test}}=1\,000$. 

\subsection{Tabular Generators}
The KL-WGAN was trained for $500$ epochs as specified in \citet{song2020bridging} using the official repository\footnote{\url{https://github.com/ermongroup/f-wgan}}. CTGAN, TVAE, and Copula GAN were trained using the implementations and the default settings provided by the \emph{SDV} library~\cite{SDV}. We set $n_{\text{tr}}=100\,000$ and $n_{\text{val}}=n_{\text{test}}=10\,000$.

\subsection{Hyperparameters of Attribution Methods}
\paragraph{DeepSAD}
DeepSAD uses the $\ell_2$-distance to the center $c$ as anomaly score, see~\eqref{eq:deepsadscore}. 
We propose $3$ variants of DeepSAD for the problem of single-model attribution:
\begin{enumerate}
    \item \raw{}, i.e., DeepSAD trained on downsampled versions of the images. 
    In the experiments involving the architectures from Section~\ref{sec:trained_gen_models}, we downsample to $32 \times 32$ images using nearest neighbors interpolation. In the Stable Diffusion experiments, we center-crop to $128 \times 128$ images due to the large downsampling factor. 
    In the style-based experiments, we downsample to $128 \times 128$ using nearest neighbors interpolation. 
    When the image is downsampled to $32\times 32$, we set $\phi$ as the 3-layer LeNet that was used on CIFAR-10 in~\citet{ruff2019deep}, otherwise we use a 4-layer version of the LeNet-architecture. 
    In the tabular experiments, we set $\phi$ to be a standard feed-forward network with hidden dimensions $512 \rightarrow 1024 \rightarrow 1024 \rightarrow 512 \rightarrow 256 \rightarrow 128$.
    
    \item \dct{}, i.e., DeepSAD on the downsampled DCT features of the images. We obtain those features by computing the 2-dimensional DCT on each channel, adding a small constant $\eps=1e^{-10}$ for numerical stability, and taking the $\operatorname{log}$. 
    To avoid distorting the DCT-spectra, we reduce the image size by taking center-crops. We use the same architecture $\phi$ as in \raw{}. 
    \item \method{}, i.e., DeepSAD on activations result from the optimization problem~\eqref{eq:opti_problem}. We solved the optimization problem using FISTA \citep{beck2009fast} with the regularization parameter $\lambda$ set to $0.0005$ for the experiments involving the models from Section~\ref{sec:trained_gen_models}, to $0.00001$ for the Stable Diffusion experiments, to $0.001$ for the medical image experiments, and to $0.0005$ for the Redwine and Whitewine experiments. 
    We found that reducing the feature dimensions via max-pooling performed best. The spatial dimensions are similar to the setup in \raw. 
    In the experiments using the architectures from Section~\ref{sec:trained_gen_models}, we end up with $64 \times 32 \times 32$-dimensional features, where the first dimension represents the channel dimension. Hence, we modified the channel dimensions of the convolutions used in $\phi$ to $64 \rightarrow 128 \rightarrow 256 \rightarrow 128$. We use the same $\phi$ for the medical image experiments. 
    In the Stable Diffusion experiments, we account for the smaller amount of training data by selecting only 3---out of a total of 128---channels. We do this by choosing the channels that differ the most on average between the in- and out-of-distribution classes. We use the exact same $\phi$ as in \raw{} and \dct{}. 
    For the tabular experiments, we choose the same architecture as in \raw{} with an adapted input dimension. 
     
\end{enumerate}
In the experiments involving the models from Section~\ref{sec:trained_gen_models} and the medical images, we trained $\phi$ using the adam optimizer for $50$ epochs with a learning rate of $0.0005$, which we reduce to $5e^{-5}$ after $25$ epochs, and a weight-decay of $0.5e^{-6}$. In the Stable Diffusion, StyleGAN, and tabular experiments, we increase the number of epochs to $100$ and reduce the learning rate after $25$ and $50$ epochs to $0.5e^{-6}$ and $0.5e^{-7}$, respectively. 

\paragraph{SM-Fingerprinting}
To compute the fingerprint $f_G$ we take the average of $n_{\text{tr}}$ samples from $G$.

\paragraph{SM-Inversion}
For each image $x$ we perform $10$ reconstruction attempts with different initializations.
We optimize using Adam \citep{kingma2015} with a learning rate of $0.1$ and $1\,000$ optimization steps. Since the inversion methods do not require training, we compensate the advantage by replacing the validation set with the bigger training set ($n_{\operatorname{val}}<n_{\operatorname{tr}} $ in all our experiments). That is, we use $n_{\text{tr}}$ samples to find the threshold $\tau$ described in Section~\ref{sec:thresholding_method}.

\section{Additional Experiments}
\label{sec:app_smallexp}

\paragraph{Activations are distinctive across Generative Models}
The motivational examples in Section~\ref{sec:motivationalExample}, as well as our theoretical derivation of \method{} in Section~\ref{sec:method}, are based on the assumption that the activations, and therefore also the intrinsic computations involved in the generative process, differ from model to model. In the following, we provide qualitative evidence for that by presenting average activations that arise from the forward pass in various different models. For fully distinct models, such as those in Figure~\ref{fig:real_activations_celeba} and Figure~\ref{fig:real_activations_lsun}, we can see a clear difference in the activation pattern. The difference is less visible in Figure~\ref{fig:real_activations_coco}, which shows the activations of the investigated Stable Diffusion models. Note that those models share much more similarities: v1-1 and v1-4 share the same autoencoder. The same applies to v1-1+, v2, and v2-1. Hence, we can see distinguishable activation patterns among models using a different autoencoder, which are much more subtle if the models are sharing the same autoencoder.

\begin{figure}[ht]
\begin{center}
\centerline{\includegraphics[width=\columnwidth]{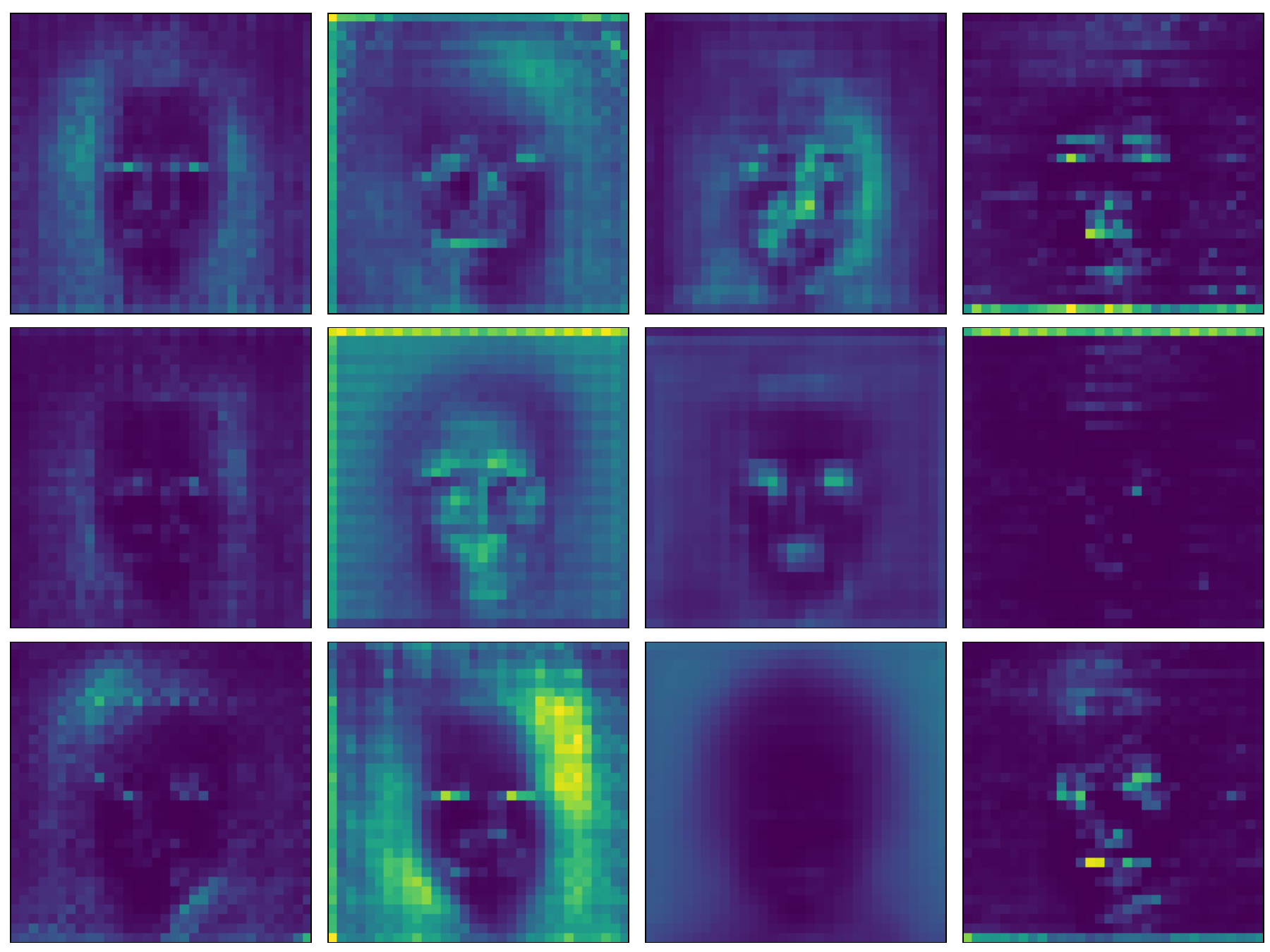}}
\caption{Average activations of DCGAN, WGAN-GP, LSGAN, EBGAN (from left to right) trained on CelebA. The rows depict the average last activation over the 6th, 22th, and 29th activation channel.
}
\label{fig:real_activations_celeba}
\end{center}
\end{figure}

\begin{figure}[ht]
\begin{center}
\centerline{\includegraphics[width=\columnwidth]{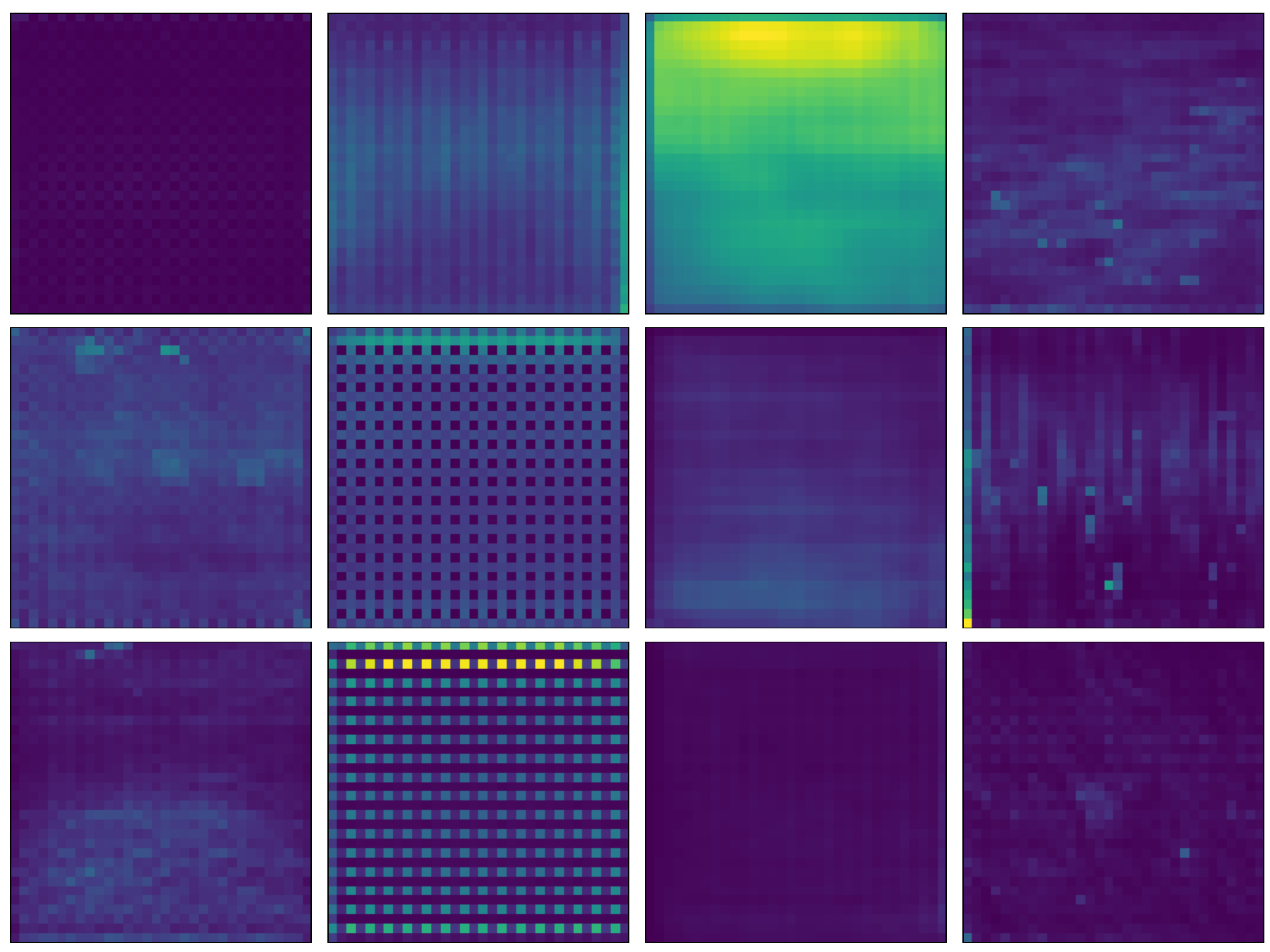}}
\caption{
Average activations of DCGAN, WGAN-GP, LSGAN, EBGAN (from left to right) trained on LSUN. The rows depict the average last activation over the 7th, 29th, and 30th activation channel.
}
\label{fig:real_activations_lsun}
\end{center}
\end{figure}

\begin{figure}[ht]
\begin{center}
\centerline{\includegraphics[width=\columnwidth]{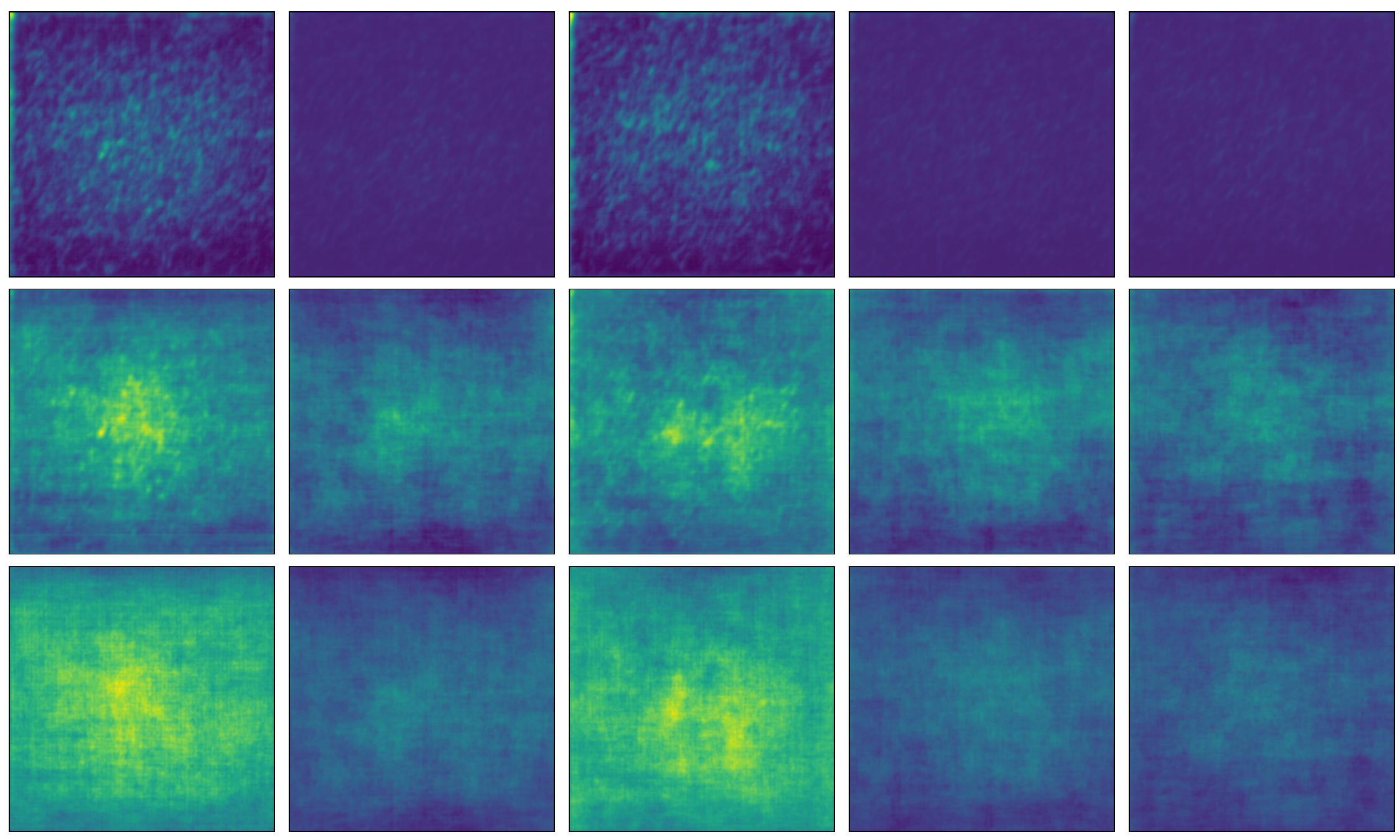}}
\caption{
Average activations of Stable Diffusion v1-1, v1-1+, v1-4, v2, v2-1 (from left to right). The rows depict the average last activation over the 60th, 78th, and 86th activation channel.
}
\label{fig:real_activations_coco}
\end{center}
\end{figure}

\paragraph{Feature Extraction} 
We visualize the average images (Figure~\ref{fig:avg_image_celeba}), that is, the average of the input of \raw{}, and the average DCT-features (Figure~\ref{fig:avg_dct_celeba}), that is, the average of the input of \dct{}. It is seen that the averaged images look very similar. Qualitatively, \method{} improves on that as demonstrated in Figure~\ref{fig:features_dcgan} in the main paper, and once more, in Figure~\ref{fig:features_dcgan_other}.
Furthermore, the features remain distinguishable even when varying the regularization parameter~$\lambda$ as visualized in Figure~\ref{fig:varying_lambda_celeba} and Figure~\ref{fig:varying_lambda_lsun}. 
For illustration purposes, we select the 3 channel dimensions (out of 64 total activation dimensions in DCGAN), which differ the most during training, i.e., we choose the channel that maximizes the channel-wise average distance between reconstructed DCGAN and real images. 
Notably, the corresponding average single-model attribution accuracies are $99.11, \, 99.10,\, 99.44,\, 86.84$ in CelebA and $99.99,\, 96.58,\, 99.63,\, 85.90$ for $\lambda\in \{ 0.0001, 0.0005, 0.001, 0.1\}$. Therefore, the performance remains stable in $\lambda$.

\begin{figure}[ht]
\begin{center}
\centerline{\includegraphics[width=\columnwidth]{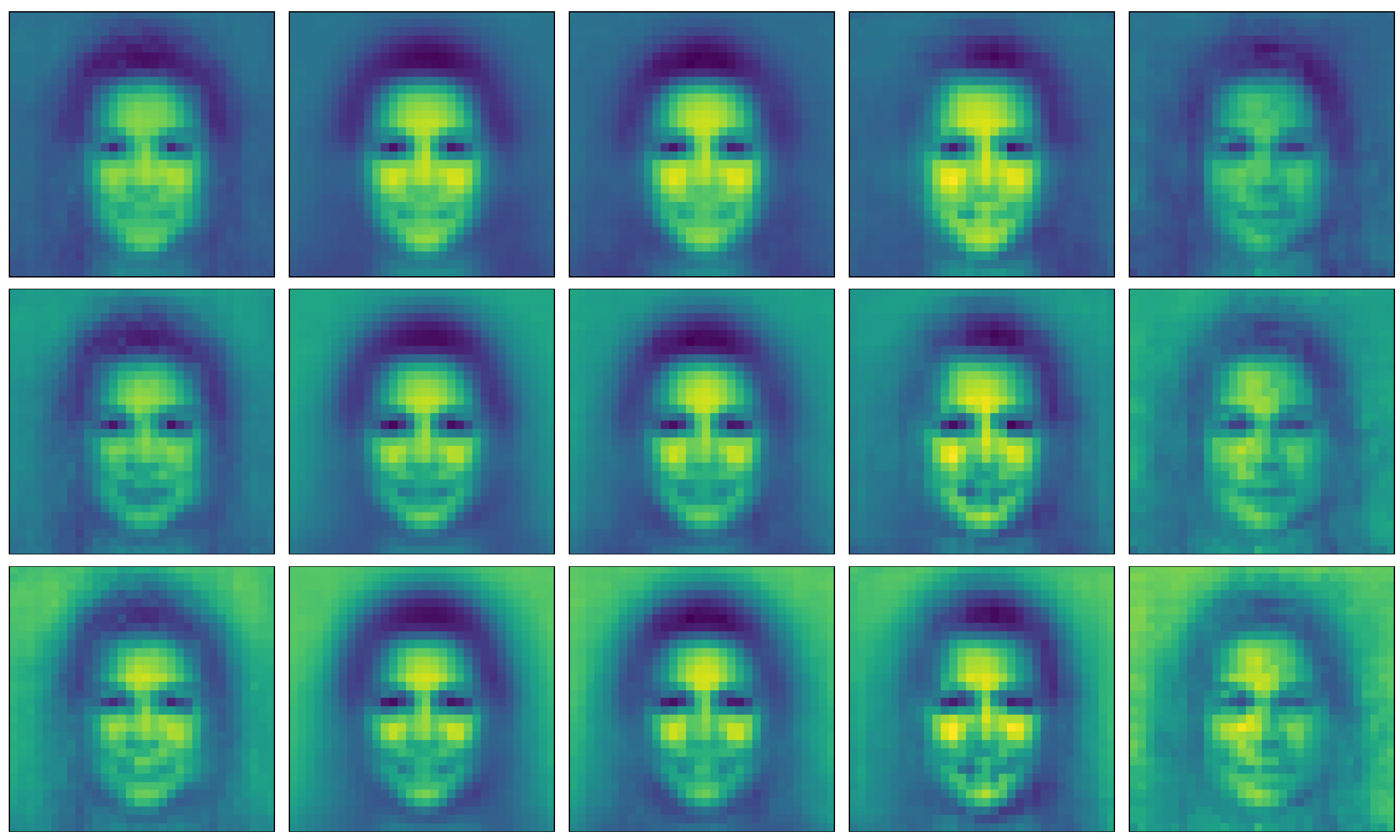}}
\caption{Average image on CelebA for real, DCGAN, WGAN-GP, LSGAN, and EBGAN samples (from left to right). Each row represents one color channel.}
\label{fig:avg_image_celeba}
\end{center}
\end{figure}

\begin{figure}[ht]
\begin{center}
\centerline{\includegraphics[width=\columnwidth]{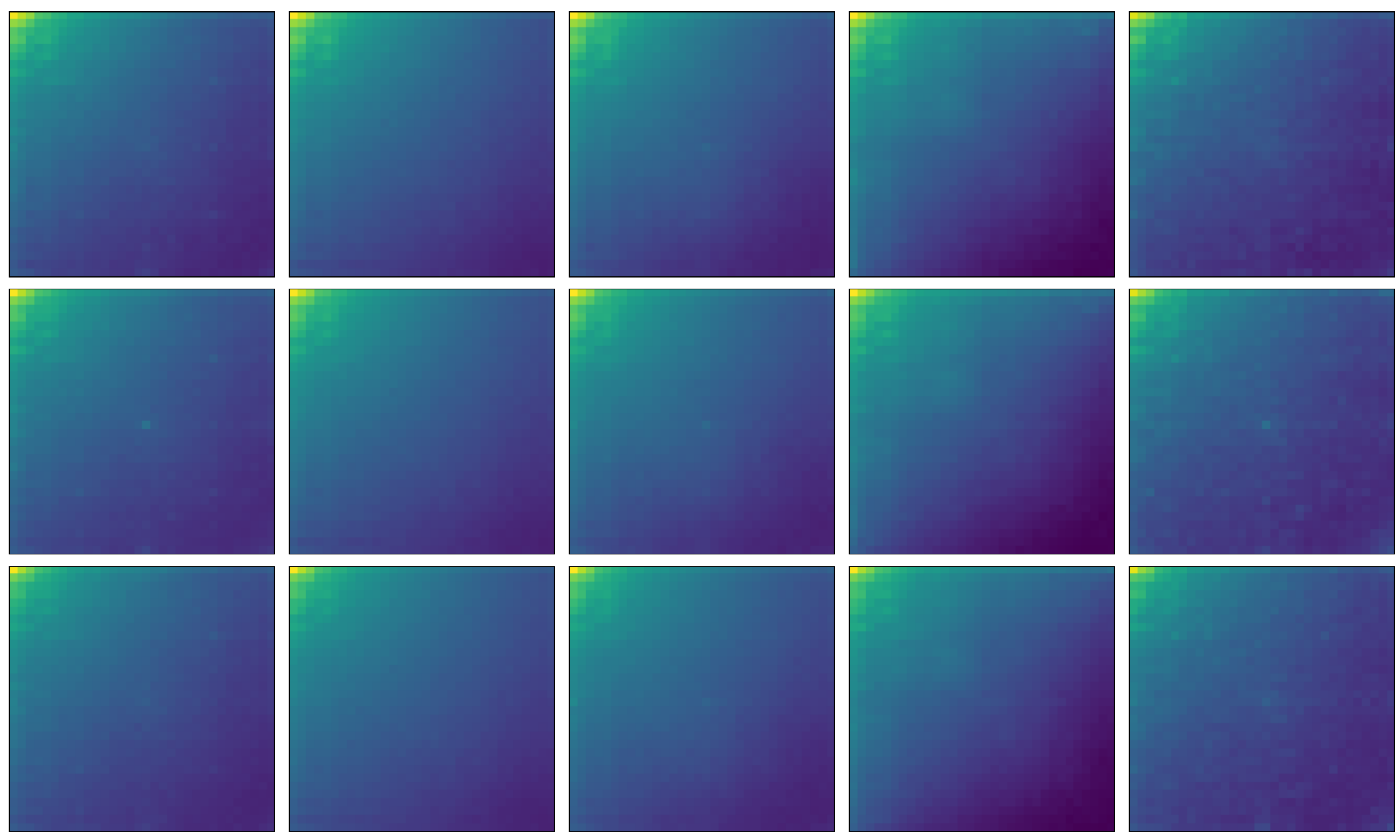}}
\caption{Average DCT-features on CelebA for real, DCGAN, WGAN-GP, LSGAN, and EBGAN samples (from left to right). Each row represents one color channel.}
\label{fig:avg_dct_celeba}
\end{center}
\end{figure}

\begin{figure*}[t]
\centering
\subfigure
{\includegraphics[width=.13\linewidth]{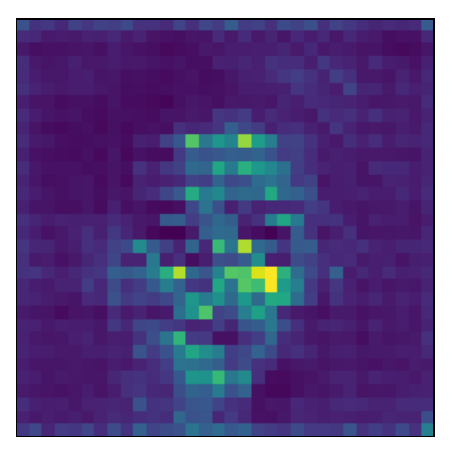}}
 \subfigure 
 {\includegraphics[width=.13\linewidth]{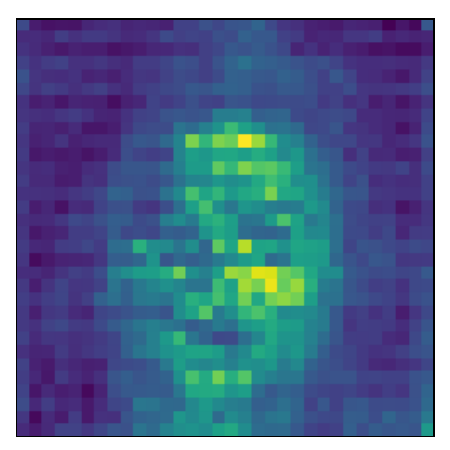}}
  \subfigure 
  {\includegraphics[width=.13\linewidth]{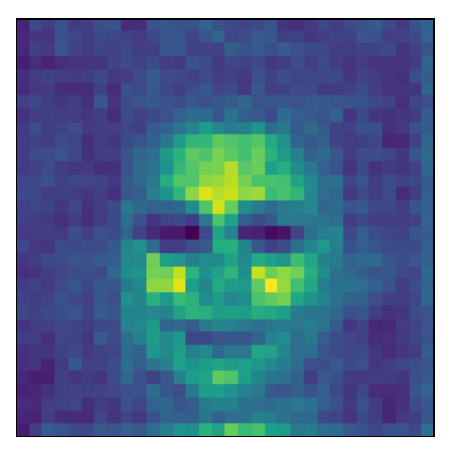}}
  \subfigure
  { \includegraphics[width=.13\linewidth]{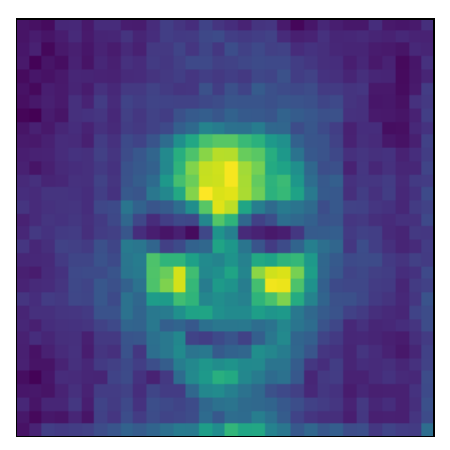}}
  \subfigure
  {\includegraphics[width=.13\linewidth]{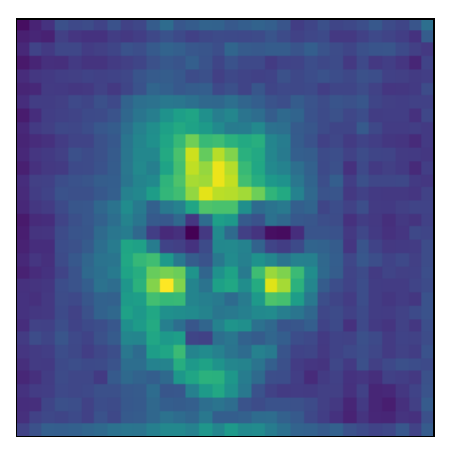}}
  \subfigure
  {\includegraphics[width=.13\linewidth]{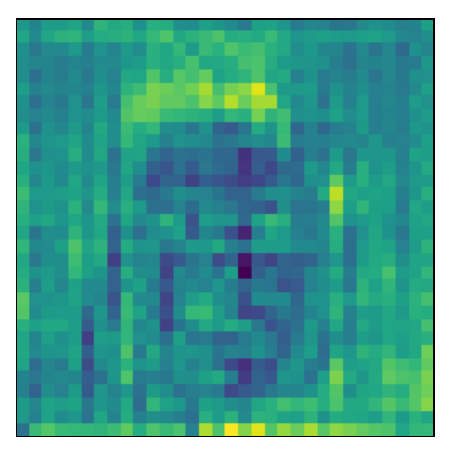}}
  
\caption{Cherry-picked channel dimension $c$ of the average reconstructed features according to (\ref{eq:opti_problem}) when $G$ is a DCGAN. The left-most figure shows the average activation $\bar{z}_c$ over channel $c$, the remaining figures show the average feature taken over DCGAN, real, WGAN-GP, LSGAN, and EBGAN samples, respectively.
}
\label{fig:features_dcgan_other}
\end{figure*}

\begin{figure*}[t]
\centering
\subfigure
{\includegraphics[width=.58\linewidth]{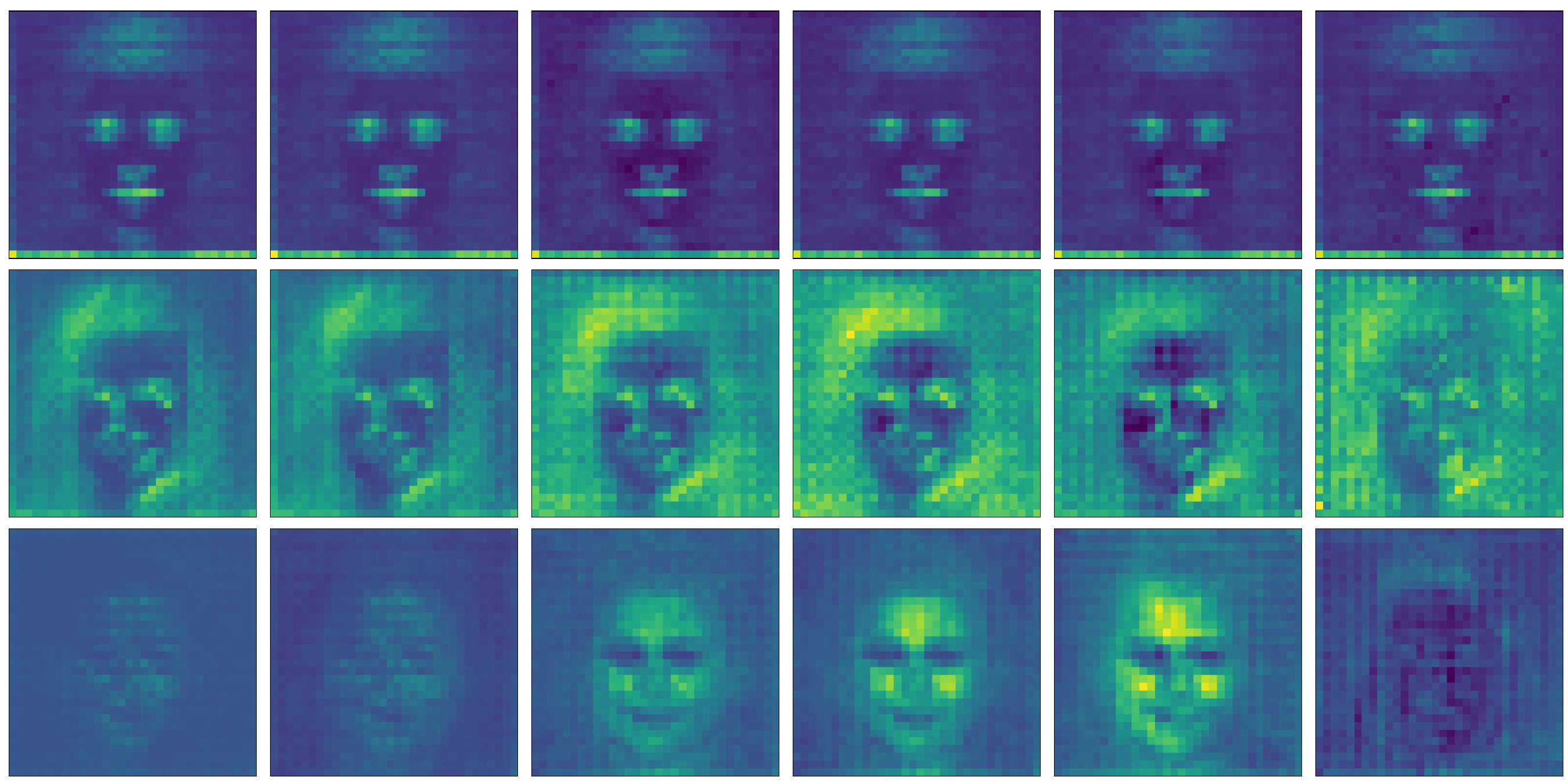}}
\subfigure
{\includegraphics[width=.58\linewidth]{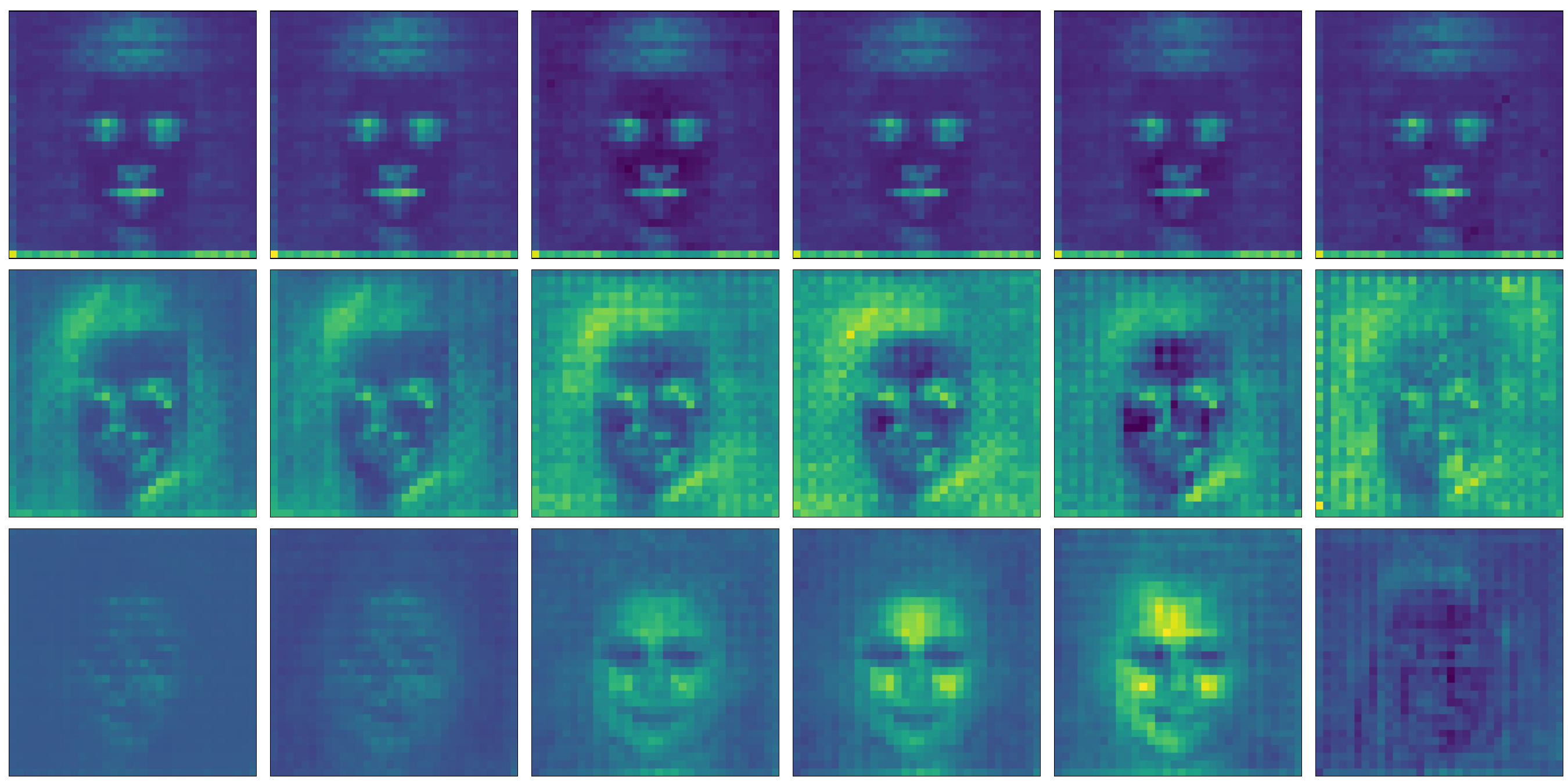}}
\subfigure
{\includegraphics[width=.58\linewidth]{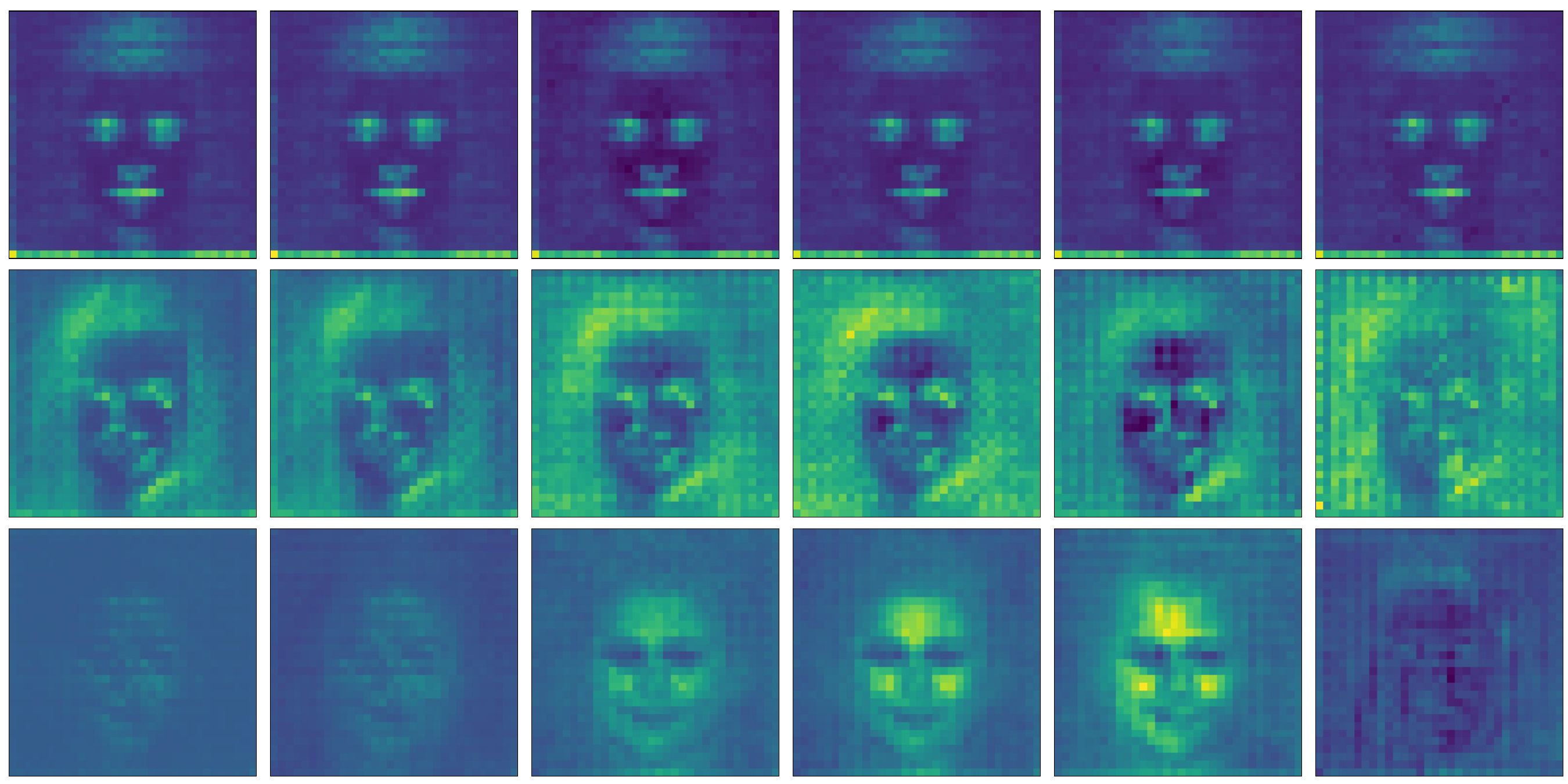}}
\subfigure
{\includegraphics[width=.58\linewidth]{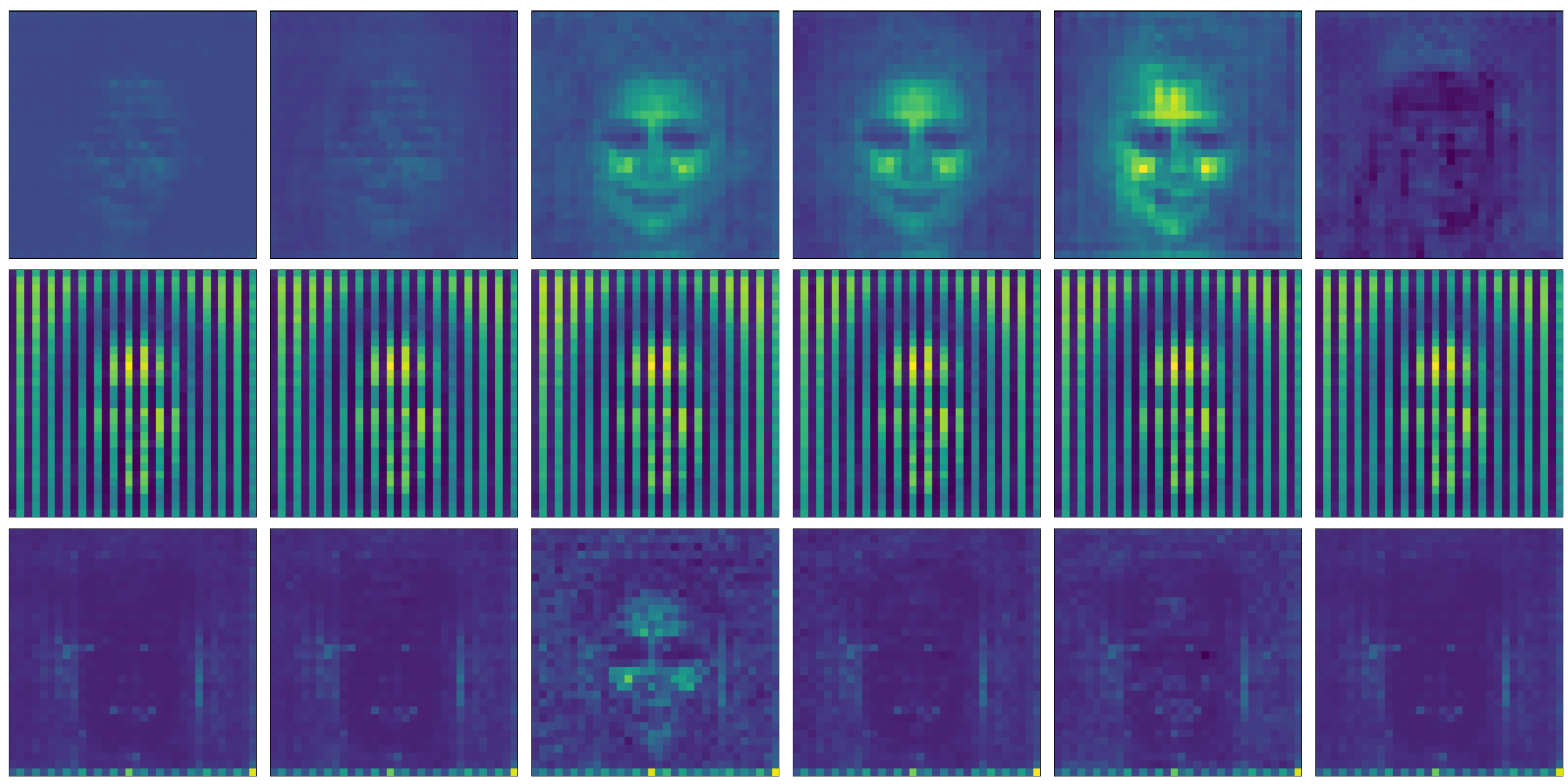}}
\caption{
Average reconstructed features according to~\eqref{eq:opti_problem} when $G$ is a DCGAN trained on CelebA. Each subfigure, from top to bottom, represents the reconstructed features for a different regularization parameter $\lambda\in\{ 0.0001, \,0.0005,\, 0.001, \, 0.1\}$, respectively. 
Every line represents one activation channel, and the columns (from left to right) show averages taken over real activations and reconstructed DCGAN, real, WGAN-GP, LSGAN, and EBGAN activations, respectively. 
}
\label{fig:varying_lambda_celeba}
\end{figure*}

\begin{figure*}[t]
\centering
\subfigure
{\includegraphics[width=.58\linewidth]{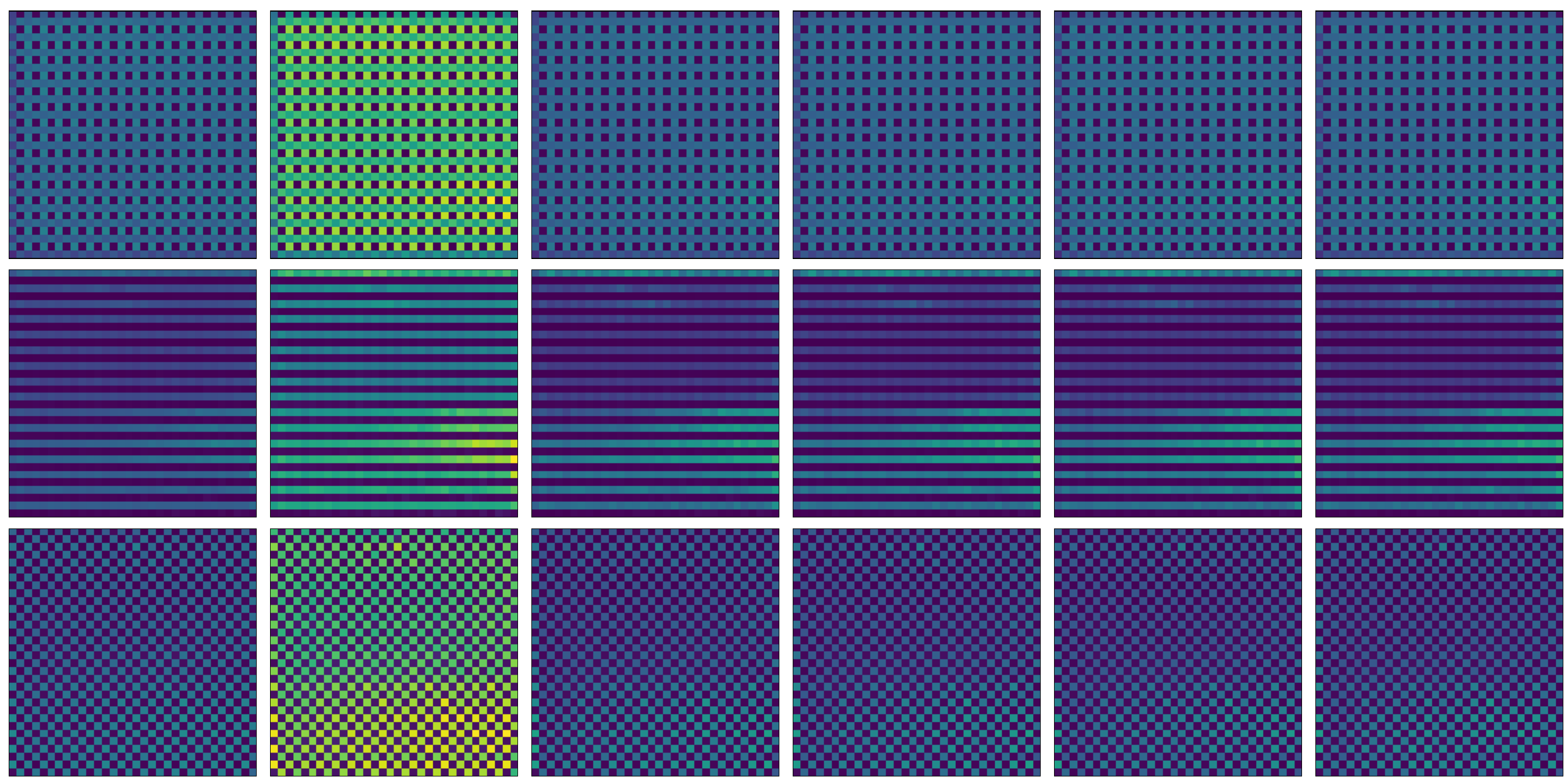}}
\subfigure
{\includegraphics[width=.58\linewidth]{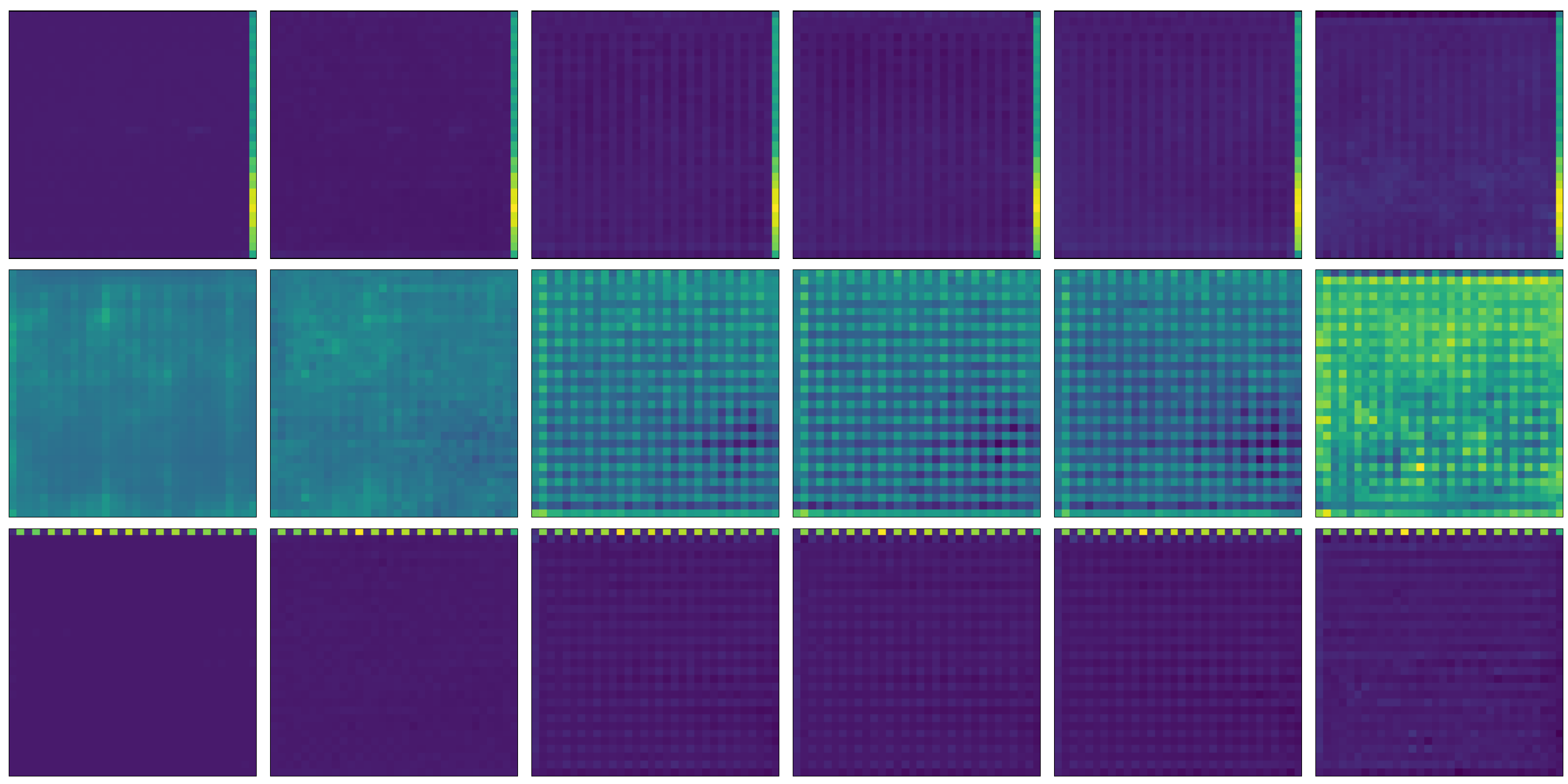}}
\subfigure
{\includegraphics[width=.58\linewidth]{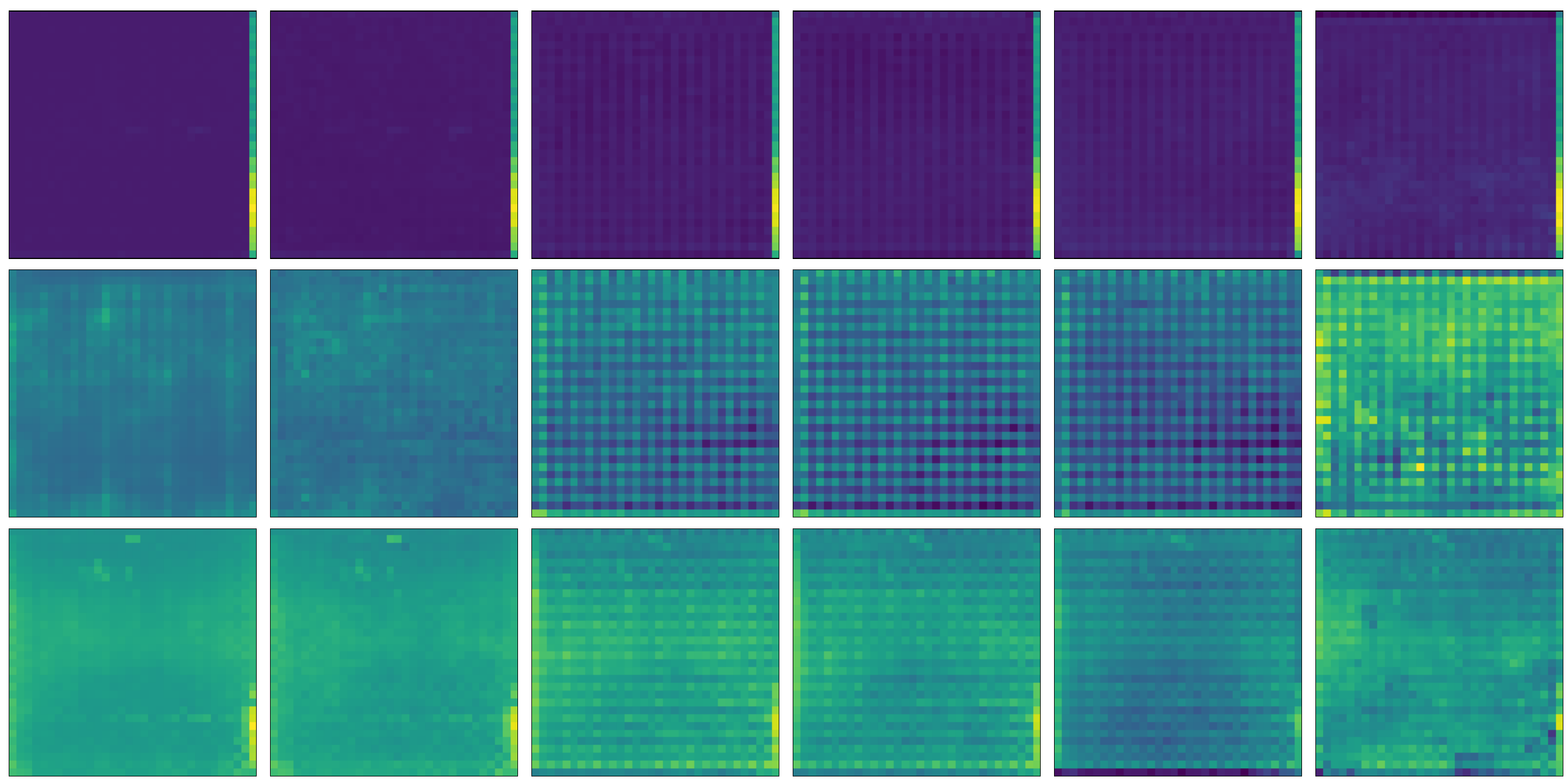}}
\subfigure
{\includegraphics[width=.58\linewidth]{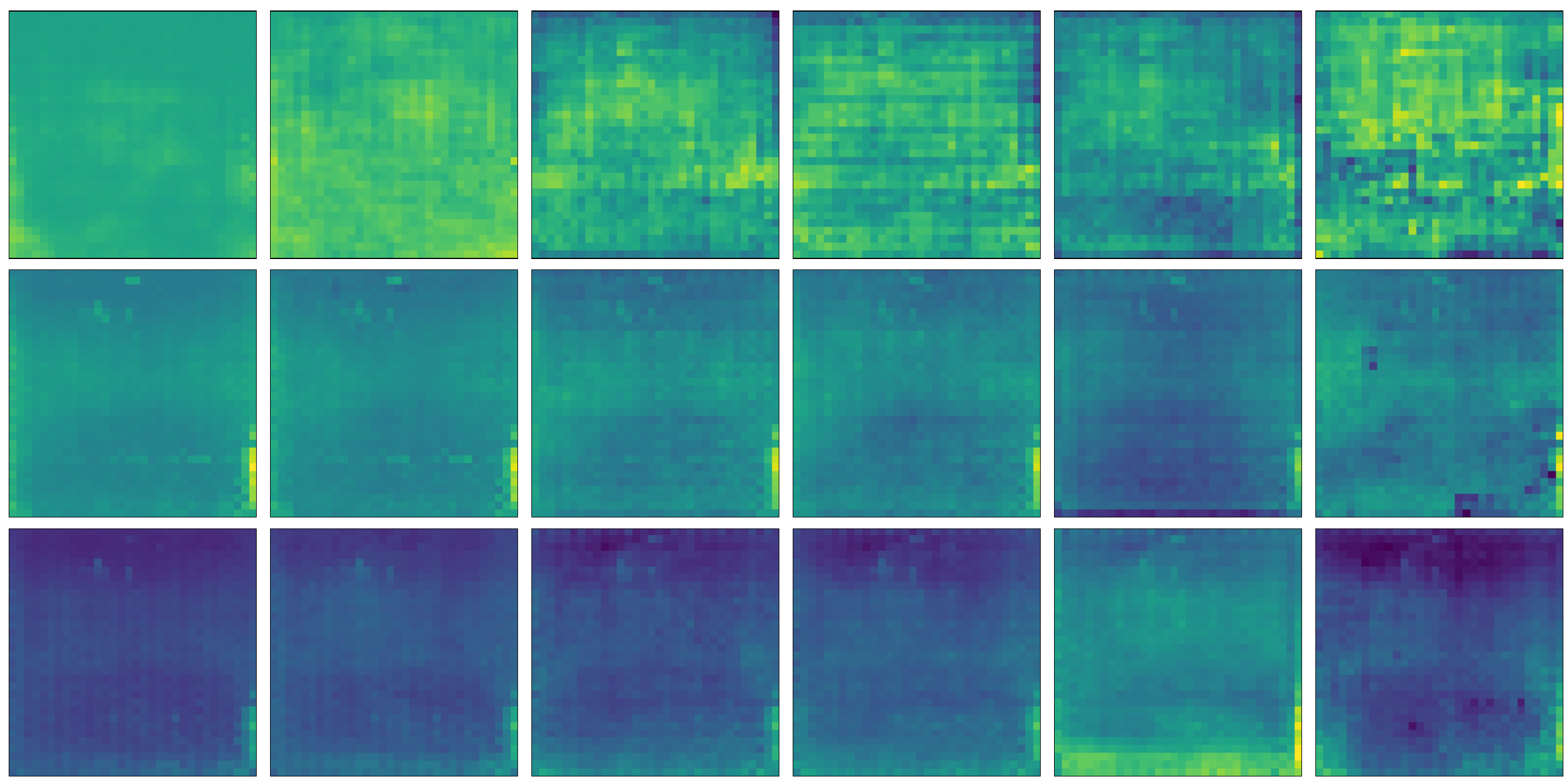}}
\caption{Average reconstructed features according to~\eqref{eq:opti_problem} when $G$ is a DCGAN trained on LSUN. Each subfigure, from top to bottom, represents the reconstructed features for a different regularization parameter $\lambda\in\{ 0.0001, \,0.0005,\, 0.001, \, 0.1\}$, respectively. 
Every line represents one activation channel, and the columns (from left to right) show averages taken over real activations and reconstructed DCGAN, real, WGAN-GP, LSGAN, and EBGAN activations, respectively.}
\label{fig:varying_lambda_lsun}
\end{figure*}

\paragraph{Single-Model Attribution}
We provide a more detailed view of the experiments from Table~\ref{tab:attribution_small} by showing each model's attribution accuracies against all individual models in Table~\ref{tab:confusion}. 
Furthermore, we present the corresponding AUCs in Table~\ref{tab:aucs}. However, we want to highlight that those results should be taken with a grain of salt since AUC is measuring an overly optimistic potential of a classifier.
In practice, we rely on a specific choice of threshold $\tau$ which might deteriorate the actual attribution performance. Choosing that threshold becomes even more crucial in the open-world setting, in which we do not have any access to the other models we test against, making the calibration of the threshold using only validation data considerably more difficult. 

We report standard deviations to the experiments from the main paper in Table~\ref{tab:attribution_small_stds},~\ref{table:attr_seed_stds}, and~\ref{tab:attribution_bigmodels_stds}. 
\begin{table*}[t]
\centering
\small
\begin{tabular}{lcccccc@{\hspace{1mm}}ccccc}
\toprule
 & \multicolumn{5}{c}{CelebA} && \multicolumn{5}{c}{LSUN}\\ 
 \cmidrule{2- 6} \cmidrule{8-12}
 & Real & DCGAN & WGAN-GP & LSGAN & EBGAN && Real & DCGAN & WGANGP & LSGAN & EBGAN \\
\midrule
\underline{\finger} \\ 
DCGAN &  64.82  &    - &  67.64 &  66.18 &  60.61 && 70.06 &     -&   69.34 &  70.74 &  69.63 \\ 
WGAN-GP &  50.31 &  49.97  &    -  & 50.31  & 50.13 && 53.07  & 52.92 &     - &  52.52  & 53.31 \\ 
LSGAN &    53.70  & 52.18  & 54.01   &   - &  54.37 &&  56.28 &  57.31&   56.92  &    - &  54.67 \\ 
EBGAN   & 96.28  & 92.19 &  96.56 &  98.07  &    - &&  75.10 &  74.68 &  74.48 &  74.07 &     -  \\ 
\midrule{}
\underline{\inv} \\
DCGAN & 99.55   &   -  & 99.45 &  99.55 &  95.55 && 51.20 &     -  &   52.00 &  51.05  &  51.50 \\ 
WGAN-GP&    99.80 &   99.80  &    -  &  99.80   & 99.80 &&  50.15 &  49.95  &    -  & 50.20   & 50.50 \\ 
LSGAN &    71.40 &  59.65  &  67.90 &     - &  51.25 &&  67.20 &   64.50 &   65.30 &     - &  64.85 \\ 
EBGAN &    99.80  &  99.80 &   99.80 &   99.80  &    - &&  54.80&   53.25 &  55.15  &  54.20 &     - \\ 
\midrule{}
\underline{\incinv} \\
DCGAN  &  51.45  &    -  &  51.10&   51.25  &  51.10 &&  50.80&      - &  50.65  &  50.50  &  50.50 \\ 
WGAN-GP &  50.15 &   50.10 &     - &  50.25&   50.15 && 50.65 &   50.30&      -  & 50.15  &  50.30 \\ 
LSGAN  &   53.30 &   52.50   & 52.90&      - &  51.85 && 52.15&   51.05  &  51.10  &    -  & 52.15 \\ 
EBGAN  &  79.65 &  73.25 &  79.25  &  81.70  &    - && 52.30&    51.90  & 52.35 &  51.95  &    - \\ 
\midrule{}
\underline{\raw} \\ 
DCGAN  &  98.23 &     - &  98.03  & 96.92 &  86.63 &&  69.30    &  - &  66.61&   62.15 &  70.52 \\ 
WGAN-GP  & 65.45 &  60.18  &    - &  59.78 &  51.08 && 52.96 &   51.20 &     - &  52.87  & 52.41 \\ 
LSGAN &   99.65&    99.50 &  99.65&      - &  99.33 && 98.23  & 97.98&   98.33&      -  & 98.33 \\ 
EBGAN  &  99.65 &  99.26 &  99.59  & 99.67 &     - &&  94.53&    94.10 & 93.79 &   93.40 &     - \\ 
\midrule 
\underline{\dct} \\ 
DCGAN  &  99.81  &    -  &  80.20  & 95.44&  49.86 && 99.64 &     - &  74.98 &  96.96&   50.34 \\ 
WGAN-GP&   99.68 &  49.71  &    -  &  60.70 &  49.71 && 99.04 &  49.71 &     -  & 73.91 &  49.73 \\ 
LSGAN  &  99.70 &  98.74 &  96.34 &     -&   92.39 && 97.95  & 52.44 &  67.15    &  -  & 54.18 \\ 
EBGAN &   99.66  & 70.39&   98.69 &  99.59   &   - && 99.65&   65.93  & 96.45  & 99.59  &    - \\ 
\midrule 
\underline{\method} \\ 
DCGAN  &  99.72  &    - & 99.72  & 99.72  & 98.21 &&  99.69 &     -  & 99.69 &  99.66 &  91.95 \\ 
WGAN-GP&   99.83  & 99.15  &    -   & 98.80 &  99.83 && 99.65&   98.71 &     - &  99.27 &  93.29 \\ 
LSGAN &   99.68  & 99.38&   97.95&      - &  98.77 &&  99.57 &  98.86 &    96.00 &     - &  98.32 \\ 
EBGAN &   99.68 &  99.68  & 99.68  & 99.68 &     - && 99.68 &  98.49 &  99.67 &  99.68   &   - \\ 
\bottomrule
\end{tabular}
\caption{Detailed view of the individual model attribution accuracies leading to the results in Table~\ref{tab:attribution_small}. Each score corresponds to an average over five runs. Note, due to their excessive computational load, we do not repeat the inversion methods multiple times.}
\label{tab:confusion}
\end{table*}

\begin{table*}[ht!]
\centering
\begin{tabular}{lccccrcccc}
\toprule
 & \multicolumn{4}{c}{CelebA} & & \multicolumn{4}{c}{LSUN} \\
 \cmidrule{2-5} \cmidrule{7-10}
 & DCGAN & WGAN-GP & LSGAN & EBGAN && DCGAN & WGANGP & LSGAN & EBGAN \\
\midrule
\finger &   95.45  &   79.09  &  87.94   & 99.84  &&  96.99  &   79.09    &86.73 &   98.36\\
\inv    &   99.42   & \textbf{100.00}   & 99.25 &  \textbf{100.00} &&   71.64   &  80.36 &   99.33    &97.66\\
\incinv      &    83.14  &   67.39   & 85.66  &  99.25 &&   61.12    & 55.71   & 74.31 &   92.28\\
\raw         &    99.33  &   85.03  &  \textbf{99.99} &   99.96 &&   93.82  &   74.65   & 99.83 &   99.53\\
\dct       &      82.85    & 75.69  &  99.80 &   97.47  &&  84.89   &  79.22  &  82.91 &   94.31 \\
\method{}    &         \textbf{99.91} &    99.96&   99.91  & \textbf{100.00}  &&  \textbf{99.81}    & \textbf{99.49}&    \textbf{99.73}  &  \textbf{99.98}\\
\bottomrule
\end{tabular}
\caption{Averaged area under curve (AUC) values corresponding to the results in Table~\ref{tab:attribution_small}.}
\label{tab:aucs}
\end{table*}

\begin{table*}
\small
\centering
\begin{tabular}{lccccrcccc}
\toprule
 & \multicolumn{4}{c}{CelebA} & & \multicolumn{4}{c}{LSUN} \\
 \cmidrule{2-5} \cmidrule{7-10}
 & DCGAN & WGAN-GP & LSGAN & EBGAN && DCGAN & WGAN-GP & LSGAN & EBGAN \\
\midrule

\finger & 2.95 &	0.26 &	0.95 &	2.81 &&	1.60 &	0.87& 	1.65 &	4.16 \\ 
\inv &	1.98 	&0.00& 	9.00 &	0.00 &&	0.42 	&0.23 &	1.20 &	0.83 \\
\incinv &	0.17 	&0.06 	&0.62 	&3.64 	&&0.14 	&0.21& 	0.62 &	0.23\\
\raw &	5.53 &	5.60 &	0.20 	&0.21 &&	3.79 	&1.08 &	0.34 &	1.66 \\
\dct 	&20.09 &	21.11 &	2.93 	&12.95 	&&20.43 &	21.04 &	18.76 	&14.57 \\
\method &	0.73 &	0.56 	&0.92 	&0.07 	&&4.48 	&2.79& 	1.54 &	0.58 \\
\bottomrule
\end{tabular}
\caption{Standard deviation corresponding to the results from Table~\ref{tab:attribution_small}. For illustration purposes, we scale all values by $10^2$.}
\label{tab:attribution_small_stds}
\end{table*}

\begin{table*}
\small
\centering
\begin{tabular}{lccccrcccc}
\toprule
 & \multicolumn{4}{c}{CelebA} & & \multicolumn{4}{c}{LSUN} \\
\cmidrule{2-5} \cmidrule{7-10}
 & DCGAN & WGAN-GP & LSGAN & EBGAN && DCGAN & WGANGP & LSGAN & EBGAN \\
\midrule
\raw 	&2.21 	&3.11 &	1.23 &	15.72&& 	3.29 &	0.88 &	5.04 &	3.73 \\ 
\dct 	&0.27 &	0.14 	&4.31& 	4.13 &&	0.19 	&2.33& 	1.05 &	1.42 \\
\method  	&0.61 	&5.49 	&3.85 &	0.11 &&	2.32 &	3.77 &	5.41 &	9.61 \\ 
\bottomrule
\end{tabular}
\caption{Standard deviation corresponding to the results from Table~\ref{table:attr_seed}. For illustration purposes, we scale all values by $10^2$.}
\label{table:attr_seed_stds}
\end{table*}

\paragraph{Single-Model Attribution with less Training Data}
Table~\ref{table:smaller_training} shows the performance of \method{} when using less training samples. Particularly, we repeat the experiments from Table~\ref{tab:attribution_small} with $1\,000$ and $5\,000$ real and generated samples. 
Unsurprisingly, we observe a performance drop but the performance remains high in most cases. 
The only exception is when training \method{} with only 1k samples from LSGAN. However, further empirical investigations reveal that we can fix it by setting the thresholding value $\tau$ to a lower value. When setting $\tau$ to a more liberal choice, for instance, the $95\%$-quantile, we can improve the performance considerably. In this case, the averaged attribution accuracy on LSGAN increases to $87.75$ and $87.34$ in CelebA and LSUN with only $1\,000$ samples, respectively.
But note that our considered perspective from the model trainer (see Section~\ref{sec:problem_setup}) implicitly implies that we do not have strict training data limitations. Since this perspective assumes to have access to $G$ we can, potentially, generate infinitely many samples from $G$. Furthermore, training a generative model typically requires great amounts of real training data, which we can reuse to train \method{}. Hence, the training data requirements in our considered setting are rather weak.

In summary, we conclude that we typically do not have strict data constraints in our investigated setting. Still, in a hypothetical limited data regime, \method{} also performs very well in most settings. In some cases (here in the case of LSGAN), one may consider adjusting the validation procedure for selecting the threshold $\tau$. This leads to more false negatives but can increase the overall attribution accuracy.

\begin{table*}
\small
\centering
\begin{tabular}{lccccrcccc}
\toprule
 & \multicolumn{4}{c}{CelebA} & & \multicolumn{4}{c}{LSUN} \\
\cmidrule{2-5} \cmidrule{7-10}
 & DCGAN & WGAN-GP & LSGAN & EBGAN && DCGAN & WGANGP & LSGAN & EBGAN \\
\midrule
\method{}-1k  	& 	96.89&	95.97	&55.77	&99.73  &&	91.97&	97.00	&82.52	&98.33  \\ 
\method{}-5k  	&99.05	&98.75	&95.24	&99.65 &&	95.61&	97.54	&94.88	&99.40 \\ 
\method{}-10k  	&99.34	&99.40	&98.94&	99.68
&&	97.75	&97.73&	98.19&99.38 \\ 
\bottomrule
\end{tabular}
\caption{Single-model attribution accuracy of \method{} using $1\,000$, $5\,000$, and $10 \, 000$ real and generated samples. The results are averaged over five runs.}
\label{table:smaller_training}
\end{table*} 

\paragraph{Single-Model Attribution on Perturbed Samples}
To illustrate the effect of the considered perturbations on the images, we present an overview in Figure~\ref{fig:perturbations}. 
We observed that we can improve the attribution performance of our proposed methods by choosing the thresholding method more liberal. Specifically, in the presence of perturbations, we can set $\operatorname{fnr}=0.05$ to find $\tau$ (compare with Section~\ref{sec:thresholding_method}). We provide the improved results in~Table~\ref{table:perturb_higher_fnr}.

\begin{figure*}[t]
\centering
\subfigure[Blur CelebA]
{\includegraphics[width=.23\linewidth]{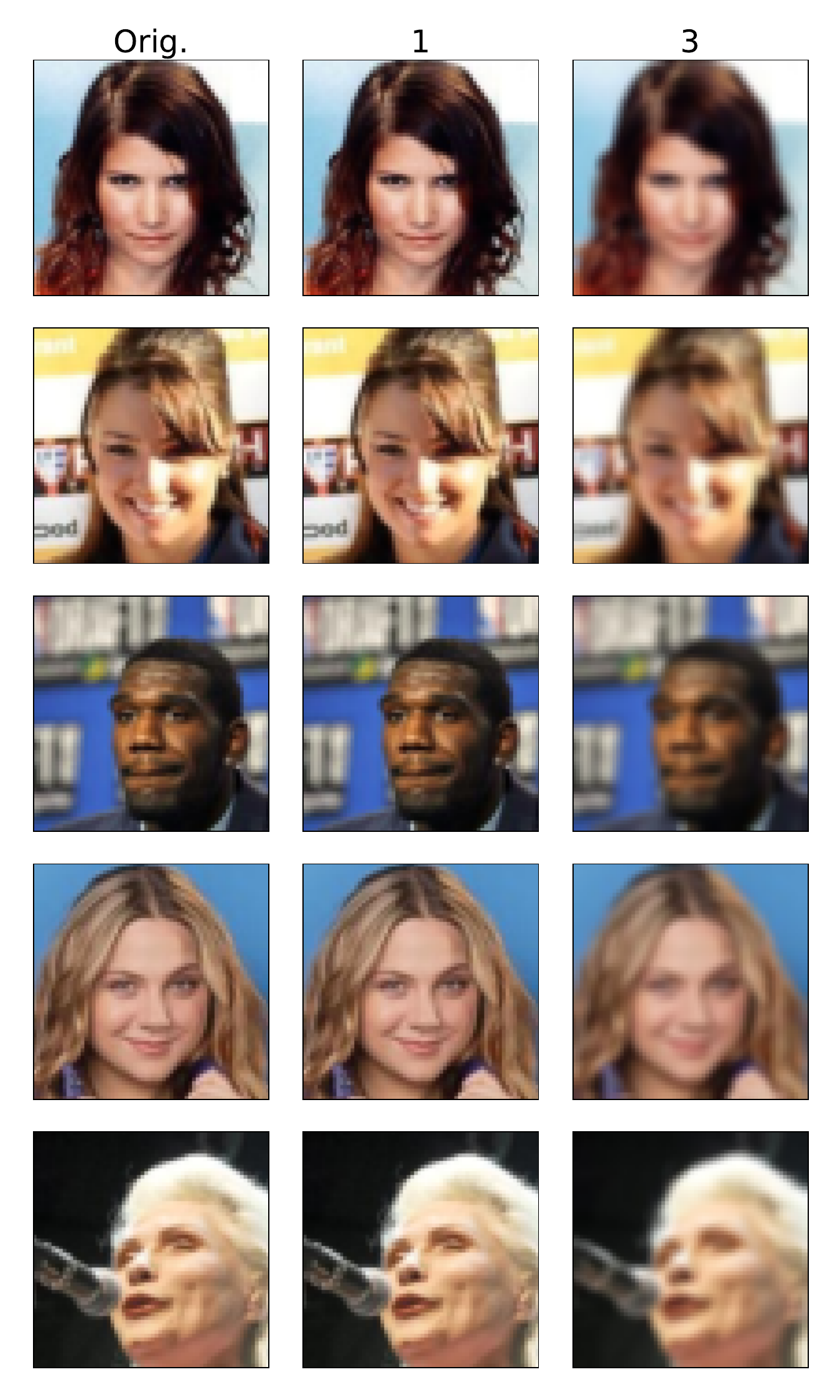}}
\subfigure[Blur LSUN]
{\includegraphics[width=.23\linewidth]{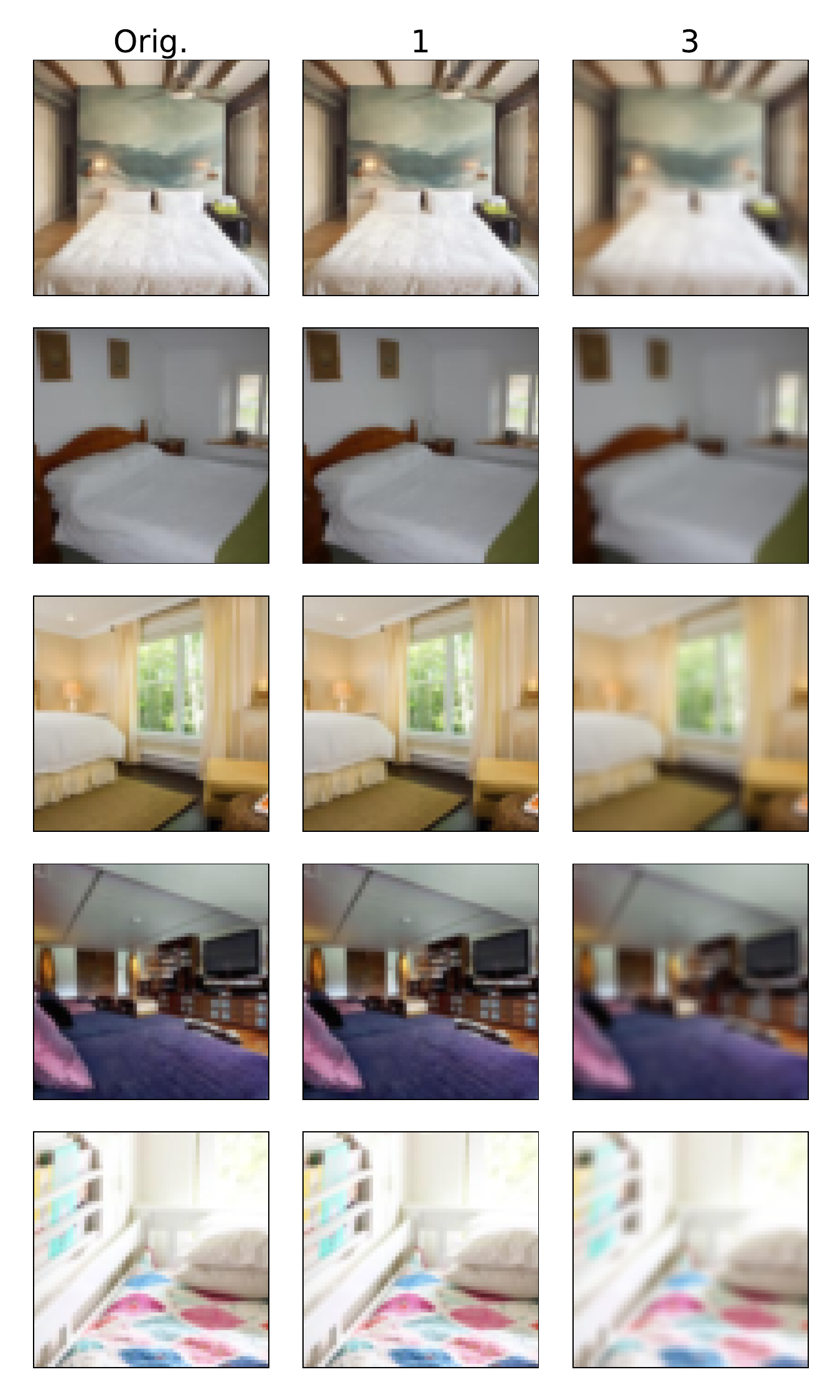}}
\subfigure[Crop CelebA]
{\includegraphics[width=.23\linewidth]{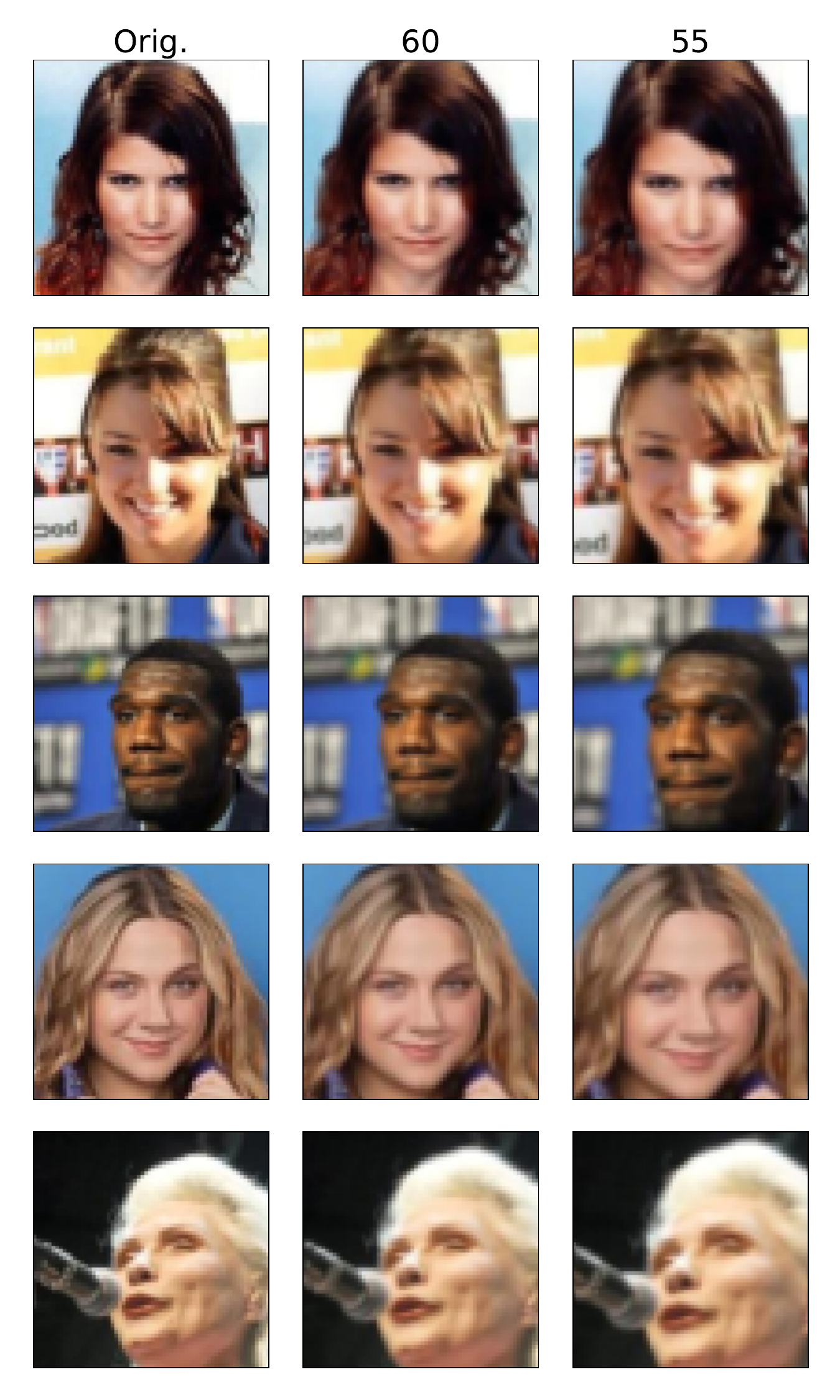}}
\subfigure[Crop LSUN]
{\includegraphics[width=.23\linewidth]{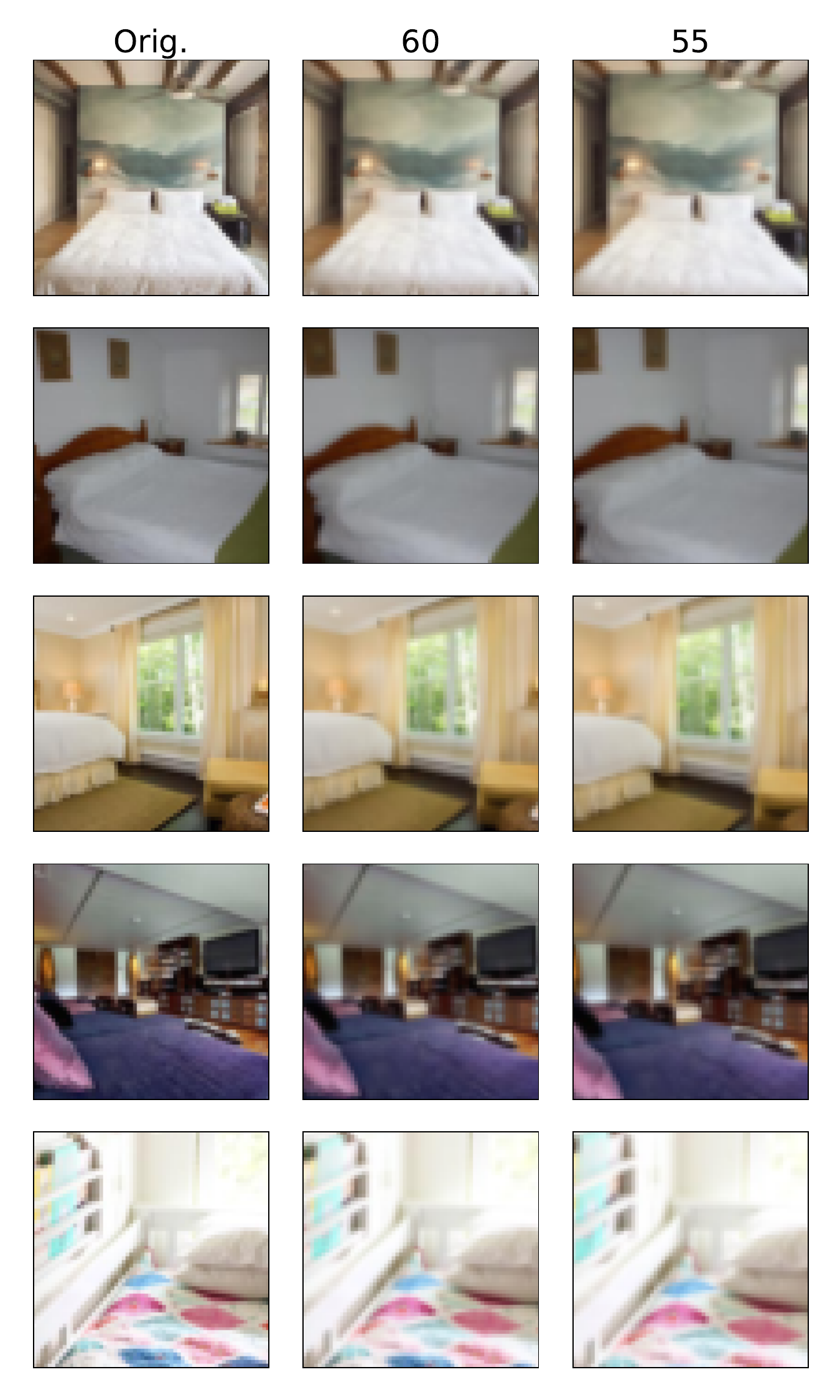}} \\
\subfigure[JPEG CelebA]
{\includegraphics[width=.23\linewidth]{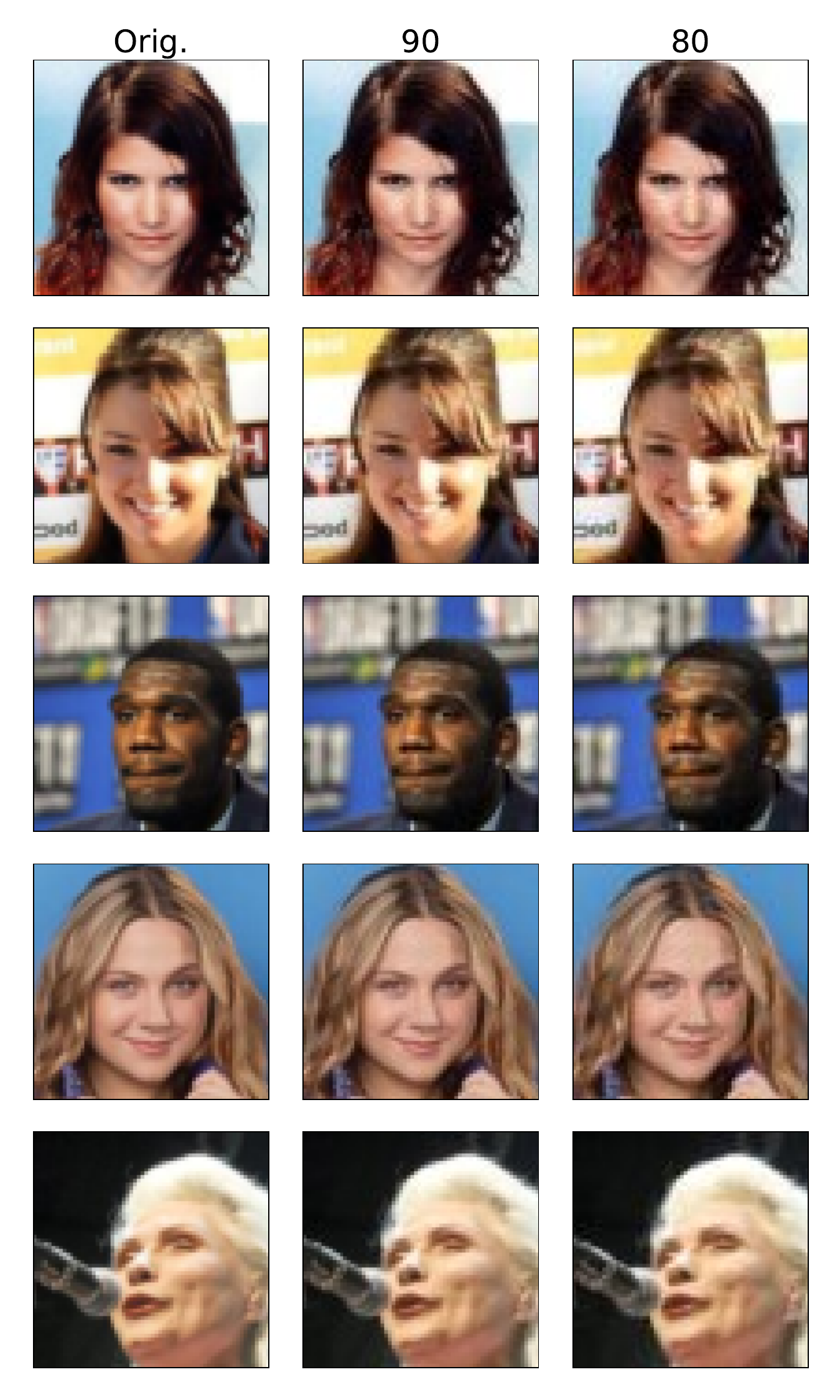}}
\subfigure[JPEG LSUN]
{\includegraphics[width=.23\linewidth]{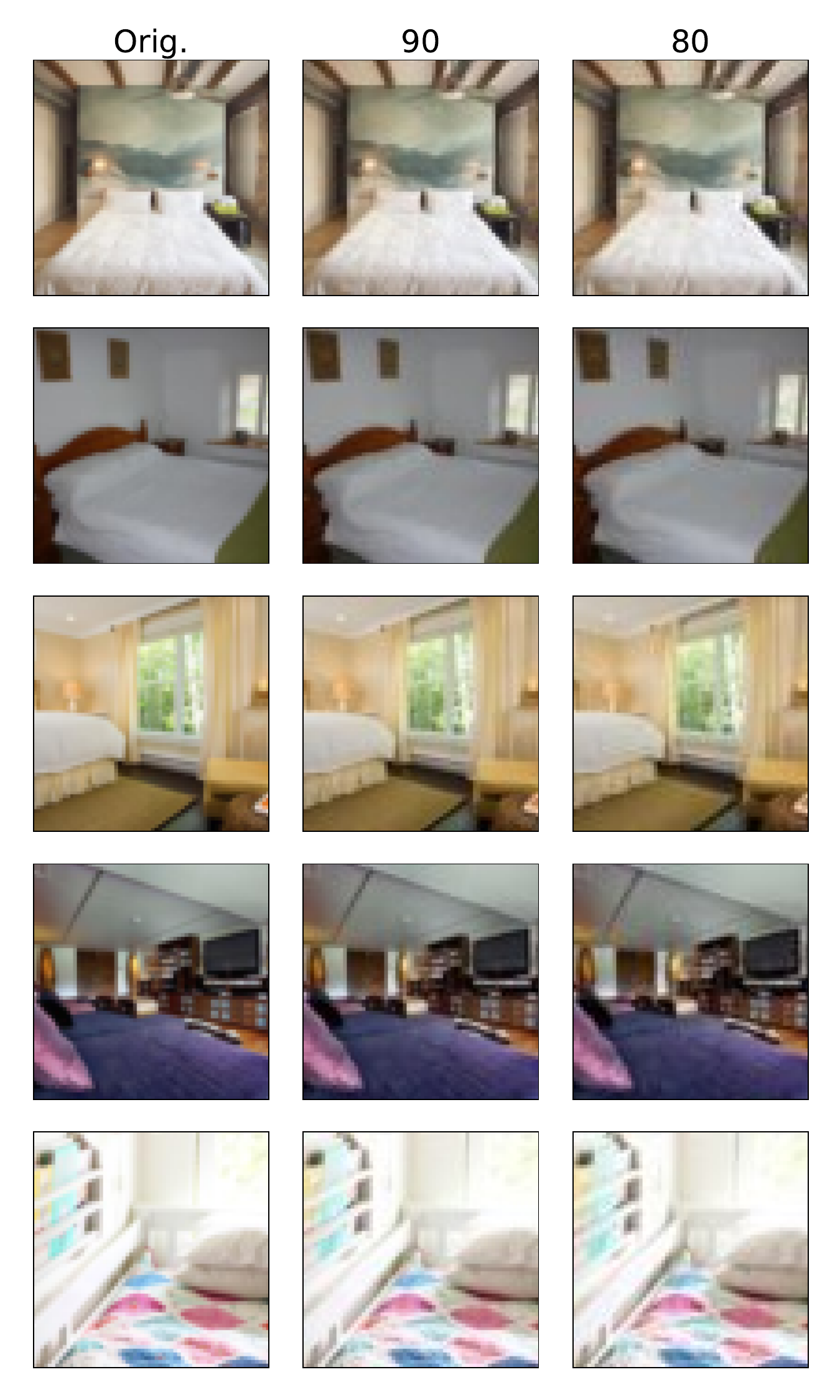}}
\subfigure[Noise CelebA]
{\includegraphics[width=.23\linewidth]{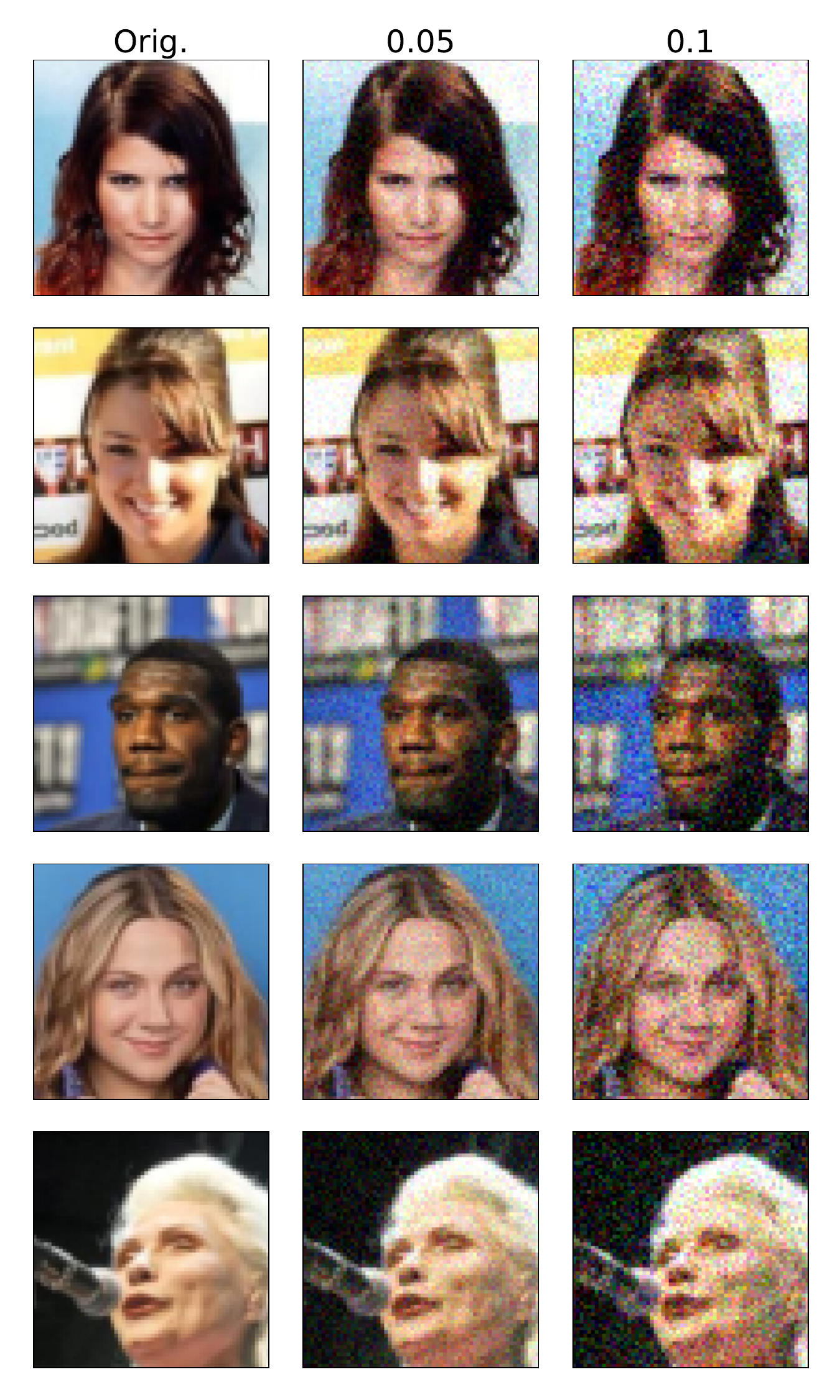}}
\subfigure[Noise LSUN]
{\includegraphics[width=.23\linewidth]{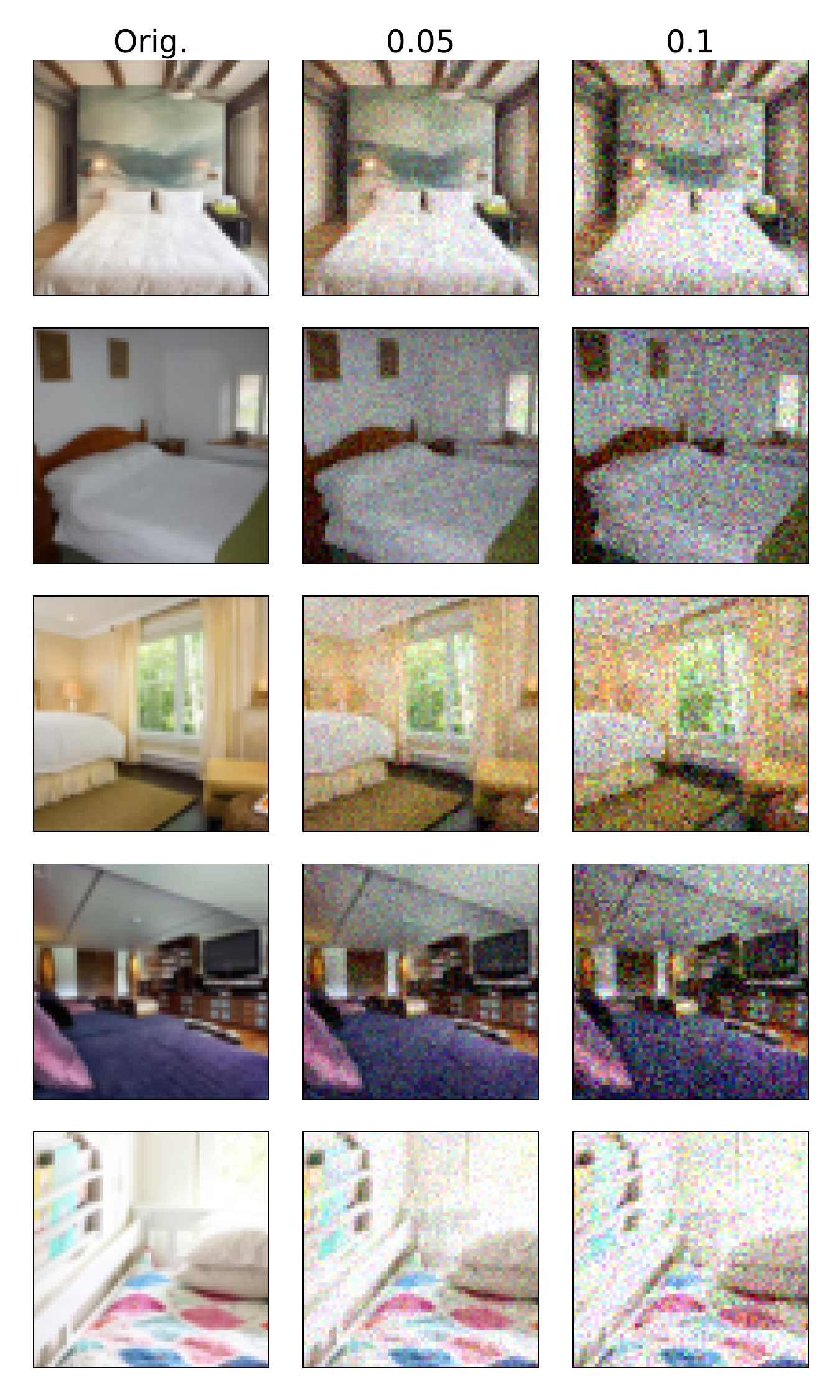}}

\caption{Visualization of different perturbations.}
\label{fig:perturbations}
\end{figure*}

\begin{table*}
    \centering
    \small 
    \begin{tabular}{@{\ }lc@{\hspace{1mm}}cc@{\hspace{1mm}}cc@{\hspace{1mm}}cc@{\hspace{1mm}}c@{\hspace{0mm}}
    cc@{\hspace{1mm}}cc@{\hspace{1mm}}cc@{\hspace{1mm}}cc@{\hspace{1mm}}c@{\ }
    }
    \toprule
    & \multicolumn{8}{c}{CelebA} && \multicolumn{8}{c}{LSUN}\\ 
    \cmidrule{2-9} \cmidrule{11-18}
    & \multicolumn{2}{c}{Blur} & \multicolumn{2}{c}{Crop} & \multicolumn{2}{c}{Noise} & \multicolumn{2}{c}{JPEG} 
    && \multicolumn{2}{c}{Blur} & \multicolumn{2}{c}{Crop} & \multicolumn{2}{c}{Noise} & \multicolumn{2}{c}{JPEG}
    \\
    & 1 & 3 & 60 & 55 & 0.05 & 0.1 & 90 & 80
    && 1 & 3 & 60 & 55 & 0.05 & 0.1 & 90 & 80\\ 
     \midrule 
     \raw 	&90.90 &	90.03 &	91.15 &	91.27 	& \textbf{87.67} 	&\textbf{83.68} 	& \textbf{89.72} &	\textbf{89.14} &&	83.95 &	83.05 &	83.59 	&82.61& 	\textbf{80.73} &	\textbf{77.78} 	&82.44& 	\textbf{81.90} \\
    \dct 	&83.35 	&80.19 &	83.34 	&83.33 	&70.86 &	64.02 &	76.49 	&74.98 &&	79.09 	&73.08 	&79.26 	&78.28 &	55.04 &	52.20 	&61.97 &	57.80 \\ 
\method{} &	\textbf{97.40} 	& \textbf{97.52} &	\textbf{97.23} 	& \textbf{97.29} &	82.15 &	73.76 &	88.18 &	84.12 	&& \textbf{96.58} 	& \textbf{96.78} 	& \textbf{96.12} &	\textbf{95.22} &	78.73 	&60.47 &	\textbf{85.50} &	81.42 \\ 
    \bottomrule
    \end{tabular}
    \caption{Single-model attribution accuracy with immunization averaged over all $G, G' \in \mathcal{G}$ over five runs using a more liberal thresholding selection ($\operatorname{fnr}=0.05$). 
    We use the same perturbations as in Table~\ref{table:perturb1}.} 
    \label{table:perturb_higher_fnr}
    \vspace{0pt}
\end{table*}

\begin{table*}
    \centering
    \small
    \begin{tabular}{@{\ }lcccccccccccc@{\ }}
        \toprule
        & \multicolumn{4}{c}{Stable Diffusion} && \multicolumn{4}{c}{Style-Based Models} && \multicolumn{2}{c}{Medical Image Models}
        \\ 
        \cmidrule{2-5} \cmidrule{7-10} \cmidrule{12-13}
         & v1-4 & v1-1 & v1-1+ & v2 & & \textit{real} & StyleGAN-XL & StyleNAT & StyleSwin && WGAN-GP & C-DCGAN \\
        \midrule
        \finger & 0.00 &  0.00 & 0.00 & 0.00 & &  0.00 &   0.00 &  0.00 &  0.00 && 0.10 &  0.10 \\
        \raw & 4.10 & 3.36 & 1.35 & 0.71 & & 4.10 & 3.36 & 2.89 & 1.75 && 0.12 & 19.76 \\ 
        \dct & 1.32 & 1.24 & 0.29 & 0.29  & & 0.07 & 0.07 & 0.16 & 0.08 && 0.18 & 17.97 \\ 
        \method & 1.35 & 1.04 & 1.19 & 1.56 && - & - & - & -  && 0.11 & 0.19 \\ 
        \bottomrule
    \end{tabular}
    \caption{Standard deviation corresponding to the results from Table~\ref{tab:attribution_other_image}. For illustration purposes, we scale all values by $10^2$.}
\label{tab:attribution_bigmodels_stds}
    \vspace{0pt}
\end{table*}

\paragraph{Generative Models for Tabular Data}
In addition to the models trained on Redwine, we repeat the experiments for models trained on the Whitewine dataset~\cite{wine}. The results are reported in Table~\ref{tab:whitewine}. Here, we can see a different behavior: \inv{} outperforms \raw{} and \method{}. 
Moreover, we observe in our experiments that \method{} tends to perform better when $\lambda$ approaches $0$. This behavior suggests that it might be less meaningful to regularize towards the mean activation in the Whitewine experiments. 
Investigating other types of regularization in~\eqref{eq:opti_problem}, however, is out of the scope of this work.

We report the corresponding standard deviations in Table~\ref{tab:attribution_tabular_stds}.

\begin{table}[t]
    \centering
    \small
    \begin{tabular}{@{\ }lccc@{\ }}
        \toprule
         & TVAE & CTGAN & Cop.GAN\\
        \midrule
        \inv & \textbf{95.69} & \textbf{96.38} & \textbf{96.21} \\ 
        \raw & 84.22 & 80.85 & 81.44 \\
        \method & 84.15 & 80.58 & 80.84 \\
        \bottomrule
    \end{tabular}
    \caption{Single-model attribution accuracy of KL-WGAN trained on Whitewine against $G'$ as indicated by the column name averaged over five runs.} 
    \label{tab:whitewine}
\end{table}

\begin{table*}[t]
    \centering
    \begin{tabular}{@{\ }lccccccc@{\ }}
        \toprule
        & \multicolumn{3}{c}{Redwine} && \multicolumn{3}{c}{Whitewine} \\ 
        \cline{2-4}
        \cline{6-8} 
         & TVAE & CTGAN & Cop.GAN &\phantom{a}& TVAE & CTGAN &  Cop.GAN  \\
        \midrule
        \inv & 16.51 &  12.50 & 14.61 &\phantom{a}& 18.87 & 16.23 & 13.53 \\ 
        \raw & 14.89 &  8.37 & 21.20 &\phantom{a}& 13.29 & 15.44 & 13.17 \\
            \method & 12.50 &  5.57 & 18.37 &\phantom{a}& 12.82 & 16.48 & 11.17 \\
        \bottomrule
    \end{tabular}
    \caption{
    Standard deviation corresponding to the results from Table~\ref{tab:redwine} and Table~\ref{tab:whitewine}. For illustration purposes, we scale all values by $10^3$.
    } \label{tab:attribution_tabular_stds}
\end{table*}

\end{document}